\documentclass[twoside, 11pt]{article}
\usepackage{jmlr-like}
\usepackage[fleqn]{amsmath}
\usepackage{amssymb}
\usepackage{amsfonts}
\usepackage{mathtools}
\usepackage[ruled,linesnumbered, boxed]{algorithm2e}
\usepackage{tikz,pgfplots}
  \usetikzlibrary{arrows, shapes, positioning}
\usepackage{multirow}

\newcommand{\set}[1] {\ensuremath{\mathbf {#1}}}
\newcommand{\prob}[1] {\ensuremath{\mathcal {#1}}}

\newcommand{\graph}[1] {\ensuremath{\mathcal {#1}}}
\newcommand{\var}[1]{\ensuremath{#1}}
\newcommand{\marg}[1]{\ensuremath{[_{#1}}}
\newcommand{\manip}[1]{\ensuremath{^{#1}}}

\newcommand{\independent}{\ensuremath{\perp\!\!\!\perp}}

\newcommand {\given}{\ensuremath{\;\vert\;}}

\newcommand{\dotcup}{\ensuremath{\dot{\cup}}}

\newcommand{\m}{\ensuremath{m}}

\newcommand{\combine}{COmbINE}

\newcommand*\circrightarrow{\ensuremath{
      \raisebox{1pt}{%
      \begin{tikzpicture}[>=latex, node distance = 0.5 cm, line width = 0.1pt]
        \node [draw,circle, inner sep =1.25pt](a) {};
        \node [coordinate,right of= a] (b) {};
        \path(a)edge[thin,->](b);
      \end{tikzpicture}}}}

\newcommand*\starrightarrow{\ensuremath{
      \raisebox{1pt}{%
      \begin{tikzpicture}[>=latex, node distance = 0.5 cm, line width = 0.1pt]
      \node (a)[star, star point ratio=4, minimum size=4pt,inner sep =0.01pt, draw] {};
        \node [coordinate,right of= a] (b) {};
        \path(a)edge[thin, ->](b);
      \end{tikzpicture}}}}

\newcommand*\starleftarrow{\ensuremath{
      \raisebox{1pt}{%
      \begin{tikzpicture}[>=latex, node distance = 0.5 cm, line width = 0.1pt]
        \node [coordinate] (a) {};
        \node (b)[star,star point ratio=4, minimum size=4pt,inner sep =0.01pt, draw, right of  =  a] {};
        \path(a)edge[thin,<-](b);
      \end{tikzpicture}}}}

\newcommand*\doublestar{\ensuremath{
      \raisebox{1pt}{%
      \begin{tikzpicture}[>=latex,node distance = 0.5 cm, line width = 0.1pt]
        \node [star,star point height=2pt, star point ratio=4, minimum size=4pt,
      inner sep =0pt, draw] (a) {};
        \node (b)[star,star point height=2pt, star point ratio=4, minimum size=4pt,
      inner sep =0pt, draw, right of  =  a] {};
        \path(a)edge[-](b);
      \end{tikzpicture}}}}

\newcommand*\doublecirc{\ensuremath{
      \raisebox{1pt}{%
      \begin{tikzpicture}[>=latex,node distance = 0.5 cm,line width = 0.1pt]
        \node [draw,circle, inner sep =1.25pt, draw] (a) {};
        \node (b)[draw,circle, inner sep =1.25pt, right of  =  a] {};
        \path(a)edge[-](b);
      \end{tikzpicture}}}}

\newcommand*\pagrightarrow{\ensuremath{
     \raisebox{2.2pt}{%
      \begin{tikzpicture}[>=latex,node distance = 0.5 cm,line width = 0.1pt]
        \node (a)[coordinate] {};
        \node (b) [coordinate, right of  =  a] {};
        \path(a)edge[thin, ->](b);
      \end{tikzpicture}}}}

\newcommand*\pagleftarrow{\ensuremath{
     \raisebox{2.2pt}{%
      \begin{tikzpicture}[>=latex,node distance = 0.5 cm,line width = 0.1pt]
        \node (a)[coordinate] {};
        \node (b) [coordinate, right of  =  a] {};
        \path(a)edge[thin, <-](b);
      \end{tikzpicture}}}}

\newcommand*\pagbidir{\ensuremath{
     \raisebox{2.2pt}{%
      \begin{tikzpicture}[>=latex,node distance = 0.5 cm,line width = 0.1pt]
        \node (a)[coordinate] {};
        \node (b) [coordinate, right of  =  a] {};
        \path(a)edge[thin, <->](b);
      \end{tikzpicture}}}}

\newcommand*\smrightbidir{\ensuremath{
\raisebox{1.2pt}{%
 \begin{tikzpicture}[>=latex,node distance = 0.5 cm,line width = 0.1pt]
        \node (a)[coordinate] {};
        \node (b) [coordinate, right of  =  a] {};
        \node (a1)[coordinate, above = 0.2 of a] {};
        \node (b1) [coordinate, above = 0.2 of b] {};
        \path(a)edge[thin, ->](b);
        \path (a1)edge[thin, <->](b1);
      \end{tikzpicture}}}}

\graphicspath{{figures/}}

\begin{document}

\title{Constraint-based Causal Discovery from Multiple Interventions over Overlapping Variable Sets}

\author{\name Sofia Triantafillou\thanks{Also in Department of Computer Science, University of Crete.} \email
striant@ics.forth.gr\\
\name Ioannis Tsamardinos$^*$ \email tsamard@ics.forth.gr \\
\addr Institute of Computer Science\\Foundation for Research and Technology - Hellas (FORTH)\\
N. Plastira 100 Vassilika Vouton\\ GR-700 13 Heraklion, Crete, Greece}

\maketitle
\begin{abstract}%
Scientific practice typically involves repeatedly studying a system, each time trying to unravel a different perspective.
In each study, the scientist may take measurements under different experimental conditions (interventions, manipulations,
perturbations) and measure different sets of quantities (variables). The result is a collection of heterogeneous data
sets coming from different data distributions. In this work, we present algorithm \combine, which accepts a collection of
data sets over overlapping variable sets under different experimental conditions; \combine\; then outputs a summary of
all causal models indicating the invariant and variant structural characteristics of all models that simultaneously fit
all of the input data sets. \combine\; converts estimated dependencies and independencies in the data into path
constraints on the data-generating causal model and encodes them as a SAT instance. The algorithm is sound and complete
in the sample limit. To account for conflicting constraints arising from statistical errors, we introduce a general
method for sorting constraints in order of confidence, computed as a function of their corresponding p-values. In our
empirical evaluation, \combine\; outperforms in terms of efficiency the only pre-existing similar algorithm; the latter
additionally admits feedback cycles, but does not admit conflicting constraints which hinders the applicability on real
data. As a proof-of-concept, \combine\; is employed to co-analyze 4 real, mass-cytometry data sets measuring
phosphorylated protein concentrations of overlapping protein sets under 3 different interventions.
\end{abstract}

\section{Introduction}

Causal discovery is an abiding goal in almost every scientific field. In order to discover the causal mechanisms of a
system, scientists typically have to perform a series of experiments (interchangeably: manipulations, interventions, or
perturbations). Each experiment adds to the existing knowledge of the system and sheds light to the sought-after
mechanism from a different perspective. In addition, each measurement may include a different set of quantities
(variables), when for example the technology used allows only a limited number of measured quantities.

However, for the most part, machine learning and statistical methods focus on analyzing a single data set. They are
unable to make joint inferences from the complete collection of available heterogeneous data sets, since each one is
following a different data distribution (albeit stemming from the same system under study). Thus, data sets are often
analyzed in isolation and independently of each other; the resulting knowledge is typically synthesized ad hoc in the
researcher's mind.

The proposed work tries to automate the above inferences. We propose a general, constraint-based algorithm named
\combine\; for learning causal structure characteristics from the integrative analysis of collections of data sets. The
data sets can be heterogeneous in the following manners: they may be measuring different overlapping sets of variables
${\bf O}_i$ under different hard manipulations on a set of observed variables ${\bf I}_i$. A hard manipulation on a
variable $I$, corresponds to a Randomized Controlled Trial \citep{Fisher1922} where the experimentation procedure
completely eliminates any other causal effect on $I$ (e.g.,  randomizing mice to two groups having two different diets;
the effect of all other factors on the diet is completely eliminated).

What connects together the available data sets and allows their co-analysis is the assumption that \emph{there exists a
single underlying causal mechanism that generates the data}, even though it is measured with a different experimental
setting each time. A causal model is plausible as an explanation if it simultaneously fits all data-sets when the effect
of manipulations and selection of measured variables is taken into consideration.

\combine\; searches for the set of causal models that simultaneously fits all available data-sets in the sense given
above. The algorithm outputs a summary network that includes all  the variant and invariant pairwise causal
characteristics of the set of fitting models. For example, it indicates the causal relations upon which all fitting
models agree, as well as the ones for which conflicting explanations are plausible. As our formalism of choice for causal
modeling, we employ Semi-Markov Causal Models (\textbf{SMCMs}). SMCMs \citep{Tian2003} are extensions of Causal Bayesian
Networks (\textbf{CBNs}) that can account for latent confounding variables, but do not admit feedback cycles. Internally,
the algorithm also makes heavy use of the theory and learning algorithm for Maximal Ancestral Graphs (\textbf{MAGs})
\citep{Richardson2002}.

The algorithm builds upon the ideas in \cite{Triantafillou2010} to convert the observed statistical dependencies and
independencies in the data to path constraints on the plausible data generating structures. The constraints are encoded
as a SAT instance and solved with modern SAT engines, exploiting the efficiency of state-of-the-art solvers. However, due
to statistical errors in the determination of dependencies and independencies, conflicting constraints may arise. In this
case, the SAT instance is unsolvable and no  useful information can be inferred. \combine\; includes a technique for
sorting constraints according to confidence:  The constraints are  added to the SAT instance in increasing order of
confidence, and the ones that conflict with the set of higher-ranked constraints are discarded. The technique is general
and the ranking score is a function of the p-values of the statistical tests of independence. It can therefore be applied
to any type of data, provided an appropriate test exists.

\combine\; is empirically compared against a similar, recently developed algorithm by \cite{hyttinen2013discovering}. The
latter is also based on conversion to SAT and is able to additionally deal with cyclic structures, but assumes lack of
statistical errors and corresponding conflicts. It can therefore not be directly applied to typical real problems that
may generate such conflicts. \combine\; proves to be  more efficient than \cite{hyttinen2013discovering} and scales to
larger problem sizes, due to an inherently more compact representation of the path-constraints. The empirical evaluation
also includes a quantification of the effect of sample size, number of data-sets co-analyzed, and other factors on the
quality and computational efficiency of learning. In addition, the proposed conflict resolution technique's superiority
is demonstrated over several other alternative conflict resolution methods.
Finally, we present a proof-of-concept computational experiment by applying the algorithm on 5 heterogeneous data sets
from \cite{Bendall2011} and \cite{Bodenmiller2012} measuring overlapping variable sets under 3 different manipulations.
The data sets measure protein concentrations in thousands of human cells of the autoimmune system using mass-cytometry
technologies. Mass cytometers can perform single-cell measurements with a rate of about 10,000 cells per second, but can
currently only measure up to circa 30 variables per run. Thus, they seem to form a suitable test-bed for integrative
causal analysis algorithms.

The rest of this paper is organized as  follows: Section \ref{sec:rel_work} presents the related literature on learning
causal models and combining multiple data sets. Section \ref{sec:mcg} reviews the necessary  theory of MAGs and SMCMs and
discusses the relation between the two and how hard manipulations are modeled in each. Section \ref{sec:COmbINE} is the
core of this paper, and it is split in three subsections; presenting the conversion to SAT; introducing the algorithm and
proving soundness and completeness; introducing the conflict resolution strategy. Section  \ref{sec:experiments} is
devoted to the experimental evaluation of the algorithm: testing the algorithm's performance in several settings and
presenting an actual case study where the algorithm can be applied. Finally, Section \ref{sec:conclusions} summarizes the
conclusions and proposes some future directions of this work.

\section{Related Work}\label{sec:rel_work}

Methods for causal discovery have been, for the most part, limited to the analysis of a single data set. However, the
great advancement of intervention and data collection technology has led to a vast increase of available data sets, both
observational and experimental. Therefore, over the last few years, there have been a number of works that focus on
causal discovery from multiple sources. Algorithms in that area may differ in the formalism the use to model causality or
in the type of heterogeneity in the studies they co-analyze. In any case, the goal is always to discover the  single
underlying data-generating causal mechanism.

One group of algorithms focuses on combining observational data that measure overlapping variables. \cite{Tillman2008}
and \cite{Triantafillou2010} both provide sound and complete algorithms for learning the common characteristics of MAGs
from data sets measuring  overlapping variables. \cite{Tillman2008} handles conflicts by ignoring conflicting evidence,
while the method presented in \citet{Triantafillou2010} only works with an oracle of conditional independence.
\cite{Tillman2011} present an algorithm for the same task that handles a limited type of conflicts (those conserning
p-values for the same pair of variables stemming from different data sets) by combining the p-values for conditional
independencies that are testable in more than one data sets. \cite{claassen2010learning} present a sound, but not
complete, algorithm for causal structure learning from multiple independence models over overlapping variables by
transforming independencies into a set of causal ancestry rules.

Another line of work deals with learning causal models from multiple experiments. \citet{Cooper1999} use a Bayesian score
to combine experimental and observational data in the context of causal Bayesian networks. \citet{Hauser2012} extend the
notion of Markov equivalence for DAGs to the case of interventional distributions arising from  multiple experiments, and
propose a learning algorithm. \cite{tong2001active} and \cite{murphy2001active}  use Bayesian network theory to propose
experiments that are most informative for causal structure discovery. \cite{eberhardt2007interventions} and
\cite{Eaton2007} discuss how some other types of interventions can be modeled  and used to learn Bayesian networks.
\citet{hyttinen2012alearning} provides an algorithm for learning linear cyclic models from a series of experiments, along
with sufficient and necessary conditions for identifiability. This method admits latent confounders but uses linear
structural equations to model causal relations and is therefore inherently limited to linear relations.
\citet{meganck2006learning}  propose learning SMCMs by learning the Markov equivalence classes of MAGs from observational
data and then designing the experiments necessary to convert it to a SMCM.

Finally, there is a limited number of methods that attempt to co-analyze data sets measuring overlapping variables under
different experimental conditions. In \cite{hyttinen2012causal} the authors extend the methods of
\cite{hyttinen2012alearning} to handle overlapping variables, again under the assumption of linearity.
\citet{hyttinen2013discovering} propose a constraint-based algorithm for learning causal structure from different
manipulations of overlapping variable sets. The method works by transforming the observed \m-connection and \m-separation
constraints into a SAT instance. The method uses a path analysis heuristic to reduce the number of tests translated into
path constraints. Causal insufficiency is allowed, as well as feedback cycles. However, this method cannot handle
conflicts and therefore relies on an oracle of conditional independence. Moreover, the method can only scale up to about
12 variables. \cite{Claassen2010} present an algorithm for learning causal models from multiple experiments;  the
experiments here are not hard manipulations,  but  general experimental conditions, modeled like variables that have no
parents in the graph but can cause other variables in some of the conditions.

To the best of our knowledge, \combine\; is the first algorithm to address both overlapping variables and  multiple
interventions for acyclic structures  without relying on  specific parametric assumptions or requiring an oracle of
conditional independence. While the limits of \combine\; in terms of input size have not been exhaustively checked, the
algorithm scales up to networks of up to 100 variables for relatively sparse networks (maximum number of parents equals
5).

\section{Mixed Causal Models}\label{sec:mcg}
Causally insufficient systems are often described using Semi-Markov causal models (SMCMs) \citep{Tian2003} or Maximal
Ancestral Graphs (MAGs) \citep{Richardson2002}. Both of them are \textbf{mixed graphs}, meaning they can contain both
directed ($\pagrightarrow$)  and bi-directed ($\pagbidir$) edges. We use the term \textbf{mixed causal graph} to denote
both. In this section, we will briefly present their common and unique properties. First, let us review the basic mixed
graph notation:

In a mixed graph $\graph G$, a path  is a sequence of distinct nodes $\langle V_{0}, V_{1},\dots, V_{n} \rangle$ s.t for
$0 \leq i < n$, $V_{i}$ and $V_{i+1}$ are adjacent in $\mathcal{G}$. $X$ is called a \textbf{parent} of $Y$ and $Y$ a
\textbf{child} of $X$ in \graph G if $X\pagrightarrow Y$ in  \graph G. A path from $V_{0}$ to $V_{n}$ is
\textbf{directed} if for $0 \leq i < n$, $V_{i} $ is a parent $V_{i+1}$. \var X is called a \textbf{ancestor} of \var Y
and \var Y is called a \textbf{descendant} of \var X in \graph G if $\var X= \var Y$ in \graph G or there exists a
directed path from $X$ to $Y$ in \graph G. We use the notation $\mathbf{Pa}_{\graph G}(\set X), \mathbf{Ch}_{\graph
G}(\set X), \mathbf{An}_{\graph G}(\set X), \mathbf{Desc}_{\graph G}(\set X)$ to denote the set of parents, children,
ancestors and descendants of nodes \set X in \graph G. A \textbf{directed cycle} in $\mathcal{G}$ occurs when $X
\rightarrow Y \in {\bf E}$  and $Y \in \mathbf{An}_{\graph G}(X)$. An \textbf{almost directed cycle} in $\mathcal{G}$
occurs when $X\leftrightarrow Y\in {\bf E}$ and $Y\in \mathbf{An}_{\graph G}(X)$. Given a path \var p in a mixed graph, a
non-endpoint node \var V on \var p is called a \textbf{collider} if the two edges incident to \var V on \var p are both
into \var V. Otherwise \var V is called a \textbf{non-collider}. A path $p=\langle\var X, \var Y, \var Z\rangle$, where
\var X and \var Z are not adjacent  in \graph G  is called an \textbf{unshielded triple}. If \var Z is a collider on this
path, the triple is called an \textbf{unshielded collider}. A path $p =\langle X\dots W, V, Y\rangle$ is called
\textbf{discriminating} for \var V if \var X is not adjacent to \var Y and every node on the path from \var X to \var V
is a collider and a parent of \var Y.

MAGs and SMCMs are graphical models that represent both  causal relations and conditional independencies among a set of
measured (observed) variables \set O, and can be viewed as  generalizations of causal Bayesian networks that can account
for latent confounders. MAGs can also account for selection bias, but in this work we assume selection bias is not
present.

sufficient. We call this hypothetical extended model the \textbf{underlying causal DAG}.

\subsection{Semi-Markov Causal Models}
Semi-Markov causal models (SMCMs), introduced by  \cite{Tian2003}, often also reported as Acyclic Directed Mixed Graphs
(ADMGs), are causal models that implicitly model hidden confounders using bi-directed edges.
A directed edge  \var X\pagrightarrow\var Y denotes that \var X is a \emph{direct} cause of \var Y in the context of the
variables included in the model. A bi-directed edge \var X \pagbidir \var Y  denotes that \var X and \var Y are
confounded by an unobserved variable. Two variables can be joined by at most two edges, one directed and one
bi-directed.

Semi-Markov causal models are designed to represent marginals of causal Bayesian networks. In DAGs, the probabilistic
properties of the distribution of variables included in the model can be determined graphically using the criterion of
\var d-separation. The natural extension of d-separation to mixed causal graphs is called \m-separation:

\begin{definition}($m$-connection, $m$-separation)
In a mixed graph  $\mathcal{G}=({\bf E}, {\bf V})$, a path $p$ between $A$ and $B$ is \textbf{m-connecting} given
(conditioned on) a set of nodes ${\bf Z}$ , ${\bf Z}\subseteq{\bf V}\setminus\{A, B\}$ if
    \begin{enumerate}
        \item Every non-collider on $p$ is not a member of $\bf{Z}$.
        \item Every collider on the path is an ancestor of some member of $\bf{Z}$.
    \end{enumerate}
    $A$ and $B$ are said to be \textbf{$m$-separated} by \set Z  if there is no $m$-connecting path between $A$ and $B$
    relative to \set Z. Otherwise, we say they are \textbf{$m$-connected} given \set Z.
\end{definition}

Let \graph G be a SMCM over a set of variables \set O, \prob P the joint probability distribution (JPD) over the same set
of variables and \prob J the independence model, defined as the set of conditional independencies that hold in \prob P.
We use $\langle \set X, \set Y |\set Z\rangle$ to denote the conditional independence of variables in \set  X with
variables in \set Y given variables in \set Z. Under the Causal Markov (\textbf{CMC}) and Faithfulness (\textbf{FC})
conditions \citep{Spirtes2000}, \emph{every $m$-separation present in  \graph G corresponds to a conditional independence
in $\mathcal{J}$ and vice-versa.}

In causal Bayesian networks, every missing edge in \graph G corresponds to a conditional independence in \prob J, meaning
there exists a subset of the variables in the model that renders the two non-adjacent variables independent.
Respectively, every conditional independence in \prob J corresponds to a missing edge in the DAG \graph G. This is not
always true for SMCMs. Figure \ref{fig:maximal} illustrates an example of a SMCM where two non-adjacent variables are not
independent given any subset of observed variables.

\citet{Evans2010, Evans2011}   deal  with the factorization and parametrization of SMCMs for discrete variables. Based on
this parametrization,  score-based methods have also recently been explored \citet{Richardson2012, Shpitser2013}, but are
still limited to small sets of discrete variables. To the best of our knowledge,  there exists no constraint-based
algorithm for learning the structure of SMCMs, probably due to the fact that the lack of conditional independence for a
pair of variables does not necessarily mean non-adjacency. \cite{Richardson2002} overcome this obstacle by introducing a
causal mixed graph with slightly different semantics, the maximal ancestral graph.

\subsection{Maximal Ancestral Graphs}
Maximal ancestral graphs (MAGs) are \textbf{ancestral} mixed graphs, meaning that they contain no directed or almost
directed cycles. Every pair of variables \var X, \var Y in an ancestral graph is joined by at most one edge. The
orientation of this edge  represents (non) causal ancestry: A bi-directed edge \var X\pagbidir\var Y  denotes that \var X
does not cause \var Y and \var Y does not cause \var X, but (under the faithfulness assumption) the two share a latent
confounder.  A directed edge \var X \pagrightarrow \var Y denotes causal ancestry: \var X is a \emph{causal ancestor} of
\var Y. Thus, if  \var X causes \var Y (not necessarily directly in the context of observed variables) and they are also
confounded, there is an edge \var X\pagrightarrow \var Y in the corresponding MAG. Undirected edges can also be present
in MAGs that account for selection bias. As mentioned above, we assume no selection bias in this work and the theory of
MAGs presented here is restricted to MAGs with no undirected edges.

Like SMCMs, ancestral graphs are also designed to represent marginals of causal Bayesian networks. Thus, under the causal
Markov and faithfulness conditions, \var X and \var Y are \m-separated given \set Z in an ancestral graph  \graph M if
and only if  $\langle X, Y |\set Z\rangle$ is in the corresponding independence model \prob J. Still, like in SMCMs, a
missing edge does not necessarily correspond to a conditional independence. The following definition describes a subset
of ancestral graphs in which every missing edge (non-adjacency) corresponds to a conditional independence:

\begin{definition}(Maximal Ancestral Graph, MAG)\citep{Richardson2002}
A mixed graph is called \emph{ancestral} if it contains no directed and almost directed cycles. An ancestral graph
$\mathcal{G}$ is called \emph{maximal} if for every  pair of non-adjacent nodes $(X, Y)$, there is a (possibly empty) set
$\mathbf{Z}$, $X, Y \notin \mathbf{Z}$ such that $\langle X, Y | \mathbf{Z}\rangle \in \mathcal{J}$.
\end{definition}

Figure \ref{fig:maximal} illustrates an ancestral graph that is not maximal, and the corresponding maximal ancestral
graph. MAGs  are closed under marginalization \citep{Richardson2002}. Thus, if \graph G is a MAG faithful to \prob P,
then there is a unique MAG \graph G$'$ faithful to any marginal distribution of \prob P.

\begin{figure}[t!]\centering
\begin{tabular}{cc}
\includegraphics[width = 0.2\columnwidth]{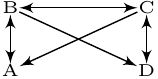}\hspace{0.15\columnwidth}&
\includegraphics[width = 0.2\columnwidth]{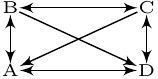}\\
(a) & (b)
\end{tabular}
\caption{\textbf{Maximality and primitive inducing paths}.(a) Both  (i) a semi Markov causal model over variables
$\{A,\;B,\;C,\;D\}$. Variables \var A and \var D are \m-connected given any subset of observed variables, but they do not
share a direct relationship in the context of observed variables and (ii) a non-maximal ancestral graph over variables
$\{A,\;B,\;C,\;D\}$. (b) The corresponding MAG. \var A and \var D are adjacent, since they cannot be \m-separated given
any subset of $\{\var B, \var C\}$. Path $\langle A, B, C, D\rangle$ is a primitive inducing path. This example was
presented in \cite{Zhang2008b}. \label{fig:maximal}}
\end{figure}

We use \marg {\set L} to denote the act of marginalizing out variables \set L, thus, if  \graph G  is a MAG  over
variables $\set O \cup \set L$ faithful to a joint probability distribution \prob P, \graph G\marg{\set L} is the MAG
over \set O faithful to the marginal joint probability distribution \prob P\marg{\set L}. Obviously, the DAG of a causal
Bayesian network is also a MAG. For a MAG \graph G over \set O and a set of variables $\set L \subset \set O$, the
marginal MAG \graph G\marg{\set L} is defined as follows:

\begin{definition}\label{def:margMAG} \citep{Richardson2002} MAG \graph G\marg{\set L} has node set $\set O\setminus \set
L$ and edges specified as follows:
If \var X, \var Y are s.t. $\forall \set Z\subset \set O\setminus \set L\cup\{X, Y\}$, \var X and \var Y are \m-connected
given \set Z in \graph G, then
\[\textnormal{ if } \left\{
\begin{array}{rl}
X \notin \mathbf{An}_{\graph G}(Y); Y \notin \mathbf{An}_{\graph G}(X)\\
X \in \mathbf{An}_{\graph G}(Y); Y \notin \mathbf{An}_{\graph G}(X)\\
X \notin \mathbf{An}_{\graph G}(Y); Y \in \mathbf{An}_{\graph G}(X)
\end{array} \right\}\textnormal{ then } \left\{
\begin{array}{rl}
X \leftrightarrow Y \\
X \rightarrow Y\\
X \leftarrow Y
\end{array} \right\} \textnormal{ in }\graph G\marg{\set L}\]

\end{definition}

As mentioned above, every conditional independence in an independence model \graph J corresponds to a missing edge in the
corresponding faithful MAG \graph G. Conversely, if \var X and \var Y are dependent given every subset of observed
variables, then \var X and \var Y are adjacent in \graph G. Thus, given an oracle of conditional independence it is
possible to learn the skeleton of a MAG \graph G over variables \set O from a data set. Still, some of the orientations
of \graph G are not distinguishable by mere observations. The set of MAGs \graph G  faithful to distributions \prob  P
that entail a set of conditional independencies form a \textbf{Markov equivalence class}. The following result was proved
in \cite{Spirtes1996}:

\begin{proposition}Two MAGs over the same variable set are Markov equivalent if and only if:
\begin{enumerate}\item They share the same edges.
\item They share the same unshielded colliders.
\item If a path \var p is discriminating for a node \var V in both graphs, \var V is a collider on the path on one
    graph if and only if it is a collider on the path on the other.
\end{enumerate}
\end{proposition}

We use [\graph G] to denote the class of MAGs that are Markov equivalent to \graph G. A \textbf{partial ancestral graph
(PAG)} is a representative graph of this class, and has the skeleton shared by all the graphs in [\graph G], and all the
orientations invariant in all the graphs in  [\graph G]. Endpoints that can be either arrows or tails in different MAGs
in \graph G are denoted with a circle ``$\circ$" in the representative PAG. We use the symbol
\raisebox{1pt}{\begin{tikzpicture}[>=latex, node distance = 0.5 cm, line width = 0.1pt] \node (a)[star, star point
ratio=4, minimum size=4pt,inner sep =0.01pt, draw] {};\end{tikzpicture}}  as a wildcard to denote any of the three marks.
We use the notations  $\graph M\in \graph P$ to denote that MAG \graph M belongs to the Markov equivalence class
represented by PAG \graph P, and we use the notation $\graph M\in \graph J$ to denote that MAG \graph M is faithful to
the conditional independence model \graph J. \textbf{FCI} Algorithm \citep{Spirtes2000, Zhang2008} is a sound and
complete algorithm for learning the complete (maximally informative) PAG of the MAGs faithful to a distribution \prob P
over variables \set O in which a set of conditional independencies \prob J hold.  An important advantage of FCI is that
it employs CMC, faithfulness and some graph theory to reduce the number of tests required to identify the correct PAG.

\subsection{Correspondence between SMCMs and MAGs}
Semi Markov Causal Models and Maximal Ancestral Graphs both represent causally insufficient causal structures, but they
have some significant differences.
While they both entail the conditional independence and causal ancestry structure of the  observed variables, SMCMs
describe the causal relations among observed variables, while MAGs encode independence structure with partial causal
ordering. Edge semantics in SMCMs are closer to the semantics of causal Bayesian networks, whereas edge semantics in MAGs
are more complicated. On the other hand, unlike in DAGs and MAGs, a missing edge in a SMCM does not necessarily
correspond to a conditional independence (SMCMs do not obey a pairwise Markov property).

Figure \ref{fig:mixedcausalmodels} summarizes the main differences  of SMCMs and MAGs. It shows two different DAGs, and
the corresponding marginal SMCMs and MAGs over four observed variables. SMCMs have a many-to-one relationship to MAGs:
For a MAG \graph M, there can exist more than one SMCMs that entail the same probabilistic and causal ancestry relations.
On the other hand, for any given SMCM there exists only one MAG  entailing the same probabilistic and causal ancestry
relations. This is clear in Figure \ref{fig:mixedcausalmodels}, where a unique MAG, $\graph M_1=\graph M_2$ entails the
same information as two different SMCMs,  $\graph S_1$ and $\graph S_2$ in the same figure.

Directed edges in a SMCM denote a causal relation that is \emph{direct} in the context of observed variables. In
contrast, a directed edge in a MAG merely denotes causal ancestry; the causal relation is not necessarily direct. An edge
\var X\pagrightarrow \var Y can be present in a MAG even though \var X does not directly causes \var Y; this happens when
\var X is a causal ancestor of \var Y and the two cannot be rendered independent given any subset of observed variables.
Depending on the structure of latent variables, this edge can be either missing or bi-directed in the respective SMCM.

In Figure \ref{fig:mixedcausalmodels} we can see examples of both cases. For example, \var A  is a causal ancestor of
\var D in DAG $\graph G_1$, but not a direct cause (in the context of observed variables). Therefore, the two are not
adjacent in the corresponding  SMCM $\graph S_1$ over $\{A, B, C, D\}$. However, the two cannot be rendered independent
given any subset of $\{B, C\}$, and therefore \var A\pagrightarrow \var D  is in the respective MAG $\graph M_1$.

\begin{figure}[t!]\centering
\begin{tabular}{cccc}
\includegraphics[width = 0.22\columnwidth]{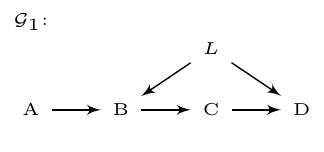}&
\includegraphics[width = 0.22\columnwidth]{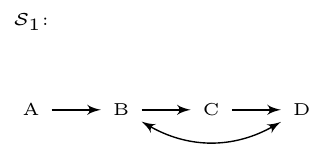}&
\includegraphics[width = 0.22\columnwidth]{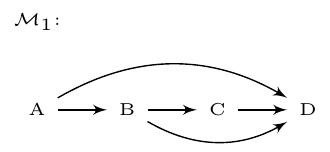}&
\includegraphics[width = 0.22\columnwidth]{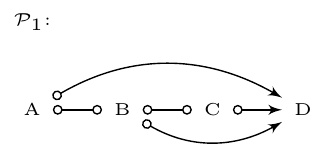}\\
\includegraphics[width = 0.22\columnwidth]{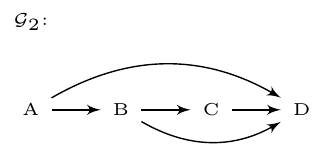}&
\includegraphics[width = 0.22\columnwidth]{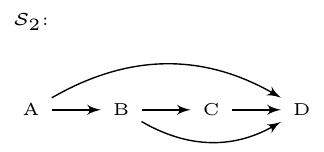}&
\includegraphics[width = 0.22\columnwidth]{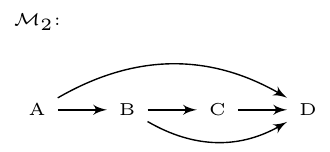}&
\includegraphics[width = 0.22\columnwidth]{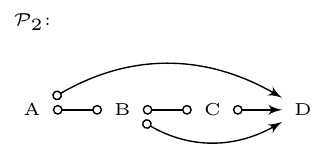}\\
\end{tabular}
\caption{\textbf{An example two different DAGs and the corresponding mixed causal graphs over observed variables}. On the
right we can see DAGs $\graph G_1$ over variables $\{A,\;B,\;C,\;D,\;L\}$ (top) and $\graph G_2$ over variables
$\{A,\;B,\;C,\;D\}$ (bottom). From left to right, on the same row as the underlying causal DAG, we can see the respective
SMCMs $\graph S_1$ and $\graph S_2$ over $\{A,\;B,\;C,\;D\}$; the respective MAGs $\graph M_1=\graph G_1\marg{L}$ and
$\graph M_2=\graph G_2$ over variables $\{A,\;B,\;C,\;D\}$; finally, the respective PAGs $\graph P_1$ and $\graph P_2$.
Notice that, $\graph M_1$ and $\graph M_2$ are identical, despite representing different underlying causal
structures.\label{fig:mixedcausalmodels}}
\end{figure}
\begin{figure}[h!]\centering
\begin{tabular}{cccc}
\includegraphics[width = 0.22\columnwidth]{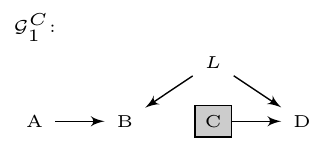}&
\includegraphics[width = 0.22\columnwidth]{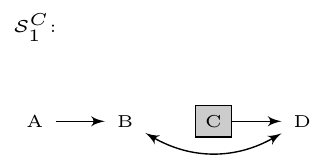}&
\includegraphics[width = 0.22\columnwidth]{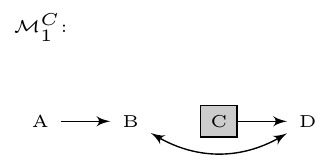}&
\includegraphics[width = 0.22\columnwidth]{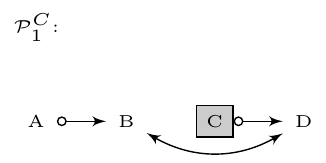}\\
\includegraphics[width = 0.22\columnwidth]{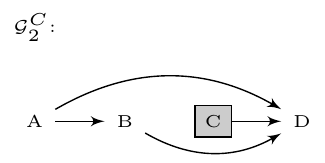}&
\includegraphics[width = 0.22\columnwidth]{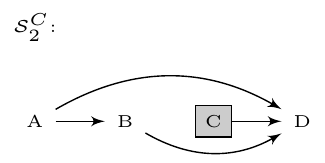}&
\includegraphics[width = 0.22\columnwidth]{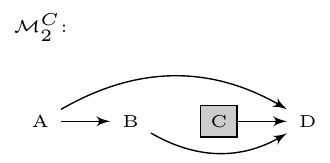}&
\includegraphics[width = 0.22\columnwidth]{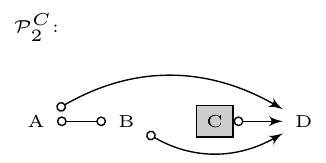}\\
\end{tabular}
\caption{\textbf{Effect of manipulating variable \var C on the causal graphs of Figure \ref{fig:mixedcausalmodels}}. From
right to left we can see the manipulated DAGs  $\graph G_1\manip{C}$ (top) and  $\graph G_2\manip{C}$ (bottom), the
manipulated SMCMs $\graph S_1\manip{C}$ (top) and $\graph S_2\manip{C}$ (bottom) over variables $\{A,\;B,\;C,\;D\}$, the
manipulated MAGs  $\graph M_1\manip{C} =\graph G_1\manip{C}\marg{L}$ (top) and $\graph M_2\manip{C} =\graph G_2\manip{C}$
(bottom) over the same set of variables, and the corresponding PAGs $\graph P_1\manip{C}$ (top) and $\graph P_2\manip{C}$
(bottom). Notice that edge $\var A\protect\pagrightarrow \var D$ is removed in $\graph M_1\manip{C}$, even though it is
not adjacent to the manipulated variable. Moreover, on the same graph,  edge  $\var B\protect\pagrightarrow \var D$ is
now $\var B\protect\pagbidir \var D$.\label{fig:manipulatemixedcausalmodels}}
\end{figure}

On the same DAG, \var B is another causal ancestor (but not a direct cause) of \var D. The two variables share the common
cause \var L. Thus, in the corresponding SMCM $\graph S_1$ over $\{A, B, C, D\}$ we can see the edge \var B\pagbidir\var
D. However, a bi-directed edge between \var B and \var D is not allowed in MAG $\graph M_1$, since it would create an
almost directed cycle. Thus, \var B\pagrightarrow \var D is in   $\graph M_1$.

We must also note that unlike SMCMs, MAGs only allow one edge per variable pair. Thus, if \var X directly causes \var Y
and the two are also confounded, both edges will be in a relevant SMCM (\var X\smrightbidir\var Y), while the two will
share a directed edge from \var X to \var Y in the corresponding MAG.

Overall, a SMCM has a subset of adjacencies (but not necessarily edges) of its MAG counterpart. These extra adjacencies
correspond to pairs of variables that cannot be \m-separated given any subset of observed variables, but neither directly
causes the other, and the two are not confounded. These adjacencies can be checked in a SMCM using a special type of
path, called \textbf{inducing path} \citep{Richardson2002}.

\begin{definition}(inducing path)\label{def:inducing_path}
A path $\var p =\langle V_1, V_2, \dots, V_n\rangle$ on a mixed causal graph \graph G over a set of variables $\set V=
\set O\dotcup \set L$ is called \textbf{inducing} with respect to \set L if every non-collider on the path is in \set L
and every collider is an ancestor of either \var {V_1} or \var {V_n}. A path that is inducing with respect to the empty
set is called a \textbf{primitive} inducing path.
\end{definition}

Obviously, an edge joining \var X and \var Y is a primitive inducing path. Intuitively, an inducing path with respect to
\set L is \m-connecting given  any subset of variables that does  not include variables in \set L. Path \var
A\pagrightarrow\var B\pagleftarrow \var L \pagrightarrow \var D is an inducing path with respect to \var L in $\graph
G_1$ of Figure \ref{fig:mixedcausalmodels}, and \var A\pagrightarrow\var B\pagbidir \var D  is an inducing path with
respect to the empty set in $\graph S_1$ of the same figure. Inducing paths are extensively discussed in
\cite{Richardson2002}, where the following theorem is proved:

\begin{theorem}\label{the:indpaths}If \graph G is an ancestral graph over variables $\set V = \set O\dotcup\set L$, and
$\var X, \var Y\in \set O $ then the following statements are equivalent:
\begin{enumerate}
\item \var X and \var Y are adjacent in \graph G\marg{\set L}.
\item There is an inducing path with respect to \set L in \graph G.
\item $\forall \set Z$, $\set Z\subseteq \set V\setminus\set L\cup\{X, Y\}$, \var X and \var Y are \m-connected given
    \set Z in \graph G.
\end{enumerate}
\end{theorem}

\begin{proof}See proof of Theorem 4.2 in \cite{Richardson2002}.\end{proof}

This theorem links inducing paths in an ancestral graph to \m-separations in the same graph and to adjacencies in any
marginal ancestral graph. The equivalence of (ii) and (iii) can also be proved for SMCMs, using the proof presented in
\cite{Richardson2002} for Theorem \ref{the:indpaths}:

\begin{theorem}\label{the:indpathssmcm}If \graph G is a SMCM  over variables $\set V = \set O\dotcup\set L$, and $\var X,
\var Y\in \set O $ then the following statements are equivalent:
\begin{enumerate}
\item There is an inducing path with respect to \set L in \graph G.
\item $\forall \set Z$, $\set Z\subseteq \set V\setminus\set L\cup\{X, Y\}$, \var X and \var Y are \m-connected given
    \set Z in \graph G.
\end{enumerate}
\end{theorem}

Primitive inducing paths are connected to the  notion of maximality in ancestral graphs: Every ancestral graph can be
transformed into a maximal ancestral graph with the addition of a finite number of bi-directed edges. Such edges are
added between variables $\var X, \var Y$ that are m-connected through a \textbf{primitive inducing path}
\citep{Richardson2002}. Path $\var A\pagbidir\var B\pagbidir\var C\pagbidir\var D$ in Figure \ref{fig:maximal} is an
example of a primitive inducing path.

Inducing paths are crucial in this work because adjacencies and non-adjacencies in marginal ancestral graphs can be
translated into existence or absence of inducing paths in causal graphs that include some additional variables. For
example, path $\var A\pagrightarrow\var B \pagleftarrow \var{L}\pagrightarrow\var D$ is an inducing path w.r.t. $\var L$
in $\graph G_1$ in Figure \ref{fig:mixedcausalmodels}, and therefore \var A and \var D are adjacent in $\graph M_1$.
Thus, inducing paths are useful for combining causal mixed graphs over overlapping variables.

Inducing paths  are also necessary to decide whether two variables in an SMCM will be adjacent in a MAG over the same
variables without having to check all possible \m-separations. Algorithm \ref{algo:SMCMtoMAG} describes how to turn a
SMCM into a MAG over the same variables. To prove the algorithm's soundness, we first need to prove the following:

\begin{proposition}\label{prop:mcmindpaths}Let \set O be a set of variables and \graph J the independence model over \set
V. Let \graph S be a SMCM over variables \set V that is faithful to \prob J and \graph M be the MAG over the same
variables that is faithful to \prob J. Let  $\var X, \var Y\in \set O$. Then there is an inducing path between $\var X$
and $\var Y$ with respect to $\set L$, $\set L\subseteq\set V$  in \graph S if and only if there is an inducing path
between $\var X$ and $\var Y$ with respect to \set L in \graph M.
\end{proposition}

\begin{proof}See Appendix \ref{app:proofs}..\end{proof}

\IncMargin{1em}
\begin{algorithm}[t!]
\SetKwInOut{Input}{input}\SetKwInOut{Output}{output}
\Input{SMCM \graph S}
\Output{MAG $\graph M$}
\BlankLine
\graph M$\leftarrow$\graph S\;
\ForEach{ordered pair of variables \var X, \var Y not adjacent in \graph S}{
\If{$\exists$ primitive inducing path from \var X to \var Y in \graph S}{
\uIf {$\var X\in \mathbf{An}_{\graph S}(Y)$}{add \var X\pagrightarrow \var Y to \graph M\;
}
\uElseIf {$\var Y\in \mathbf{An}_{\graph S}(X)$}{add \var Y\pagrightarrow \var X to \graph M\;
}
\Else{add \var Y\pagbidir \var X to \graph M\;
}
}
}
\ForEach{\var X\smrightbidir\var Y in \graph M}{remove \var X\pagbidir \var Y\;}
\caption{SMCMtoMAG}\label{algo:SMCMtoMAG}
\end{algorithm}
\DecMargin{1em}

Algorithm \ref{algo:SMCMtoMAG} takes as input a SMCM and adds the necessary edges to transform it into a MAG by looking
for primitive inducing paths.
The soundness of the algorithm is a direct consequence of Proposition \ref{prop:mcmindpaths}. The inverse procedure,
converting a MAG into the  underlying SMCM, is not possible, since we cannot know in general which of the edges
correspond to direct causation or confounding and which are there because of a (non-trivial) primitive inducing path.
Note though that, there exist sound and complete algorithms that identify all edges for which such a determination is
possible \citep{Borboudakis2012}. In addition, we later show that co-examining manipulated distributions can  indicate
that some edges stand for indirect causality (or indirect confounding).

\subsection{Manipulations under causal insufficiency}\label{sec:mcgmanip}

An important motivation for using causal models is to predict causal effects. In this work, we focus on hard
manipulations, where the value of the manipulated variables is set exclusively  by the manipulation procedure. We also
adopt the assumption of locality, denoting that the intervention of each manipulated variable should not directly affect
any variable other than its direct target, and more importantly, local mechanisms for other variables should remain the
same as before the intervention \citep{zhang2006causal}. Thus, the intervention is merely a local surgery with respect to
causal mechanisms. These assumptions may seem a bit restricting, but this type of experiment is fairly common in several
modern fields where the technical capability for  precise interventions is available, such as, for example, molecular
biology. Finally, we assume that the manipulated model is faithful to the corresponding manipulated distributions.

In the context of causal Bayesian networks, hard interventions are modeled using what is referred to as ``graph surgery",
in which all edges incoming to the manipulated variables are removed from the graph. The resulting graph is referred to
as the \textbf{manipulated graph}. Parameters of the distribution that refer to the probability of manipulated variables
given their parents are replaced by the parameters set by the manipulation procedure, while all other parameters remain
intact. Naturally, DAGs are closed under manipulation. We use the term \textbf{intervention target} to denote a set of
manipulated variables. For a DAG \graph D and an intervention target \set I, we use \graph D\manip{\set I} to denote the
manipulated DAG. The same notation (the intervention targets as a superscript) is used to denote a manipulated
independence model.

Graph surgery can be easily extended to SMCMs: One must simply remove edges into the manipulated variables. Again, we use
the notation \graph S\manip{\set I} to denote the graph resulting from  a SMCM \graph S after the manipulation of
variables in \set I. On the contrary, predicting the effect of manipulations in MAGs is not trivial. Due to the
complicated semantics of the edges, the manipulated graph is usually not unique.

This becomes more obvious by looking at Figures \ref{fig:mixedcausalmodels} and \ref{fig:manipulatemixedcausalmodels}.
Figure \ref{fig:mixedcausalmodels} shows two different causal DAGs and the corresponding SMCMs and MAGs, and Figure
\ref{fig:manipulatemixedcausalmodels} shows the effect of a manipulation on the same graphs. In Figure
\ref{fig:mixedcausalmodels} the marginals DAGs $\graph  D_1$ and $\graph D_2$ are represented by the same MAG $\graph
M_1$ =$\graph M_2$. However, after manipulating variable \var C, the resulting manipulated MAGs $\graph M_1\manip{C}$ and
$\graph M_2\manip{C}$ do not belong to the same equivalence class (they do not even share the same skeleton). We must
point out, that the indistinguishability of $\graph M_1$ and $\graph M_2$ refers to \m-separation only; the absence of a
direct causal edge between \var A and  \var D could be detected using other types of tests, like the Verma constraint
\citep{Verma1991}.

While we cannot predict the effect of manipulations on a MAG \graph M, given a data set measuring variables \set O when
variables in $\set I\subset \set O$ are manipulated, we can obtain (assuming an oracle of conditional independence) the
PAG representative of the actual manipulated MAG \graph M\manip{\set I}. We use  \graph P\manip{\set I} to denote this
PAG. Moreover, by observing PAGs  $\{\graph P\manip{\set I_i}\}_i$ that stem from different manipulations of the same
underlying distribution, we can infer some more refined information for the underlying causal model.

Let's suppose, for example, that $\graph G_1$ in Figure \ref{fig:mixedcausalmodels}  is the true underlying causal graph
for  variables $\{A, B, C, D, L\}$ and that we have the learnt PAGs $\graph P_1\manip{A}$ and $\graph P_1\manip{C}$ from
relevant data sets. Graph $\graph P_1\manip{A}$ is not shown, but is identical to $\graph P_1$ in Figure
\ref{fig:mixedcausalmodels} since  \var A has no incoming edges in the underlying DAG (and SMCM). $\graph P_1\manip{C}$
is illustrated in Figure \ref{fig:manipulatemixedcausalmodels}. Edge \var A \doublecirc \var D is present in $\graph
P_1\manip A$, but is missing in $\graph P_1\manip{C}$ even though neither \var A nor \var D are manipulated in $\graph
P_1\manip{C}$. By reasoning on the basis of both graphs, we can infer that edge \var A \pagrightarrow \var D in $\graph
P_1\manip{A}$ cannot denote a \emph{direct} causal relation among the two variables, but must be the result of a
primitive, non-trivial inducing path.

\section{Learning causal structure from multiple data sets measuring overlapping variables under different
manipulations\label{sec:COmbINE}}
In the previous section we described the effect of manipulation on MAGs and saw an example of how co-examining  PAGs
faithful to different manipulations of the same underlying distribution can help classify an edge between two variables
as  not direct.

In this section, we expand this idea and present a general, constraint-based algorithm for learning causal structure from
overlapping manipulations. The algorithm takes as input a set of data sets measuring overlapping variable sets $\{\set
O_i\}_{i=1}^N$; in each data set, some of the observed variables can be manipulated. The set of manipulated variables in
data set $i$ is also provided and is denoted with $\set I_i$.

We assume that there exists an underlying causal mechanism over the union of observed variables $\set O = \bigcup_i{\set
O_i}$ that can be described with a probability distribution \prob P over \set O and a semi Markov causal model \graph S
such that \prob P and \graph S are faithful to each other. We denote with $\prob J$ the independence model of \prob P.
Every manipulation is then performed on \graph S and only on variables observed in the corresponding data set. In
addition, we assume Faithfulness holds for the manipulated graphs as well. The data are then sampled from the manipulated
distribution. In each data set $i$, the set $\set L_i = \set O\setminus\set O_i$ is latent. We denote the independence
model of each data set $i$ as $\graph J_i \equiv \graph J\manip{\set I_i}\marg{\set L_i}$. We now define the following
problem:

\begin{definition}[Identify a consistent SMCM]\label{def:probmain}

Given sets $\{\set O_i\}_{i=1}^N$, $\{\set I_i\}_{i=1}^N$, and $\{\prob J_i\}_{i=1}^N$ identify a SMCM \graph S, such
that:
\begin{equation*}\graph M_i \in \prob J_i, \;\forall i\textnormal{ where } \graph M_i = \textnormal{SMCMtoMAG}(\graph
S\manip{\set I_i})\marg{\set L_i}
\end{equation*}
that is, $\graph M_i$ is the MAG corresponding to the manipulated marginal of \graph S, for each data set $i$. We call
such a graph \graph S a \textbf{possibly underlying SMCM} for $\{\graph J_i\}_{i=1}^N$.
\end{definition}

We present an algorithm that converts the problem above into a satisfiability instance s.t. a SMCM is consistent iff it
corresponds to a truth-setting assignment of the SAT instance. Notice that, an independence model \graph J corresponds to
a PAG \graph P over the same variables when they represent the same Markov equivalence class of MAGs. Thus, in what
follows we use the corresponding set of manipulated marginal PAGs $\{\graph P_i\}_{i=1}^N$ instead of the independence
models $\{\graph J_i\}_{i=1}^N$. Notice that, PAGs $\{\graph P_i\}_{i=1}^N$ can be learnt with a sound and complete
algorithm such as FCI.

In the following section, we discuss converting the problem presented above into a constraint satisfaction problem.

\subsection{Conversion to SAT}

\begin{figure}[p!]
\resizebox{\columnwidth}{!}{
\fbox{
\begin{minipage}{\columnwidth}
\textbf{Formulae relating properties of observed PAGs to the underlying SMCM \graph S:}
\begin{flalign*}
\mathbf{adjacent(X, Y, \graph P_i)}\leftrightarrow\exists p_{XY}: inducing(p_{XY},i)
\end{flalign*}\vskip -0.5cm
\begin{flalign*}
\begin{multlined}
\mathbf{unshielded\_dnc}(X, Y, Z,\graph P_i)\rightarrow\\
unshielded(\langle X, Y, Z\rangle , \graph P_i)\wedge(ancestor(Y, X, i) \vee ancestor(Y, Z, i))\big]
\end{multlined}
\end{flalign*}
\begin{flalign*}
\begin{multlined}
\mathbf{discriminating\_dnc}(\langle W, \dots, X, Y, Z\rangle, Y, \graph P_i)\rightarrow\\
(discriminating(\langle W, \dots, X, Y, Z\rangle, Y, \graph P_i)\wedge ancestor(Y, X,i) \vee ancestor(Y, Z, i))
\end{multlined}
\end{flalign*}\vskip -0.5cm
\begin{flalign*}
\begin{multlined}
\mathbf{unshielded\_collider}(X, Y, Z,\graph P_i)\rightarrow\\
unshielded(\langle X, Y, Z\rangle , \graph P_i)\wedge(\neg ancestor(Y, X, i) \wedge \neg ancestor(Y, Z, i))
\end{multlined}
\end{flalign*}
\begin{flalign*}
\begin{multlined}
\mathbf{disriminating\_collider}(\langle W, \dots, X, Y, Z\rangle, Y, \graph P_i)\rightarrow\\
(discriminating(\langle W, \dots, X, Y, Z\rangle, Y, \graph P_i)\wedge (\neg ancestor(Y, X,i) \wedge \neg ancestor(Y, Z,
i)))
\end{multlined}
\end{flalign*}\vskip -0.5cm
\begin{flalign*}
\begin{multlined}
\mathbf{unshielded}(\langle X, Y, Z\rangle, \graph P_i)\leftrightarrow \\adjacent(X, Y, \graph P_i)\wedge adjacent(Y,Z,
\graph P_i) \wedge\neg adjacent(X, Z, \graph P_i)
\end{multlined}
\end{flalign*}\vskip -0.5cm
\begin{flalign*}
\begin{multlined}
\mathbf{discriminating}(\langle V_0, V_1, \dots,V_{n-1}, V_n, V_{n+1} \rangle, V_n, \graph P_i)\leftrightarrow\\
\shoveleft[1cm]\forall j\big [ V_j\not \in \set I_i\wedge adjacent(V_{j-1}, Y, \graph P_i)\wedge ancestor(V_j,
V_{n+1},i)\wedge\\
adjacent(V_{j-1}, V_j, \graph P_i)\wedge\neg ancestor(V_j, V_{j-1},i)\wedge \neg ancestor(V_{j-1},V_j,i)
\big] \\
\end{multlined}
\end{flalign*}
\textbf{Formulae reducing path properties of  \graph S to the core variables:
}\begin{flalign*}
\begin{multlined}
\mathbf{inducing}(p_{XY}, i)\leftrightarrow\\
\shoveleft[1cm]\forall j\; V_j\not \in \set I_i \wedge(X\in \set I_i\rightarrow tail(V_1, X))\wedge(Y\in \set
I_i\rightarrow tail(V_n, Y))\wedge\\
 (|p_{XY}|=2\rightarrow edge(X, Y))\wedge(|p_{XY}|>2\rightarrow\forall j\; unblocked(\langle V_{j-1}, V_j,
 V_{j+1}\rangle, X, Y, i))
\end{multlined}
\end{flalign*}\vskip -0.5cm
\begin{flalign*}
\begin{multlined}
\mathbf{unblocked}(\langle Z, V, W\rangle, X, Y, i)\leftrightarrow\\
\shoveleft[1cm] edge(Z, V)\wedge edge(V, W) \wedge\\
\shoveleft[1cm][V\in\set L_i \rightarrow \neg head2head(Z, V, W, i)\vee ancestor(V, X, i)\vee ancestor(V, Y, i)]\wedge\\
\shoveleft[1cm][V\not\in\set L_i \rightarrow head2head(Z, V, W, i)\wedge(ancestor(V, X, i)\vee ancestor(V, Y, i))]
\end{multlined}
\end{flalign*}\vskip -0.5cm
\begin{flalign*}
\mathbf{head2head}(X, Y, Z, i)\leftrightarrow Y\not \in \set I_i\wedge arrow (X, Y)\wedge arrow(Z, Y)
\end{flalign*}\vskip -1cm
\begin{flalign*}
\mathbf{ancestor}(X, Y, i)\leftrightarrow\exists p_{XY}: ancestral(p_{XY}, i)
\end{flalign*}\vskip -1cm
\begin{flalign*}
\begin{multlined}
\mathbf{ancestral}(p_{XY}, i)\leftrightarrow\\
\shoveleft[1cm] \forall j\big[ V_j\not \in \set I_i\wedge(edge(V_{j-1}, V_j)\wedge tail(V_j, V_{j-1})\wedge
arrow(V_{j-1}, V_j))\big]
\end{multlined}
\end{flalign*}\vskip -1cm
\end{minipage}
}
}
\caption{\label{fig:sat_constraints}Graph properties expressed  as boolean formulae using  the variables $edge$, $arrow$
and $tail$. In all equations, we use $\var p_{XY}$ to denote a path of length n+2 between \var X and \var Y in \graph S:
$p_{XY} =\langle V_0=X, V_1, \dots V_j, \dots V_n, V_n+1=Y\rangle$. Index i is used to denote experiment i, where
variables $\set L_i$ are latent and variables $\set O_i$ are manipulated. Conjunction and disjunction are assumed to have
precedence over implication with regard to bracketing.}
\end{figure}

Definition \ref{def:probmain} implies that each  $\graph M_i$ has the same edges (adjacencies), the same unshielded
colliders and the same discriminating colliders as $\graph P_i$, for all \var i. We impose these constraints on \graph S
by converting them to a SAT instance. We express the constraints in terms of the following \textbf{core} variables,
denoting edges and orientation orientations in any consistent SMCM \graph S.

\begin{itemize}
\item edge(\var X, \var Y): true if \var X and \var Y are adjacent in \graph S, false otherwise.
\item tail(\var X, \var Y): true if there exists an edge between \var X and \var Y in \graph S that is out of \var Y,
    false otherwise.
\item arrow(\var X, \var Y): true if there exists an edge between \var X and \var Y in \graph S that is into \var Y,
    false otherwise.
\end{itemize}

Variables tail and arrow are not mutually exclusive, enabling us to represent $X \smrightbidir \; Y$ edges when
$tail(\var Y, \var X) \wedge arrow(\var Y, \var X)$. Each independence model $\graph J_i$ is entailed by the (non)
adjacencies and (non) colliders in each observed PAG $\graph P_i$. These structural characteristics correspond to paths
in any possibly underlying SMCM as follows:

\begin{enumerate}
\item $\forall\var X, \var Y\in \set O_i$,  \var X and \var Y are adjacent in $\graph P_i$ if and only if there
    exists  an inducing path between \var X and \var Y with respect to \set{L_i} in $\graph S\manip{\set I_i}$ (by
    Theorems \ref{the:indpaths} and \ref{the:indpathssmcm} and Proposition \ref{prop:mcmindpaths}).
\item If $\langle X, Y, Z\rangle$ is an unshielded definite non collider in $\graph P_i$, then $\langle X, Y,
    Z\rangle$ is an unshielded triple in $\graph P_i$ and  \var Y is an ancestor of either \var X or \var Z in
    $\graph S\manip{\set I_i}$ (by the semantics of edges in MAGs).
\item If$\langle X, Y, Z\rangle$  is an unshielded collider in $\graph P_i$, then $\langle X, Y, Z\rangle$ is an
    unshielded triple in $\graph P_i$ and \var Y is not an ancestor of \var X nor \var Z in $\graph S\manip{\set
    I_i}$ (by the semantics of edges in MAGs).
\item If $\langle W, \dots, X, Y, Z\rangle$ is a discriminating collider in $\graph P_i$, then $\langle W\dots, X, Y,
    Z\rangle$ is a discriminating path for \var Y in $\graph P_i$ and \var Y is an ancestor of either \var X or \var
    Z in $\graph S\manip{\set I_i}$ (by the semantics of edges in MAGs).
\item If $\langle W,\dots, X, Y, Z\rangle$ is a discriminating definite non collider in $\graph P_i$, then $\langle
    W\dots, X, Y, Z\rangle$ is a discriminating path for \var Y in $\graph P_i$ and  \var Y is not an ancestor of
    \var X nor \var Z in $\graph S\manip{\set I_i}$ (by the semantics of edges in MAGs).
    \newcounter{enumTemp}
    \setcounter{enumTemp}{\theenumi}
\end{enumerate}

These constraints are expressed using the core variables (edges, tails and arrows), as described in Figure
\ref{fig:sat_constraints}. For example, if \var X and \var Y are adjacent in $\graph P_i$, in a consistent SMCM \graph S
there must exist an inducing path \var p between \var X and \var Y in \graph S\manip{\set I_i} with respect to variables
\set{L_i}. Any truth-assignment to the core variables that does not entail the presence of such an inducing path should
not satisfy the SAT instance. The following constraints are added to ensure that the graphs satisfying constraints 1-5
above are SMCMs:
\begin{enumerate}
 \setcounter{enumi}{\theenumTemp}
\item $\forall\var X, \var Y\in \set O$, either \var X is not an ancestor of \var Y or \var Y is not an ancestor of
    \var X in \graph S (no directed cycles).
\item $\forall\var X, \var Y\in \set O$, at most one of $tail(X, Y)$ and $tail(Y, X)$ can be true (no selection
    bias).
\item $\forall\var X, \var Y\in \set O$, at least one of $tail(X, Y)$ and $arrow(Y, X)$ must be true.
\end{enumerate}

Naturally, Constraints 7 and 8 are meaningful only if \var X and \var Y are adjacent (if edge(X, Y) is true), and
redundant otherwise.

\subsection{Algorithm COmbINE}

\begin{algorithm}[t!]
\SetKwInOut{Input}{input}\SetKwInOut{Output}{output}
\SetKwFunction{addConstraints}{addConstraints}
\SetKwFunction{backBone}{backBone}
\SetKwFunction{initializeSMCM}{initializeSMCM}
\SetKwFunction{FCI}{FCI}
\Input{data sets $\{\set D_i\}_{i=1}^N$, sets of intervention targets $\{\set I_i\}_{i=1}^N$, FCI parameters
\emph{params}, maximum path length $mpl$,  conflict resolution strategy \emph{str}}
\Output{Summary graph $\graph H$}
\BlankLine
\lForEach{i}{$\graph P_i\leftarrow \FCI(\set D_i$, \emph{params})}\label{algoline:fci}
$\graph H\leftarrow$ \initializeSMCM($\{\graph P_i\}_{i=1}^N$)\;
$(\Phi, \graph F)\leftarrow$ \addConstraints($\graph H$, $\{\graph P_i\}_{i=1}^N$, $\{\set I_i\}_{i=1}^N$, $mpl$)\;
\graph F$'$ $\leftarrow$ select a subset of non-conflicting literals \graph F$'$ according to strategy  \emph{str}\;
$\graph H\leftarrow$ \backBone($\Phi\wedge \graph F'$)
\caption{COmbINE}\label{algo:COmbINE}
\end{algorithm}

We now present algorithm  \textbf{COmbINE} (Causal discovery from Overlapping INtErventions) that learns causal features
from multiple, heterogenous data sets. The algorithm takes as input a set of data sets $\{\set D_i\}_{i=1}^N$ over a set
of overlapping variable sets $\{\set O_i\}_{i=1}^N$. In each data set, a (possibly empty) subset of the observed
variables $\set I_i\subset \set O_i$ may be manipulated. FCI is run on each data set and the corresponding PAGs $\{\graph
P_i\}_{i=1}^N$ are produced. The algorithm then creates an candidate underlying SMCM \graph H. Subsequently, for each PAG
$\graph P_i$, the features of $\graph P_i$ are translated into constraints, expressed in terms of edges and endpoints in
\graph H, using the formulae in Figure \ref{fig:sat_constraints}. In the sample limit (and under the assumptions
discussed above), the SAT formula $\Phi\wedge F'$ produced by this procedure is satisfied by all and only the possible
underlying SMCMs for $\{\set P_i\}_{i=1}^N$. In the presence of statistical errors, however, $\Phi\wedge F'$ may be
unsatisfiable. To handle conflicts, the algorithm takes as input a strategy for selecting a non-conflicting subset of
constraints and ignores the rest. Finally, \combine\; queries the SAT formula for variables that have the same
truth-value in all satisfying assignments, translates them into graph features, and returns a graph that summarizes the
invariant edges and orientations of all possible underlying SMCMs. In the rest of this paper we call the graphical output
of \combine\; a \textbf{summary graph}.

The pseudocode for \combine\; is presented in Algorithm \ref{algo:COmbINE}. Apart from the set of data sets described
above, \combine\; takes as input the chosen parameters for FCI (threshold $\alpha$, maximum conditioning set $maxK$), the
maximum length of possible inducing paths to consider and a strategy for selecting a subset of non-conflicting
constraints.

Initially, the algorithm runs FCI on each data set $\set D_i$ and produces the corresponding  PAG $\graph P_i$. Then the
candidate SMCM \graph H is initialized: \graph H is the graph upon which all path constraints will be imposed. Therefore,
\graph H must have at least a superset of edges and at most a  subset of orientations of any consistent SMCM \graph S: If
\var p is an inducing (ancestral) path in \graph S, it must be a possibly inducing (ancestral) path in \graph H. An
obvious--yet not very smart--choice for \graph H would be the complete unoriented graph. However, looking for possible
inducing and ancestral paths on the complete unoriented graph over the union of variables could make the problem
intractable even for small input sizes. To reduce the number of possible inducing and ancestral paths, we use  Algorithm
\ref{algo:initializeSMCM} to construct \graph H.

Algorithm \ref{algo:initializeSMCM} constructs a graph \graph H that has all edges observed in any PAG $\graph P_i$  as
well as some additional edges that would not have been observed even if they existed: Edges connecting variables that
have never been  observed together, and edges connecting variables that have been observed together, but at least one of
them was manipulated in each joint appearance in a data set. For example, variables \var {X9} and \var {X15} in Figure
\ref{fig:example_io} are measured together in two data sets: $\set D_2$ and $\set D_3$. If \var {X9}\pagrightarrow
\var{X15} in the underlying SMCM, this edge would be present in $\graph P_3$. Similarly, if \var {X15}\pagrightarrow
\var{X9} in the underlying SMCM, the variables would be adjacent in $\graph P_2$. We can therefore rule out the
possibility of a directed edge between the two variables in \graph S. However, it is possible that \var {X15} and
\var{X9} are confounded in \graph S, and the edge disappears by the manipulation procedure in both $\graph P_2$ and
$\graph P_3$. Thus, Algorithm \ref{algo:initializeSMCM} will add these possible edges in \graph H. In addition, in Line
\ref{algoline:addOrientations}, Algorithm \ref{algo:initializeSMCM} adds all the orientations found so far in all $\graph
P_i$'s that are invariant\footnote{Other options would be to keep all non-conflicting arrows, or keep non-conflicting
arrows and tails after some additional analysis on definitely visible edges (see \cite{Zhang2008b},
\cite{Borboudakis2012} for more on this subject). These options are asymptotically correct and would constrain search
even further. Nevertheless, orientation rules in FCI seem to be prone to error propagation and we chose a more
conservative strategy giving a chance to the conflict resolution strategy to improve the learning quality. Naturally, if
an oracle of conditional independence is available or there is a reason to be confident on certain features, one can opt
to make additional orientations.}. The resulting graph has, in the sample limit, a superset of  edges and a subset of
orientations compared to the actual underlying SMCM. Lemma \ref{lemma:indpaths} formalizes and proves this property.

Having initialized the search graph, Algorithm \ref{algo:COmbINE} proceeds to generate the constraints. This procedure is
described in detail in Algorithm \ref{algo:addConstraints}, that is the core of \combine. These are: (i) the
bi-conditionals regarding the presence/absence of edges (Line \ref{algoline:bc1}), (ii) conditionals regarding unshielded
and discriminating colliders (Lines \ref{algoline:bc2a}, \ref{algoline:bc2b}, \ref{algoline:bc3a} and
\ref{algoline:bc3b}), (iii) constraints that ensure that any truth-setting assignment is a SMCM, i.e., it has no directed
cycles and that every edge has at least one arrowhead (Lines \ref{algoline:ndc} and \ref{algoline:ntt} respectively).

\begin{algorithm}[t!]
\SetKwInOut{Input}{input}\SetKwInOut{Output}{output}
\Input{PAGs $\{\graph P_i\}_{i=1}^{N}$}
\Output{initial graph $\graph H$}
\BlankLine
$\graph H\leftarrow$ empty graph over $\cup \set O_i$\;
\lForEach{i}{
$\graph H \leftarrow$ add all edges in $\graph P_i$ unoriented\label{algoline:addedges}}
Orient only arrowheads that are present in every $\graph P_i$\;\label{algoline:addOrientations}
\tcc{\small Add edges between variables never measured unmanipulated together}\label{algoline:de}
\normalfont
\ForEach {pair \var X, \var Y of non-adjacent nodes}{
\If{$\not \exists i$ s.t. $\var X, \var Y \in \set O_i \setminus\set I_i$}{
add \var X\doublecirc \var Y to \graph H\label{algoline:addedgesu}\;
\lIf{$\exists i$ s.t. $\var X, \var Y\in \set O_i$, $\var X\in \set I_i$, $\var Y\not\in\set I_i$}{add arrow into \var
X\label{algoline:addedges1}}
\lIf{$\exists i$ s.t. $\var X, \var Y\in \set O_i$, $\var Y\in \set I_i$, $\var X\not\in\set I_i$}{add arrow into \var
Y\label{algoline:addedges2}}
}
} 
\caption{initializeSMCM}\label{algo:initializeSMCM}
\end{algorithm}

\begin{algorithm}[t!]
\SetKwInOut{Input}{input}\SetKwInOut{Output}{output}
\Input{$\graph H$, $\{\graph P_i\}_{i=1}^N$, $\{\set I_i\}_{i=1}^N$, $mpl$}
\Output{$\Phi$, list of literals \graph F}
\BlankLine
$\Phi \leftarrow \emptyset\;$
\ForEach {$\var X, \var Y$}{
\textbf{posIndPaths}$\leftarrow$ paths in \graph H of maximum length $mpl$ that are possibly inducing with respect to
$\set L_i$\;

\ForEach{i}{
$\Phi\leftarrow \Phi \wedge \big [adjacent(\var X, \var Y, \graph P_i)\leftrightarrow \exists
p_{XY}\in$\textbf{posIndPaths} s. t. $inducing(p_{XY},i)$\big ]\;\label{algoline:bc1}
\lIf{\var X, \var Y are adjacent in $\graph P_i$}{add $adjacent(\var X, \var Y, \graph P_i)$ to \graph F}
\lElse{add $\neg adjacent(\var X, \var Y, \graph P_i)$ to \graph F}
} 
$\Phi\leftarrow\Phi\wedge\big [ \neg ancestor(X, Y)\vee \neg ancestor (Y, X)$\big ]\label{algoline:ndc}\;
$\Phi\leftarrow\Phi\wedge\big [\neg tail(X, Y)\vee \neg tail (Y, X)\big ]\wedge\big [(arrow(X, Y)\vee tail(X, Y)$\big
]\label{algoline:ntt}\;
}
\ForEach {i}{
\ForEach {unshielded triple in $\graph P_i$}{
{$\Phi\leftarrow \Phi \wedge \big [ dnc (X, Y, Z, \graph P_i) \rightarrow unshielded\_dnc(X, Y, Z, \graph
P_i)\big]$\;\label{algoline:bc2a}}
{$\Phi\leftarrow \Phi \wedge \big [collider (X, Y, Z, \graph P_i) \rightarrow unshielded\_collider(X, Y, Z, \graph
P_i\big]$\;\label{algoline:bc2b}}
\lIf{$\langle \var X, \var Y, \var Z\rangle$ is a collider in $\graph P_i$}{add $collider(\var X, \var Y, \var Z, \graph
P_i)$ to \graph F}\lElse{add $dnc(\var X, \var Y, \var Z, \graph P_i)$ to \graph F}
}
\ForEach {discriminating path $p_{WZ}=\langle W, \dots, X, Y, Z\rangle$}{
{$\Phi\leftarrow \Phi \wedge\big [ dnc(X, Y, Z, \graph P_i) \rightarrow discriminating\_dnc(p_{WZ}, Y,\graph
P_i)\big]$\;\label{algoline:bc3a}}
{$\Phi\leftarrow \Phi \wedge \big [collider(X, Y, Z, P_i) \rightarrow discriminating\_collider(p_{WZ}, Y, \graph
P_i)\big]$\;\label{algoline:bc3b}}
\lIf{\var X, \var Y, \var Z is a collider in $\graph P_i$}{add $collider(\var X, \var Y, \var Z, \graph P_i)$ to \graph
F}\lElse{add $dnc(\var X, \var Y, \var Z, \graph P_i)$ to \graph F}
}
}
\caption{addConstraints}\label{algo:addConstraints}
\end{algorithm}\DecMargin{1em}

The constraints are realized on the basis of the \emph{plausible} configurations of \graph H: Thus, for the constraints
corresponding to $adjacent(X, Y, i)$ the algorithm finds all paths between \var X and \var Y in \graph H that are
possibly inducing. Then, for the literal $adjacent(X, Y, i)$ to be true,  at least one of these paths is constrained to
be inducing; for the opposite, none of these paths is allowed to be inducing. This step is the most computationally
expensive part of the algorithm. The parameter \var{mpl} controls the length of the possibly inducing paths;  instead of
finding \emph{all} paths between \var X and \var Y that are possibly inducing, the algorithm looks for all paths of
length at most $mpl$. This plays a major part in the ability of the algorithm to scale up, since finding all possible
paths between every pair of variables can blow up even in relatively small networks, particularly in the presence of
unoriented cliques or in relatively dense networks.

Notice that the information on manipulations is included in the satisfiability instance through the encoding of the
constraints: For every adjacency between \var X and \var Y observed in  $\graph P_i$, the plausible inducing paths are
consistent with the respective intervention targets: No inducing path is allowed to include an edge that is incoming to a
manipulated variable. For example, in Figure \ref{fig:example_io} \var {X15} and \var {X14} are adjacent in $\graph P_3$,
where \var {X15} is manipulated. Since no information concerning experiments is employed up to the initialization of the
search graph, $\var X15\circrightarrow \var X14$ is in the initial search graph \graph H, and the edge is a possible
inducing path for \var {X15} and \var{X14} in $\graph P_3$. However, since \var {X15} is manipulated in $\graph P_3$, the
edge cannot have an arrow into \var {X15}. This is imposed by the constraint:
\begin{equation*}
\begin{multlined}
inducing(\langle X15, X14 \rangle, 3)\leftrightarrow  \\
(X15\in \set I_3\rightarrow tail(X14, X15)) \wedge (X14\in \set I_3\rightarrow tail(X15, X14))\wedge edge(X14, X15)
\end{multlined}
\end{equation*}
which is then added to $\Phi$ as \[inducing(X15, X14, 3)\leftrightarrow tail(X14, X15) \wedge edge(X14, X15).\]
Thus, in any SMCM \graph S that satisfies the final formula of Algorithm \ref{algo:COmbINE}, if\\$inducing(\langle X15,
X14\rangle, 3)$ is true, the edge will be consistent with the manipulation information.

As mentioned above, in the absence of statistical errors, all the constraints stemming from all PAGs $\graph P_i$ are
simultaneously satisfiable. In practical settings however, it is possible that some of the PAGs have some erroneous
features due to statistical errors, and these features can lead to conflicting constraints. To tackle this problem,
Algorithm \ref{algo:addConstraints} using the following technique: For every observed feature, instead of imposing the
implied constraints on the formula $\Phi$, the algorithm adds a bi-conditional connecting the feature to the constraints.
For example, if \var X and \var Y are found adjacent in $\graph P_i$, then instead of adding the constraints  $\exists
p_{XY}: inducing (X, Y, i)$ to $\Phi$, we add the bi-conditional $adjacent(X, Y, \graph P_i)\leftrightarrow \exists
p_{XY}: inducing (X, Y, i)$. The antecedents of the conditionals are stored in a list of literals \graph F. The conflict
resolution strategy is then imposed on this list of literals, selecting a subset $F'$ that results in a satisfiable SAT
formula $\Phi \wedge \graph F'$. The formula $\Phi \wedge \graph F'$ is expressed in Conjunctive Normal Form (CNF) so it
can be input to standard SAT solvers.

\def\imagetop#1{\vtop{\null\hbox{#1}}}
 \begin{figure}[t!]\centering
\begin{tabular}{|c|c|c|c|c|}\hline
\includegraphics[width = 0.18\columnwidth]{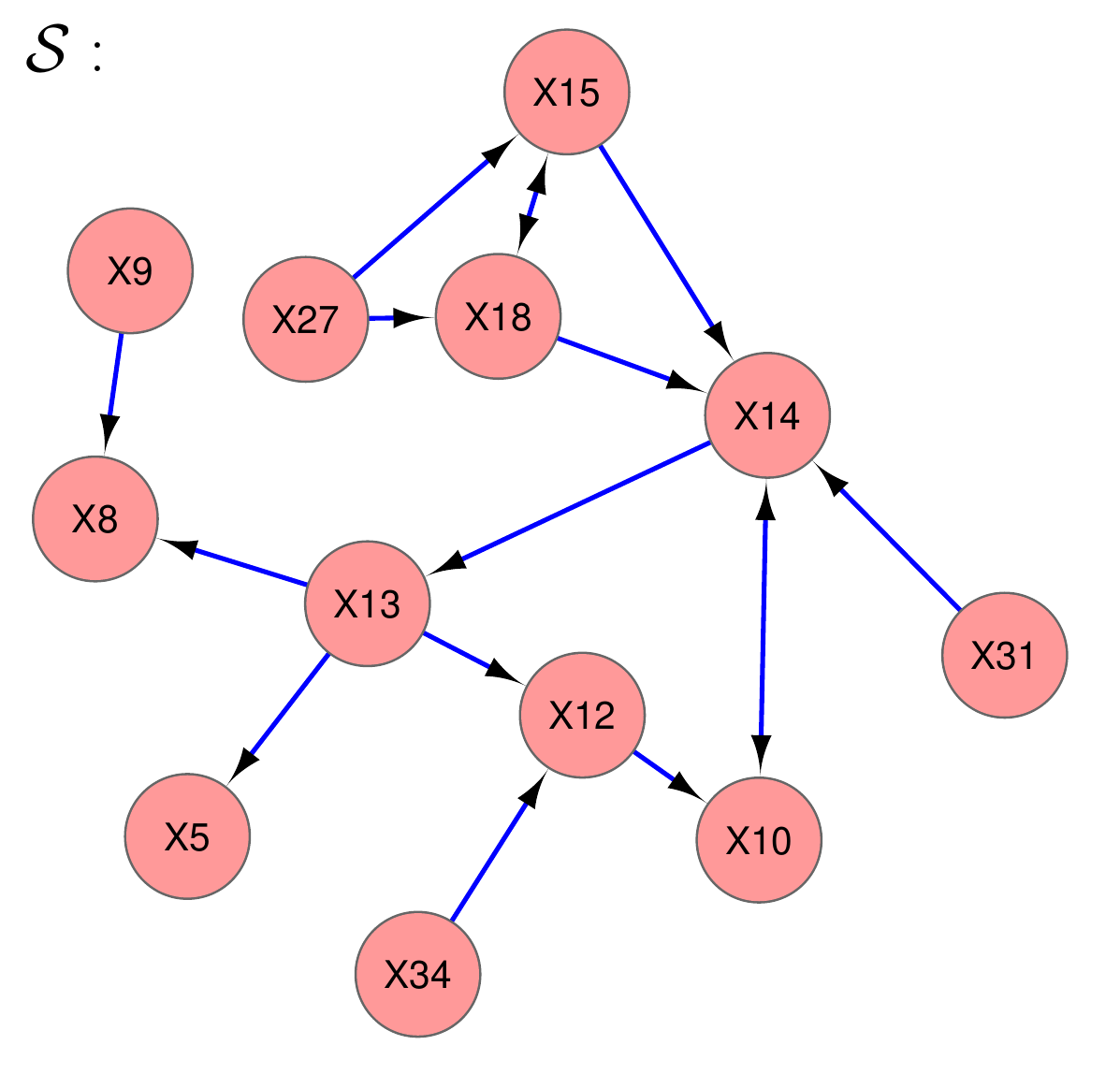}&
\includegraphics[width = 0.18\columnwidth]{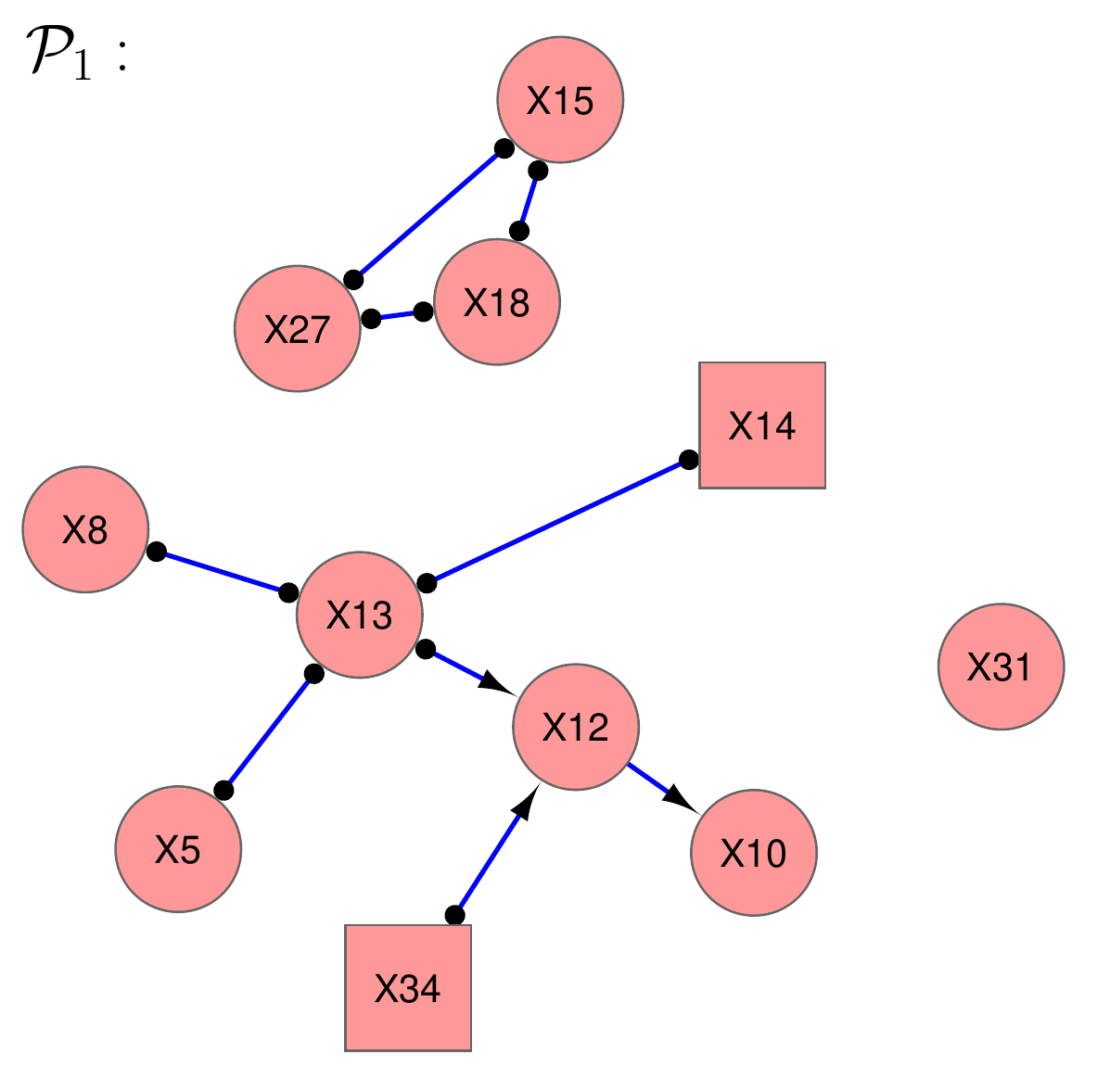}&
\includegraphics[width = 0.18\columnwidth]{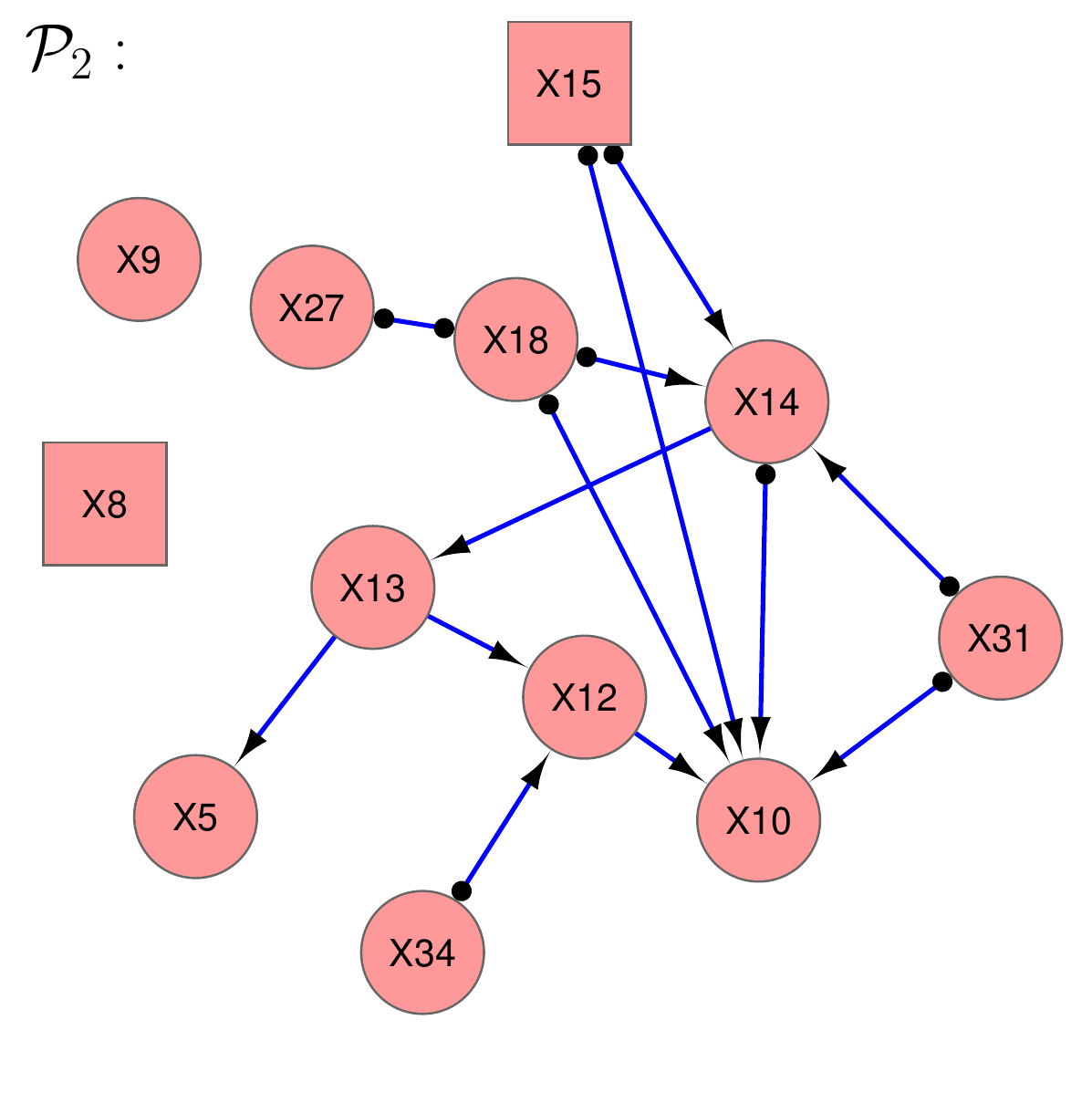}&
\includegraphics[width = 0.18\columnwidth]{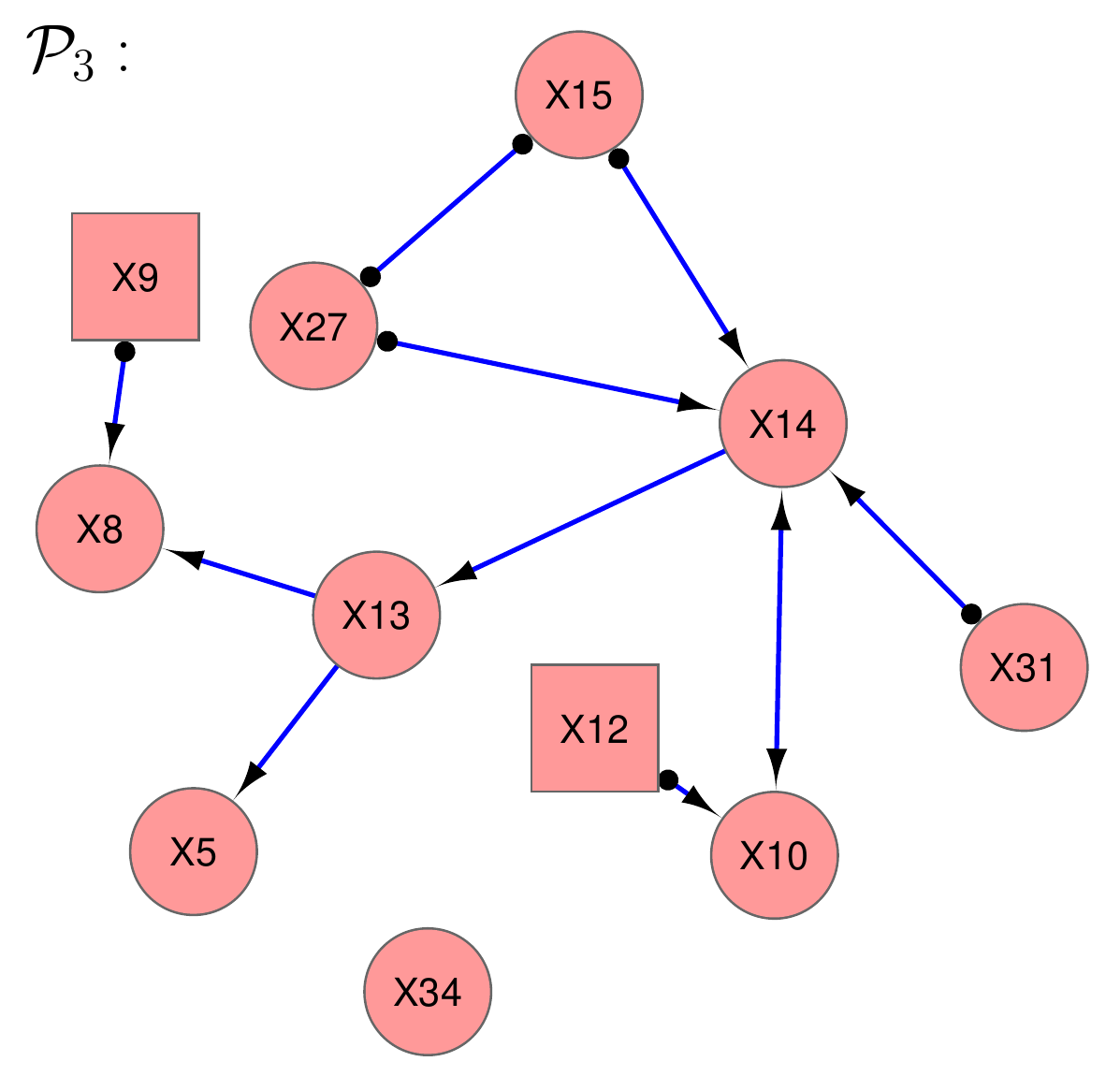}&
\includegraphics[width = 0.18\columnwidth]{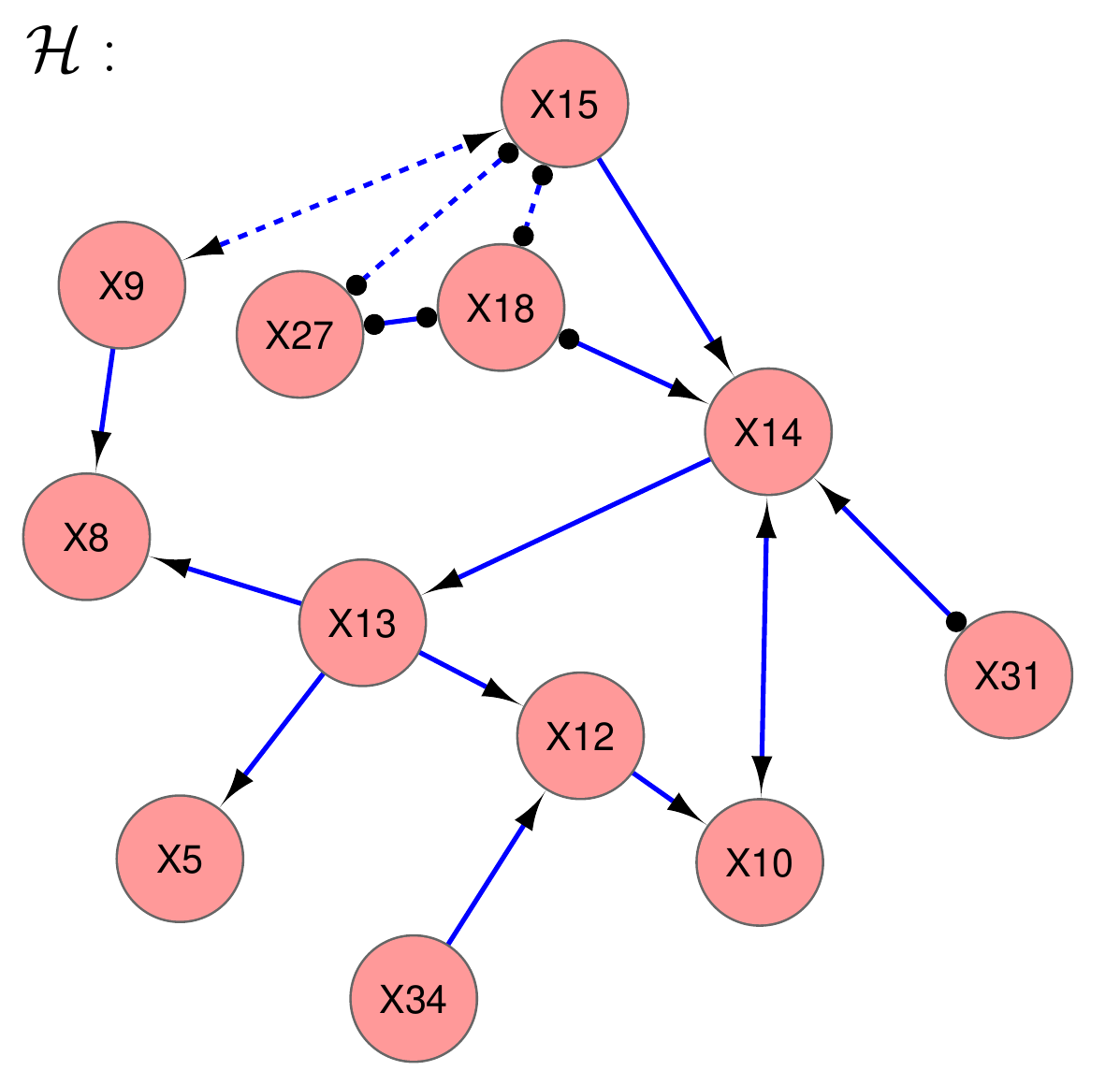}\\ \hline
\end{tabular}
\caption{\label{fig:example_io}\textbf{An example of COmbINE input - output}. Graph \graph S is the actual,
data-generating, underlying SMCM over 12 variables. PAGs $\graph P_1, \graph P_2 $ and $\graph P_3$ are the output of FCI
ran with an oracle of conditional independence on three different marginals of \graph G. \graph H is the output of
\combine\; algorithm. The sets of latent variables (with respect to the union of observed variables) per data set are:
$\set L_1 =\{X9\}$, $\set L_2= \{\emptyset\}$, $\set L_3= \{X18\}$. The sets of manipulated variables (annotated as
rectangle nodes instead of circles in the respective graphs) are: $\set I_1 = \{X14, X34\}$, $\set I_2 = \{X15, X8\}$,
$\set I_3 = \{X9, X12\}$. Notice that \var{X10} and \var{X31} are adjacent in $\graph P_2$, but not in $\graph P_1$ or
$\graph P_3$. This happens because there exists an inducing path in the underlying SMCM ($X31\protect\pagrightarrow
X14\protect\pagbidir X10$ in \graph S) that is ``broken" by the manipulation of \var{X14}  and \var{X12}, respectively.
Also notice a dashed edge between \var{X9} and \var{X15}, which cannot be excluded since the variables have never been
observed unmanipulated together. Even if the link existed, it would be destroyed in both $\graph  P_2$ and $\graph P_3$,
where both variables are observed.  All graphs were visualized in Cytoscape \citep{smoot2011cytoscape}.}
\end{figure}

without imposing the antecedent. These semantics should always be guaranteed and thus, $\Phi$ forms a set of
hard-constraints. In contrast, if the list of antecedents in \graph F leads to a conflict, one can select only a subset
of antecedents to satisfy (soft-constraints).

Recall that the propositional variables of $\Phi$ correspond to the features of the actual underlying SMCM (its edges and
endpoints). Some of these variables have the same value in all the possible truth-setting assignments of $\Phi \wedge
\graph F'$, meaning the respective features are invariant in all possible underlying SMCMs. Such variables are called
\textbf{backbone} variables of $\Phi \wedge \graph F'$  \citep{hyttinen2013discovering}. The actual value of a backbone
variable is called the polarity of the variable. For sake of brevity, we say an edge or endpoint has polarity 0/1 if the
corresponding variable is a backbone variable in $\Phi \wedge \graph F'$ and has polarity 0/1. Based on the backbone of
$\Phi \wedge \graph F'$, the final step of \combine\; is to construct the summary graph \graph S. \graph S has the
following types of edges and endpoints:

\begin{itemize}
\item \textbf{Solid Edges:} in \graph H that have polarity 1 in $\Phi \wedge \graph F'$, meaning that they are
    present in all possible underlying SMCM.
\item \textbf{Absent Edges:} Edges that are not in \graph H or edges in \graph H  that have polarity 0 in $\Phi$,
    meaning that they are absent in all possible underlying SMCM.
\item \textbf{Dashed Edges:} Edges in \graph H  that are not backbone variables in  $\Phi \wedge \graph F'$, meaning
    that there exists at least one possible underlying SMCM where this edge is present and one where this edge is
    absent.
\item \textbf{Solid Endpoints:} Endpoints in \graph H  that are  backbone variables in $\Phi \wedge \graph F'$,
    meaning that this orientation is invariant in all possible underlying SMCMs.
\item \textbf{Dashed (circled) Endpoints:} Endpoints  in \graph H  that are  not backbone variables in $\Phi \wedge
    \graph F'$, meaning that there exists at least one SMCM where this orientation does not hold.
\end{itemize}

We use the term \textbf{solid features} of the summary graph to denote the set of solid edges, absent edges and solid
endpoints of the summary graph.

\def\imagetop#1{\vtop{\null\hbox{#1}}}
 \begin{figure}[t!]\centering
\begin{tabular}{|c|c|c|}\hline
\includegraphics[width = 0.3\columnwidth]{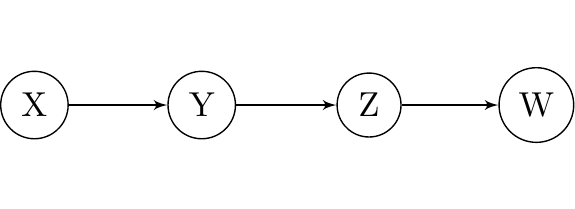}&
\includegraphics[width = 0.3\columnwidth]{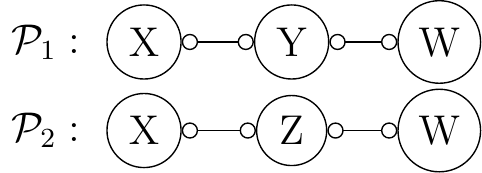}&
\includegraphics[width = 0.3\columnwidth]{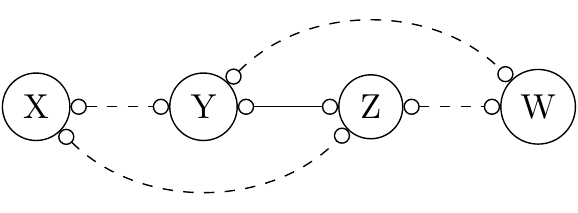}\\ \hline
\end{tabular}
\caption{\label{fig:inca_example}
\textbf{A detailed example of a non-trivial inference}. From left to right: The true underlying SMCM over variables \var
X, \var Y, \var Z, \var W;
PAGs $\graph P_1$ and $\graph P_2$ over $\{X, Y , W\}$ and $\{X, Z , W\}$, respectively; The output \graph H of Algorithm
\ref{algo:COmbINE} ran with an oracle of conditional independence. Notice that, the edges in $\graph P_1$ can not both
simultaneously occur in a consistent SMCM \graph S: This would make $\var X\protect\doublecirc\var
Y\protect\doublecirc\var W$ an inducing path for \var X and \var W with respect to $\set L_2 =\{Y\}$ and contradict the
features of $\graph P_2$, where \var X and \var W are not adjacent. Similarly, $\var X\protect\doublecirc\var
Z\protect\doublecirc\var W$ cannot occur in any consistent SMCM \graph S. The only possible edge structures that explain
all the observed adjacencies and definite non colliders are $\var X\protect\doublecirc\var Y\protect\doublecirc\var
Z\protect\doublecirc\var W$  or $\var X\protect\doublecirc\var Z\protect\doublecirc\var Y\protect\doublecirc\var W$.
Either way, \var Y and \var Z share an edge in all consistent SMCMs, and the algorithm will predict a solid edge between
\var Y and \var Z, even if the two have not been measured in the same data set. This example is discussed in detail in
\citep{Tsamardinos2012towards}.}
\end{figure}

Overall, \emph{Algorithm \ref{algo:COmbINE} takes as input a set of data sets and a list of parameters  and outputs a
summary graph that has all invariant edges and orientations of the SMCMs that satisfy as many constraints as possible
(according to some strategy)}.  The algorithm
is capable of non-trivial inferences, like for example the presence of a solid edge among variables never measured
together. Figures \ref{fig:example_io} and \ref{fig:inca_example} illustrate the output of Algorithm \ref{algo:COmbINE},
along with the corresponding input PAGs. For an oracle of conditional independence, Algorithm \ref{algo:COmbINE} is sound
and complete in the manner described in Theorem \ref{the:soundcomplete}. Lemmas \ref{lemma:indpaths} to
\ref{lemma:complete} are necessary for the proof of soundness and completeness: Lemma \ref{lemma:indpaths} proves that
the possibly inducing and ancestral paths employed by \combine\; are complete: for any consistent \graph S, if \var p is
a path  that is  inducing with respect to a set \set L (ancestral) in \graph S, \var p  is possibly inducing  with
respect to \set L (possibly ancestal) in the initial search graph \graph H, and will therefore be considered during
Algorithm \ref{algo:addConstraints}. This also implies that if there exists no possibly inducing (ancestral) path in
\graph H there exists no inducing (ancestral) in \graph S. Lemma \ref{lemma:sound} proves that any consistent SMCM \graph
S satisfies the final formula $\Phi \wedge F'$ of Algorithm \ref{algo:COmbINE}, and Lemma \ref{lemma:complete} proves
that any truth-setting assignment of the final formula $\Phi \wedge F'$ corresponds to a consistent SMCM \graph S.

\begin{lemma}\label{lemma:indpaths}Let $\{\graph P_i\}_{i=1}^N$ be a set of PAGs and \graph S a SMCM such that \graph S
is possibly underlying for $\{\graph P_i\}_{i=1}^N$, and let \graph H be the initial search graph returned by Algorithm
\ref{algo:initializeSMCM} for $\{\graph P_i\}_{i=1}^N$. Then, if \var p is an ancestral path in \graph S, then \var p is
a possibly ancestral path in \graph H. Similarly, if \var p is a possibly inducing path with respect to \set L in \graph
S, then \var p is a possibly inducing path with respect to \set L in \graph H.
\end{lemma}
\begin{proof}  See Appendix \ref{app:proofs}.\end{proof}
\begin{lemma}\label{lemma:sound}Let $\{\set D_i\}_{i=1}^N$ be a set of data sets over overlapping subsets of \set O, and
$\{\set I_i\}_{i=1}^N$ be a set of (possibly empty) intervention targets such that $\set I_i\subset \set O_i$ for each i.
Let $\graph P_i$  be output PAG of  FCI  for data set $\set D_i$, $\Phi \wedge\graph F'$ be the final formula of
Algorithm \ref{algo:COmbINE}, and \graph S be a possibly underlying SMCM for $\{\set P_i\}_{i=1}^N$. Then \graph S
satisfies  $\Phi \wedge\graph F'$ .
\end{lemma}
\begin{proof} See Appendix \ref{app:proofs}.\end{proof}
\begin{lemma}\label{lemma:complete}Let $\{\set D_i\}_{i=1}^N$, $\{\set I_i\}_{i=1}^N$, $\{\set P_i\}_{i=1}^N$, $\Phi
\wedge\graph F'$ be defined as in Lemma \ref{lemma:sound}. If graph S satisfies  $\Phi \wedge\graph F'$, then \graph S is
a possibly underlying SMCM for $\{\set P_i\}_{i=1}^N$.
\end{lemma}
\begin{proof} See Appendix \ref{app:proofs}.\end{proof}

\begin{theorem}\label{the:soundcomplete}(Soundness and completeness of Algorithm \ref{algo:COmbINE})
Let $\{\set D_i\}_{i=1}^N$, $\{\set I_i\}_{i=1}^N$, $\{\set P_i\}_{i=1}^N$, $\Phi \wedge\graph F'$ be defined as in Lemma
\ref{lemma:sound}. Finally, let \graph H be the summary graph returned by \combine\;. Then the following hold: \\
\textbf{Soundness}: If a feature (edge, absent edge, endpoint) is \emph{solid} in \graph H, then this feature is present
in \emph{all} consistent SMCMs. \\
\textbf{Completeness}: If a feature is present in \emph{all} consistent SMCMs, the feature is solid in \graph H.
\end{theorem}
\begin{proof} \emph{Soundness:} Solid features correspond to backbone variables. By Lemma \ref{lemma:sound} every
possible underlying SMCM \graph S for $\{\set P_i\}_{i=1}^N$ satisfies the final formula $\Phi \wedge \graph F'$. Thus,
if a core variable has the same value in all the possible truth-setting assignments of $\Phi \wedge \graph F'$, this
feature is present in all possible underlying SMCMs.
\emph{Completeness:} By Lemma \ref{lemma:complete} the final formula $\Phi \wedge \graph F$' of Algorithm
\ref{algo:COmbINE} is satisfied only by possibly underlying SMCMs. Thus, if a core variable is present in \emph{all}
consistent SMCMs, the corresponding core variable will be a backbone variable for $\Phi \wedge \graph F'$.
\end{proof}

\subsection{A strategy for conflict resolution based on the Maximum MAP Ratio}
In this section, we present a method for assigning a measure of confidence to every literal in list \graph F described in
Algorithm \ref{algo:COmbINE}, and a strategy for selecting a subset of non-conflicting constraints. List \graph F
includes four types of literals, expressing different statistical information:
\begin{enumerate}
\item $adjacent(X, Y, \graph P_i)$: \var X and \var Y are independent given some $\set Z\subset \set O_i$
\item $\neg adjacent(X, Y, \graph P_i)$: \var X and \var Y are not independent given any subset of $\set O_i$.
\item $collider(X, Y, Z, \graph P_i)$: \var  Y is in no subset of $\set O_i$ that renders \var X and \var Z
    independent.
\item $dnc(X, Y, Z, \graph P_i)$: \var Y is in every subset of $\set O_i$ that renders \var X and \var Z
    independent.
\end{enumerate}
For the scope of this work, we will  focus on ranking the first two types of antecedents: Adjacencies and
non-adjacencies. We will then assign unshielded colliders and non-colliders to the same rank as the non-adjacency of the
triple's endpoints; similarly, discriminating colliders and non-colliders will be assigned to the same rank as the
non-adjacency of the path's endpoints. Naturally, this criterion of sorting colliders and non-colliders is merely a
heuristic, as more than one tests of independence are involved in deciding that a triple is a (non) collider.

Assigning a measure of likelihood or posterior probability to every single (non) adjacency would enable their comparison.
A non-adjacency in a PAG corresponds to a conditional independence given some subset of the observed variables. In
contrast, an adjacency corresponds to the lack of such a subset. Thus, an edge between \var X and \var Y should be
present in $\graph P_i$ if the evidence (data) is less in favor of hypothesis:
\begin{equation}H_0: \exists \set Z\subset \set O_i: Ind(\var X, \var Y|\set Z)\textnormal{
than the alternative } H_1: \not\exists  \set Z\subset \set O_i: Ind(\var X, \var Y|\set
Z)\label{eq:initialHypotheses}\end{equation}
This is a complicated set of hypotheses, that involves multiple tests of independence. We try to approximate testing by
using a single test of independence as a surrogate: During FCI, several conditioning sets are tested for every pair of
variables \var X and \var Y. Let $\set Z_{XY}$ be the conditioning test for which the highest p-value is identified for
the given pair of variables. Notice that it is this maximum p-value that is employed in FCI and similar algorithms to
determine whether an edge is included in the output or not. We use the set of hypotheses \[H_0:Ind(\var X, \var Y|\set
Z_{XY})\textnormal{ against the alternative }H_1:\neg Ind(\var X, \var Y|\set Z_{XY})\] as a surrogate for the set of
hypotheses in Equation \ref{eq:initialHypotheses}. Under the null hypothesis, the p-values follow a uniform
$\mathcal{U}([0, 1])$ distribution\footnote{This is actually an approximation in this case, since these p-values are
maximum p-values over several tests.}, also known as the $Beta(1, 1)$ distribution. Under the alternative hypothesis, the
density of the p-values should be decreasing in \var p. One class of decreasing densities is the $Beta(\xi, 1)$
distribution for $0<\xi<1$, with density $f(p|\xi)= \xi p^{\xi-1}$.  Thus, we can approximate the null and alternative
hypotheses in terms of the p-value as
\begin{equation}H_0:p_{XY.\set Z}\sim Beta(1,1) \text{ against }H_1: p_{XY.\set Z}\sim Beta(\xi, 1) \text{ for some
}\xi\in(0, 1).\label{eq:beta_alt}\end{equation}
Taking the Beta alternatives was presented as a method for calibrating p-values in \cite{sellke2001calibration}. For the
purpose of this work, we use them to estimate whether dependence is more probable than independence for a given p-value
\var p, by estimating which of the Beta alternatives it is most likely to follow.

Let \graph F be a set of $M$ literals corresponding to adjacencies and non-adjacencies, and $\{ p_j\}_{j=1}^M$ the
respective maximum p-values: If the j-th literal in \graph F is $(\neg) adjacent(X, Y, \graph P_i)$, then $p_j$ is the
maximum p-value obtained for \var X, \var Y during FCI over $\set D_i$. We assume that this population of p-values
follows a mixture of $Beta(\xi, 1)$ and $Beta(1,1)$ distribution. If $\pi_0$  is the proportion of p-values following
$Beta(\xi, 1)$,  the probability density function is

\begin{equation*}f(p|\xi, \pi_0) = \pi_0+(1-\pi_0)\xi p^{\xi-1}\end{equation*} and the likelihood for a set of p-values
$\{p_j\}_{j=1}^M$ is
\begin{equation*}L(\xi, \pi_0) = \prod_j (\pi_0+(1-\pi_0)\xi p_j^{\xi-1}).\end{equation*}
The respective negative log likelihood is
\begin{equation}\label{eq:nll}-LL(\xi, \pi_0) = -\sum_j log(\pi_0+(1-\pi_0)\xi p_i^{\xi-1}).\end{equation}
For given estimates $\hat{\pi_0}$ and $\hat{\xi}$, the MAP ratio of $H_0$ against $H_1$ is

\begin{equation*}E_0(p) = \frac{P(p|H_0)P(H_0)}{P(p|H_1)P(H_1)} = \frac{P(p|p\sim Beta(1, 1))P(p\sim Beta(1,
1))}{P(p|p\sim Beta(\hat{\xi}, 1))P(p\sim Beta(\hat{\xi}, 1))}=\frac{\hat{\pi_0}}{\hat{\xi}
p^{\hat{\xi}-1}(1-\hat{\pi_0})}.\end{equation*}
$E_0(p)>1$ implies that for the test of independence represented by the p-value \var p, independence is more probable
than dependence, while $E_0(p)<1$ implies the opposite. Moreover, the value of $E_0(p)$ \emph{quantifies} this belief.
Conversely, the corresponding MAP ratio of $H_1$ against $H_0$ is
\begin{equation*}E_1(p) = \frac{\hat{\xi}
p^{\hat{\xi}-1}(1-\hat{\pi_0})}{\hat{\pi_0}}.\label{eq:mapratio2}\end{equation*}
We define the \textbf{maximum MAP ratio (MMR)} for a p-value \var p to be the maximum between the two:
\begin{equation}E(p) = max\big\{\frac{\hat{\pi_0}}{\hat{\xi} p^{\hat{\xi}-1}(1-\hat{\pi_0})}, \frac{\hat{\xi}
p^{\hat{\xi}-1}(1-\hat{\pi_0})}{\hat{\pi_0}}\big\}.\label{eq:mapRatio}\end{equation}

\begin{figure}[!t]
\begin{center}
\begin{tabular}{cccc}
\includegraphics[width = 0.3\columnwidth]{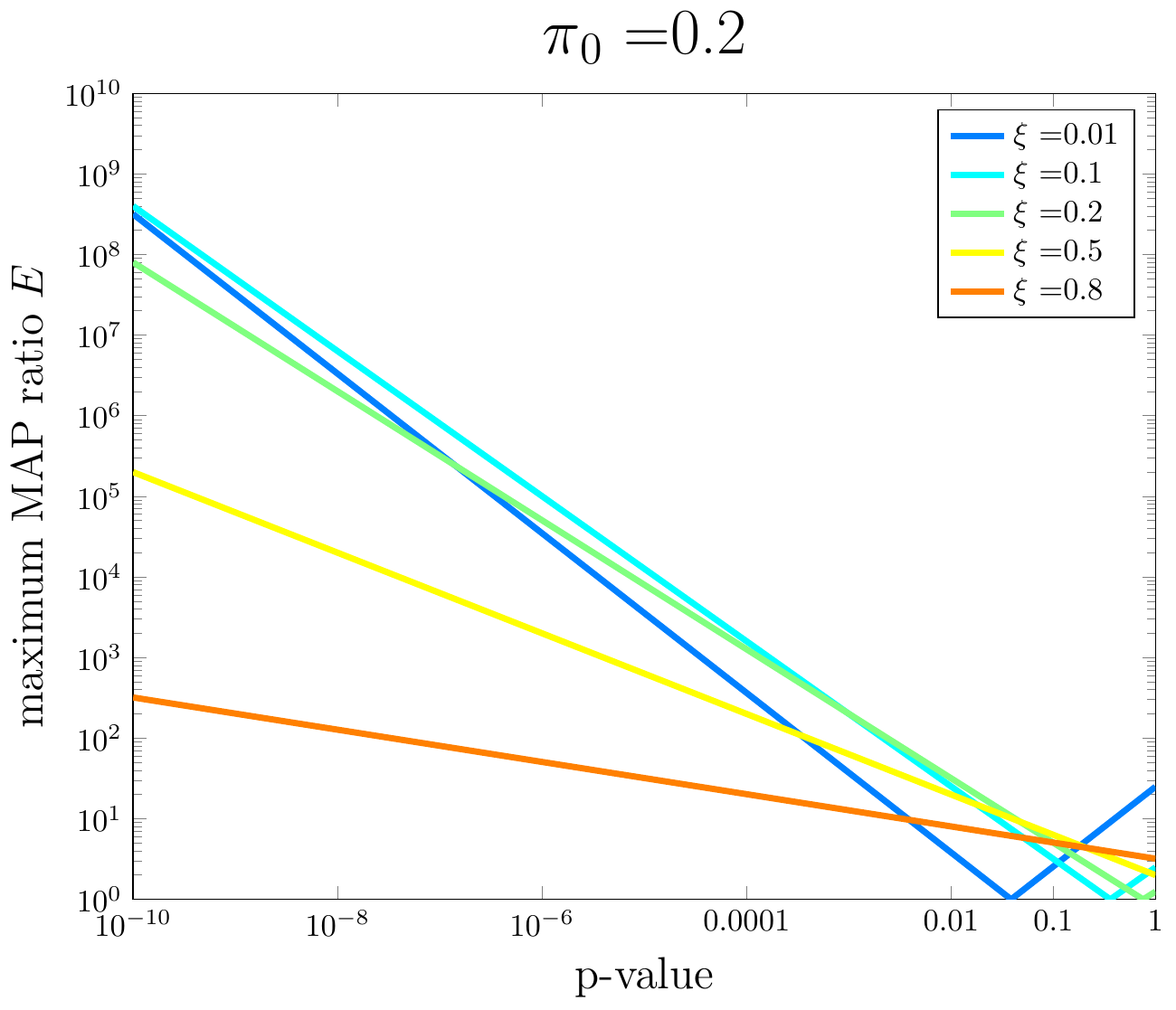}&
\includegraphics[width = 0.3\columnwidth]{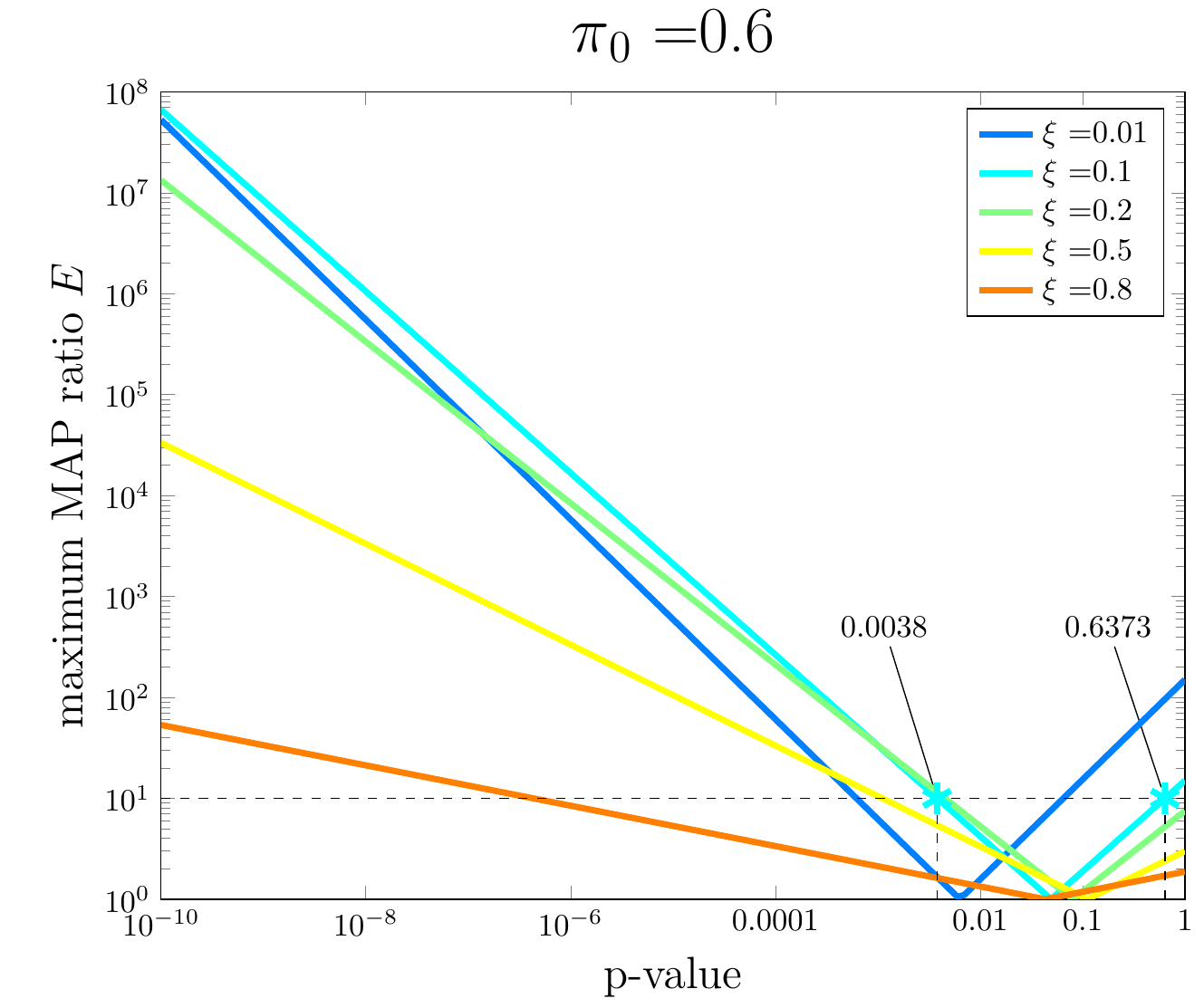}&
\includegraphics[width = 0.3\columnwidth]{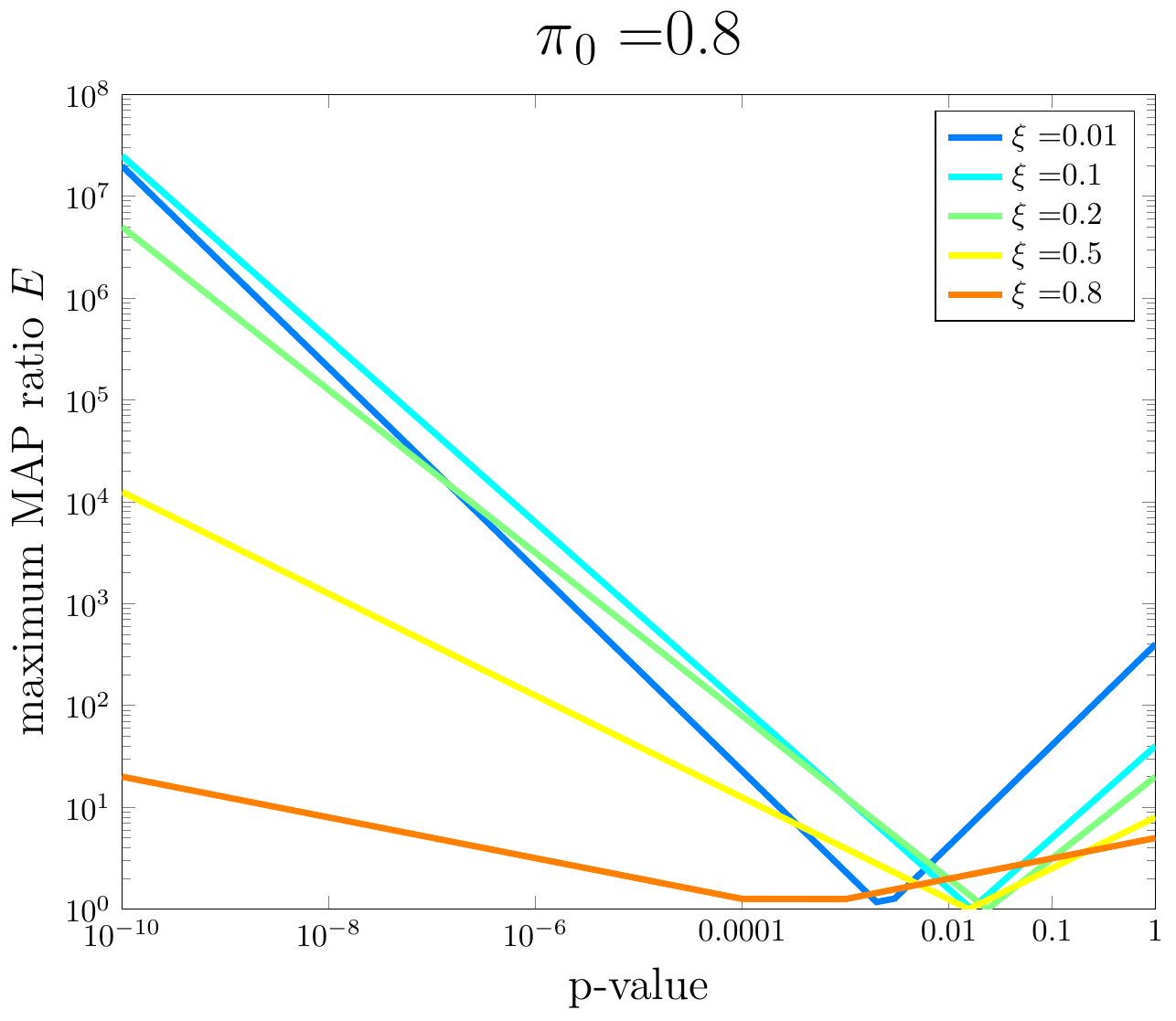}\\
\end{tabular}
\end{center}
\caption{\label{fig:mapratios}\textbf{Log of the maximum map ratio $E(p)$ versus log of the p-value \var p for various
$\hat{\pi_0}$ and $\hat{\xi}$.}. For $\hat{\pi_0} = 0.6$ and $\hat{\xi} = 0.1$, an adjacency supported by a maximum
p-value of 0.0038 corresponds to the same $E$ as a non-adjacency supported by a p-value of 0.6373. The intersection point
of the line with the x axis is the p for which $E_0(p) =E_1(p)=1$.}
\end{figure}
\begin{figure}[!t]
\begin{center}
\includegraphics[width = 0.3\columnwidth]{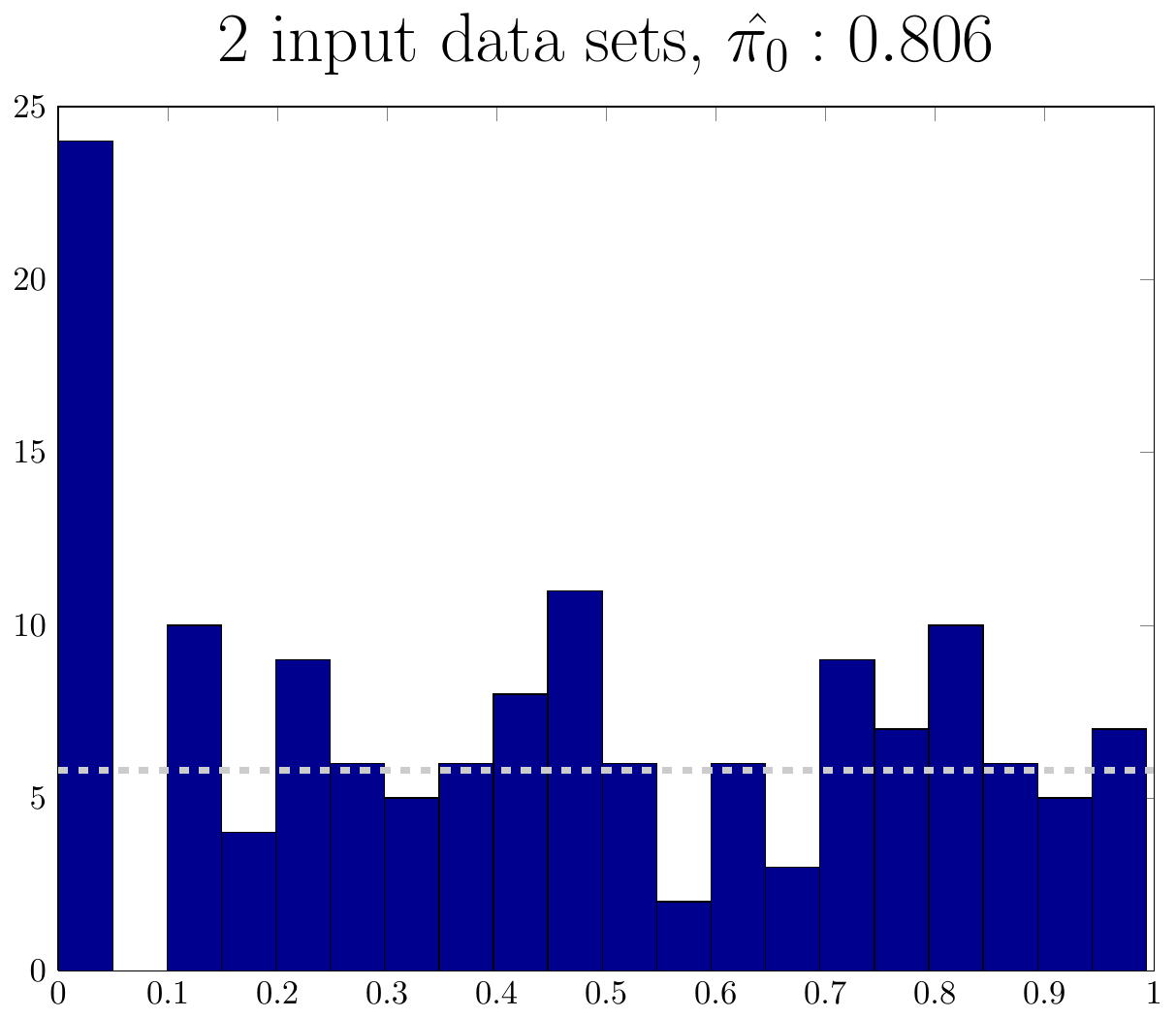}
\includegraphics[width = 0.3\columnwidth]{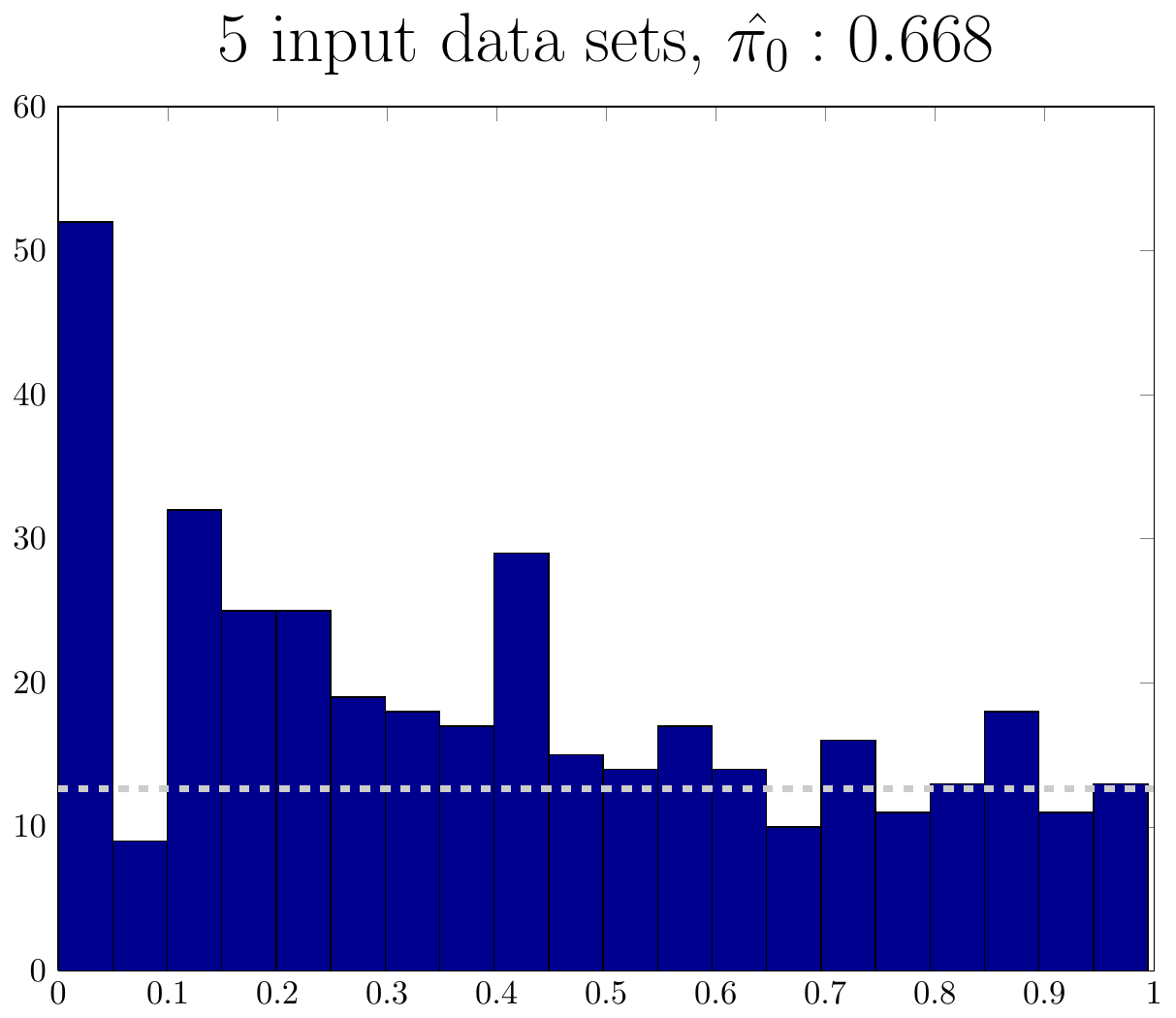}
\includegraphics[width = 0.3\columnwidth]{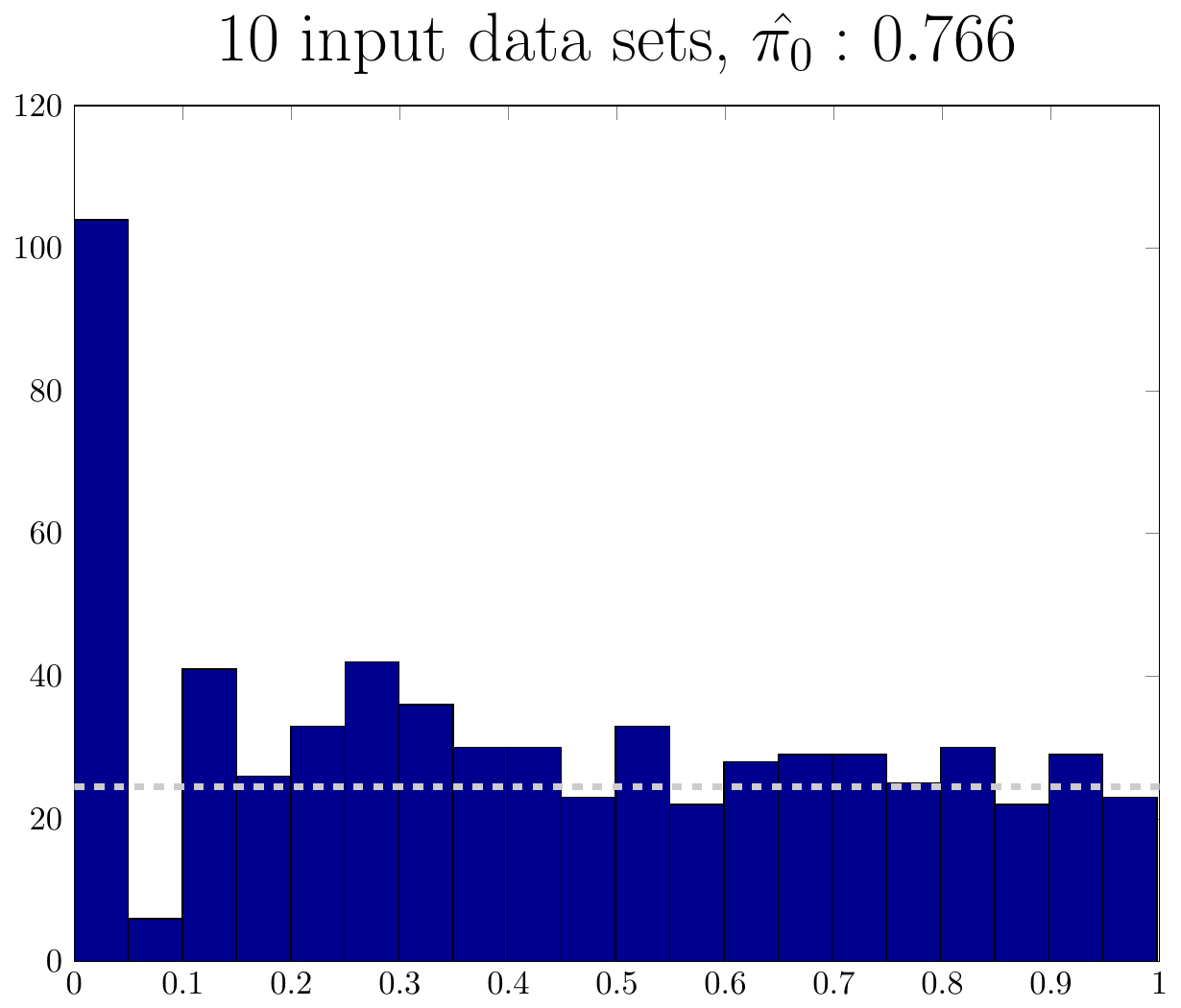}\end{center}
\caption{\label{fig:pi0xiALARM}\textbf{Distribution of p-values and estimated $\hat{\pi_0}$}. We used the method of
\cite{Storey2003} to estimate $\hat{\pi_0}$ for a sample of p-values corresponding to 2 (left), 5 (center) and 10 (right)
input data sets. We generated networks by manipulating a marginal of the ALARM network \citep{Beinlich1989} consisting of
14 variables. In each experiment, at most 3 variables were latent and at most 2 variables were manipulated. We  simulated
data sets of 100 samples each from the resulting manipulated graphs. We ran FCI on each data set with $\alpha =0.1$ and
$maxK=5$ and cached the maximum p-value reported for each pair of variables. We used the p-values from all data sets to
estimate $\hat{\pi_0}$. The dashed line corresponds to the proportion of p-values that come from the null distribution
based on the estimated $\hat{\pi_0}$.}
\end{figure}

MMR estimates heuristically quantify our confidence in the observed adjacencies and non-adjacencies and are employed to
create a  list of literals as follows:  Let \var X and \var Y be a pair of observed variables, and $\var p_{XY}$ be the
maximum p-value reported during FCI for these variables. Then, if $E_0(p_{XY})>E_1(p_{XY})$, the literal $\neg
adjacent(X, Y, i)$ is added to \graph F  with confidence estimate $E(p_{XY})$. Otherwise, the literal $adjacent(X, Y, i)$
is added to \graph F with a confidence estimate $E(p_{XY})$. The list can then be sorted in order of confidence, and the
literals can be satisfied incrementally. Whenever a literal in the list is encountered that cannot be satisfied in
conjunction with the ones already selected, it is ignored.

Notice that, it is possible that for a p-value $E_0(p_{XY})>E_1(p_{XY})$ (i.e., MMR determines independence is more
probable), even though $p_{XY}$ is smaller than the FCI threshold used. In other words, given a fixed FCI threshold,
dependence maybe accepted; but, when analyzing the set of p-values encountered to compute MMR, independence seems more
probable. The reverse situation is also possible. The pseudo-code in Algorithm \ref{algo:MMRstrategy} (Lines 6---10)
accepts the MMR decisions for dependencies and independencies; {\em this is equivalent to dynamically readjusting the
decisions made by FCI}. Nevertheless, in anecdotal experiments we found that the literals for which this situation occurs
are near the end of the sorted list; thus, whether one accepts the initial decisions of FCI based on a fixed threshold,
or a dynamic threshold based on MMR usually does not have a large impact on the output of the algorithm.

\begin{algorithm}[t!]
\SetKwInOut{Input}{input}\SetKwInOut{Output}{output}
\SetKwFunction{FCI}{FCI}
\Input{SAT formula $\Phi$, list of literals \graph F, their corresponding p-values $\{p_j\}$}
\Output{List of non conflicting literals $\graph F'$}
\BlankLine
$\graph F'\leftarrow\emptyset$\;
Estimate $\hat{\pi_0}$ from  $\{p_j\}$ using the method described in \cite{Storey2003}\;
Find $\hat{\xi}$ that minimizes $-\sum_j log(\hat{\pi_0}+(1-\hat{\pi_0})\xi p_i^{\xi-1})$\;
\ForEach{literal $(\neg)adjacent(X, Y, \graph P_i) \in \graph F$ with p-value $p_j$}{
$E_0(p_j)\leftarrow\frac{\hat{\pi_0}}{\hat{\xi} p_j^{\hat{\xi}-1}(1-\hat{\pi_0})}$, $E_1(p_j)\leftarrow \frac{\hat{\xi}
p_j^{\hat{\xi}-1}(1-\hat{\pi_0})}{\hat{\pi_0}}$\;
\uIf{$E_1(p_j)<E_0(p_j)$}{add  $\neg adjacent(X, Y, \graph P_i)$ to \graph F}
\Else{add  $adjacent(X, Y, \graph P_i)$ in \graph F}
$Score(literal)\leftarrow max\{E_0(p_j),E_1(p_j)\}$\;
}
\ForEach{literal $collider(X, Y, Z, \graph P_i)$, $dnc(X, Y, Z, \graph P_i)$}{
\uIf{\var X, \var Y, \var Z is an unshielded triple in $\graph P_i$}{$Score(literal)\leftarrow Score(X, Z, \graph
P_i)$\;}
\ElseIf{$\langle\var W \dots\var X, \var Y, \var Z\rangle$ is discriminating for \var Y in $\graph
P_i$}{$Score(literal)\leftarrow Score(W, Z, \graph P_i)$\;}
}
$\graph F\leftarrow$ sort \graph F by descending score\;

\ForEach{$\phi \in \graph F$}{\label{algoline:sat}
\If{$\Phi\wedge \phi$ is satisfiable}{$\Phi\leftarrow\Phi \wedge \phi$\; Add $\phi$ to $\graph F'$}\label{algoline:sate}
}
\caption{MMRstrategy}\label{algo:MMRstrategy}
\end{algorithm}

Figure \ref{fig:mapratios} shows how the MMR varies with the p-value for several combinations of $\hat{\pi_0}$ and $\hat
{\xi}$. The lowest possible value of the MMR is 1, and corresponds to the p-value \var p for which $E_0(p)=E_1(p)$.
Naturally, for the same $\xi$, this p-value (where the odds switch in favor of non-adjacency) is larger for a lower
$\pi_0$. In Figure \ref{fig:mapratios} for $\pi_0=0.6$ we can see an example of two p-values that correspond to the same
$E$: An adjacency represented by a p-value of $0.0038$ ($0.0038$ being the \emph{maximum} p-value of any test performed
by FCI for the pair of variables) is as likely as a non-adjacency represented by a p-value of $0.6373$ ($0.6373$ being
the  p-value based on which FCI removed this edge).

To obtain MMR estimates, we need to estimate $\pi_0$ and $\xi$. We used the method described in \cite{Storey2003} to
estimate $\pi_0$ on the pooled (maximum) p-values $\{p_j\}_{j=1}^M$ over all data-sets obtained during FCI. For a given
$\hat{\pi_0}$, Equation \ref{eq:nll} can then be easily optimized for $\xi$.

The method used to obtain $\hat{\pi_0}$ assumes independent p-values, which is of course not the case since the test
schedule of FCI depends on previous decisions. In addition, each p-value may be the maximum of several p-values; these
maximum p-values may not follow a uniform distribution even when the non-adjacency (null) hypothesis is true. Finally,
given that p-values stem from tests over different conditioning set sizes, p-values corresponding to adjacencies do not
necessarily follow the same beta distribution. Thus, the approach presented here is at best an approximation.

In the algorithm as presented, a single beta is fit from the pooled p-values of FCI runs over all data-sets. This is
strategy is perhaps more appropriate when individual data-sets have a small number of p-values, so the pooled set
provides a larger sample size for the fitting. Other strategies though, are also possible. One could instead fit a
different beta for each data-set and its corresponding set of p-values. This approach could perhaps be more appropriate
in case the PAG structures $\graph P_i$ vary greatly in terms of sparseness. In addition, one could also fit different
beta distributions for each conditioning set size. Figure \ref{fig:pi0xiALARM} shows the empirical distribution of
p-values and the estimated $\hat{\pi_0}$ based on the p-values returned from FCI on 2, 5 and 10 input data sets,
simulated from a network of 14 variables.

The strategy for selecting non-conflicting constraints based on the MMR strategy is presented in Algorithm
\ref{algo:MMRstrategy}. MMR is a general criterion that can be used to compare confidence in dependencies and
independencies. The method is based on  p-values and thus, can be applied in different types of data (e.g., continuous
and discrete) in conjunction with any appropriate test of independence. Moreover, since it is based on cached p-values,
and fitting a beta distribution is efficient, it adds minimal computational complexity. On the other hand, the estimation
of maximum MAP ratios is based on heuristic assumptions and approximations. Nevertheless, experiments presented in the
following section showcase that the method works similarly if not better than other conflict resolution methods, while
being orders of magnitude computationally more efficient.

\section{Experimental Evaluation\label{sec:experiments}}

We present a series of experiments to characterize how the behavior of \combine\; is affected by the characteristics of
the problem instance and compare it against other alternative algorithm in the literature. We also present a comparative
evaluation of conflict resolution methods, including the one based on the proposed MMR estimation technique. Finally, we
present a proof-of-concept application on real mass cytometry data on human T-cells. In more detail, we initially compare
the complete version of \combine\; (i.e., without restrictions on the maximum path length or the conditioning set)
against SBCSD \citep{hyttinen2013discovering} in ideal conditions (i.e., both algorithms are provided with an
independence oracle). We perform a series of experiments to explore the (a) learning accuracy of \combine\; as a function
of the maximum path length considered by the algorithm, the density and size of the network to reconstruct, the number of
input data sets, the sample size, and the number of latent variables, and (b) the computational time as a function of the
above factors.

All experiments were performed on data simulated from randomly generated networks as follows. The graph of each network
is a DAG with a specified number of variables and maximum number of parents per variable. Variables are randomly sorted
topologically and for each variable the number of parents is uniformly selected between 0 and the maximum allowed. The
parents of each variable are selected with uniform probability from the set of preceding nodes. Each DAG is then coupled
with random parameters to generate conditional linear gaussian networks. To avoid very weak interactions, minimum
absolute conditional correlation was set to 0.2. Before generating a data set, the variables of the graph are partitioned
to unmanipulated, manipulated, and latent. Mean value and standard deviation for the manipulated variables were set to 0
and 1, respectively. Subsequently, data instances are sampled from the network distribution, considering the
manipulations and removing the latent variables. All experiments are performed on \textbf{conservative} families of
targets; the term was introduced in \cite{Hauser2012} to denote families of intervention targets in which all variables
have been observed unmanipulated at least once.

\begin{table}[!t]
\centering
 \begin{tabular}{|c|c|}\hline
\textbf{Problem attribute}&\textbf{Default value used}\\ \hline \hline
Number of variables in the generating DAG & 20 \\ \hline
Maximum number of parents per variable & 5 \\ \hline
Number of input data sets & 5 \\ \hline
Maximum number of latent variables per data set & 3 \\ \hline
Maximum number of manipulated variables per data set & 2 \\ \hline
Sample size per data set & 1000 \\ \hline
\end{tabular}\caption{\label{tab:params}\textbf{Default values used in generating experiments in each iteration of
\combine}. Unless otherwise stated, the input data sets of \combine\; were generated according to these values.}
\end{table}

For each invocation of the algorithm, the problem instance (set of data sets) is generated using the parameters shown in
Table \ref{tab:params}. \combine\; default parameters were set as follows: maximum path length = 3, $\alpha =0.1$ and
maximum conditioning set $maxK=5$, and the Fisher z-test of conditional independence. As far as orientations are
concerned, in our experience, FCI is very prone to error propagation, we therefore used the rule in \citep{Ramsey2006}
for \emph{conservative} colliders. Unless otherwise stated, Algorithm \ref{algo:MMRstrategy} is employed to resolve
conflicts. SAT instances were solved using  MINISAT2.0 \citep{minisat} along with the modifications presented in
\cite{hyttinen2013discovering} for iterative solving and computing the backbone with some minor modifications for
sequentially performing literal queries. In the subsequent experiments, {\em one of the problem parameters in Table
\ref{tab:params} is varied each time, while the others retain the values above}.

To measure learning performance, ideally one should know the ground truth, i.e., the structure that the algorithm would
learn if ran with an oracle of conditional independence, and unrestricted infinite maxK and maximum path length
parameters. Notice, that {\em the original generating DAG structure cannot serve directly as the ground truth}. This is
because the presence of manipulated and latent variables implies that not all structural features of the generating DAG
can be recovered. For example, for the problem instance presented in Figure \ref{fig:inca_example}(middle), the ground
truth structure has one solid edge out of 5, no solid endpoint \ref{fig:inca_example}(right), one absent, and four dashed
edges. Dashed edges and endpoints in the output of the algorithm can only be evaluated if one knows the ground truth
structure. Unfortunately, the ground truth structure cannot be recovered in a timely fashion in most problems involving
more than 15 variables.

As a surrogate, we defined metrics that do not consider dashed edges or endpoints and can be directly computed by
comparing the ``solid'' features of the output with the original data generating graph. Specifically, we used two types
of precision and recall; one for edges (s-Precision/s-Recall) and one for orientations (o-Precision/o-Recall). Let \graph
G be the graph that generated the data (the SMCM stemming from the initial random DAG after marginalizing out variables
latent in all data sets), and \graph H be the summary graph returned by \combine. s-Precision and s-Recall were then
calculated as follows:
\begin{equation*}\textnormal{s-Precision} = \frac{\textnormal{\# solid edges in \graph H that are also in \graph
G}}{\textnormal{\# solid edges in \graph H}}\end{equation*}
and\begin{equation*}\textnormal{s-Recall} = \frac{\textnormal{\# solid edges in \graph H that are also in \graph
G}}{\textnormal{\# edges in \graph G}}.\end{equation*}
Similarly, orientation precision and recall are calculated as follows:
\begin{equation*}\textnormal{o-Precision} = \frac{\textnormal{\# endpoints in \graph G correctly oriented in \graph
H}}{\textnormal{\# of orientations(arrows/tails) in \graph H}}\end{equation*} and \begin{equation*}\textnormal{o-Recall}
= \frac{\textnormal{\# endpoints in \graph G correctly oriented in \graph H}}{\textnormal{\# endpoints in \graph
G}}.\end{equation*}
Since dashed edges and endpoints do not contribute to these metrics, precision in particular could be favorable for
conservative algorithms that tend to categorize all edges (endpoints) as dashed. To alleviate this problem, we accompany
each precision / recall figure with the percentage of dashed edges out of all edges in the output graph to indicate how
conservative is the algorithm. Similarly, we present the percentage of dashed (circled) endpoints out of all endpoints in
the output graph. Finally, we note that in the experiments that follow, unless otherwise stated, we report the median, 5,
and 95 percentile over 100 runs of the algorithm with the same settings.

\subsection{COmbINE vs. SBCSD}
\begin{table}[!t]
\centering
\resizebox{\columnwidth}{!}{
\begin{tabular}{|c|c|c|c|c|c|c|c|}\hline
&&\multicolumn{3}{|c|}{Running time}& \multicolumn{3}{|c|}{Completed instances/}\\
\# &\# max&\multicolumn{3}{|c|}{\textbf{Median} (5 \%ile, 95 \%ile)}& \multicolumn{3}{|c|}{total instances}\\
\cline{3-8}
variables& parents&\combine\;&SBCSD&SBCSD$^*$&\combine\;&SBCSD& SBCSD$'$\\ \hline
\multirow{2}{*}{10}&3&$\textbf{17} (1, 113) $&$\textbf{149} (14, 470)^*$&$\textbf{91} (30, 369)^*$& $50/50$ &30/50&
$48/50$\\ \cline{2-8}
&5&$ \textbf{80} (4, 1192)$&$\textbf{365}(133, 500)^*$&$\textbf{264}(68, 554)^*$& $50/50$ &16/50& $32/50$\\ \hline
\multirow{2}{*}{14}&3&$\textbf{28} (4, 6361)^* $&$-$&$\textbf{451} (407, 492)^*$& $49/50$ &0/50& $4/50$\\ \cline{2-8}
&5&$ \textbf{272} (23, 16107)^*$&$-$&$-$& $43/50$ &$0/50$& $0/50$\\ \hline
\end{tabular}
}
\caption{\label{tab:vsDiscoverer} \textbf{Comparison of running times for \combine\; and SBCSD for networks of 10 and 14
variables}. The table reports the median running time along with the  5 and 95 percentiles, as well as the number of
instances (problem inputs) in which each algorithm managed to complete; $^*$numbers are computed only on the problems for
which the algorithm completed.}
\end{table}

\cite{hyttinen2013discovering} presented a similar algorithm, named SAT-based causal structure discovery (SBCSD). SBCSD
is also capable of learning causal structure from manipulated data-sets over overlapping variable sets. In addition, if
linearity is assumed, it can admit feedback cycles. SBCSD also uses similar techniques for converting conditional
(in)dependencies into a SAT instance. However, the algorithm requires all \m-connections to constrain the search space
(at least the ones that guarantee completeness), while \combine\; uses inducing paths to avoid that. For each adjacency
$X\doublestar Y$ in a data set, \combine\; creates a constraint specifying that at least one path between the variables
is inducing with respect to \set{L_i}. In contrast, SBCSD creates a constraint specifying that at least one path between
the variables is \m-connecting path given each possible conditioning set. So, both algorithms are forced to check every
possible path, yet \combine\; examines each path once (with respect to \set{L_i}), while SBCSD examines it for multiple
possible conditioning sets. The latter choice may be necessary to deal with cyclic structures, but leads to significantly
larger SAT problems when acyclicity is assumed.

SBCSD is not presented with a conflict resolution strategy and so it can only be tested by using an oracle of conditional
independence. Equipping SBCSD with such a strategy is possible, but it may not be straightforward: SBCSD computes the SAT
backbone incrementally for efficiency, which complicates pre-ranking constraints according to some criterion. Since SBCSD
cannot handle conflicts, we compared it to the complete version of our algorithm (infinite maxK and maximum path length)
using an oracle of conditional independence. Since no statistical errors are assumed, the initial search graph for
\combine\; includes all observed arrows. Both algorithms are sound and complete, hence we only compare running time.
SBCSD uses a path-analysis heuristic to limit the number of tests to perform. However, the authors suggest that in cases
of acyclic structures, this heuristic could be substituted with the FCI test schedule. To better characterize the
behavior of SBCSD on acyclic structures, we equipped the original implementation as suggested\footnote{However, we do not
include the Possible d-Separating step of FCI; this step hardly influences the quality of the algorithm
\cite{colombo2012}. Thus, the timing results of Table \ref{tab:vsDiscoverer} are a lower bound on the execution time of
the SBCSD algorithm.}. We denote this version of the algorithm as SBCSD$'$. Also note, that the available implementation
of SBCSD by its authors has an option to restrict the search to acyclic structures, which was employed in the comparative
evaluation. Finally, we note that SBCSD is implemented in C, while \combine\;is implemented in Matlab.

For the comparative evaluation, we simulated random acyclic networks with 10 and 14 variables. The default parameters
were used to generate 50 problem instances for networks with 3 and 5 maximum parents per variable. Both algorithms were
run on the same computer, with 4GB of available memory. SBCSD reached maximum memory and aborted without concluding in
several cases for networks of 10 variables, and {\em in all cases for networks of 14 variables}. SBCSD$'$ slightly
improves the running time over SBCSD. Median running time along with the 5 and 95 percentiles as well as number of cases
completed are reported in Table \ref{tab:vsDiscoverer}. The metrics for each algorithm were calculated only on the cases
where the algorithm completed.

The results in Table \ref{tab:vsDiscoverer} indicate that \combine\; is more time-efficient than SBCSD and SBCSD$'$.
While the running times do depend on implementation, the fact that SBCSD have much higher memory requirements indicates
that the results must be at least in part due to the more compact representation of constraints by \combine\;. \combine\;
managed to complete all cases for networks of 10 and most cases for 14 variables, while SBCSD completed less than 50\%
and 0\%, respectively. SBCSD$'$ completed most cases for 10 variables but only 4\% of cases for 14 variables.
Interestingly, the percentiles for \combine\;are quite wide spanning two orders of magnitude for problems with maxParents
equal to 5 (we cannot compute the actual 95 percentile for SBCSD since it did not complete for most problems). Thus,
performance highly depends on the input structure. Such heavy-tailed distributions are well-noted in the constraint
satisfaction literature \citep{gomes2000heavy}. We also note the fact that \combine\; seems to depend more on the
sparsity and less on the number of variables, while SBCSD's time increases monotonically with the number of variables.
Based on these results, we would suggest the use of \combine\; for problems where acyclity is a reasonable assumption and
the number of variables is relatively high.

\subsection{Evaluation of Conflict Resolution Strategies} 

In this section we evaluate our Maximum Map Ratio strategy ({\bf MMR}) against three other alternatives: A ranking
strategy where constraints are sorted based on Bayesian probabilities as proposed in \cite{Claassen2012}
(\textbf{BCCDR}), as well as a Max-SAT (\textbf{MaxSAT}) and a weighted max-SAT (\textbf{wMaxSAT}) approach.

\textbf{MMR}: This strategy sorts constraints according to the Maximum Map Ratio (Algorithm \ref{algo:MMRstrategy}) and
greedily satisfies constraints in order of confidence; whenever a new constraint is not satisfiable given the ones
already selected, it is ignored (lines \ref{algoline:sat}- \ref{algoline:sate} in Algorithm \ref{algo:MMRstrategy}).

\textbf{BCCDR}: BCCDR sorts constraints according to Bayesian probability estimates of the literals in \graph F as
presented in \cite{Claassen2012}. The same greedy strategy for satisfying constraints in order is employed. Briefly, the
authors of \citep{Claassen2012} propose a method for calculating Bayesian probabilities for any feature of a causal graph
(e.g. adjacency, \m-connection, causal ancestry). To estimate the probability of a feature, for a given data set \set D,
the authors calculate the score of all DAGs of \var N variables. Let $\graph G \vdash f$ denote that a feature \var f is
present in DAG \graph G. The probability of the feature is then calculated as $P(f) =\sum_{\graph G \vdash
\textnormal{f}}{P(\set D|\graph G )P(\graph G)}$. Scoring all DAGs is practically infeasible for networks with more than
5 or 6 variables. Thus, for data sets with more variables, a subset of variables must be selected for the calculation of
the probability of a feature. Following \citep{Claassen2012}, we use 5 as the maximum \var N attempted.

The literals in \graph F represent information on adjacencies: $(\neg) adjacent (\var X, \var Y , \graph P_i)$ and
colliders: $(\neg) collider (\var X, \var Y , \var Z, \graph P_i)$. To apply the method above for a given feature, we
have to select the variables used in the DAGs, a suitable scoring function, and suitable DAG priors. For (non)
adjacencies \var X\doublestar \var Y in PAG $\graph P_i$, we scored the DAGs over variables \var X, \var Y and \set Z,
for the conditioning set \set Z maximizing the p-value of the tests \var X\independent\var Y\given\set Z performed by
FCI. Since the total number of variables cannot exceed 5, the maximum conditioning set for FCI is limited to 3 in all
experiments in this section for a fair comparison. For a (non) collider \var X \doublestar\var Y \doublestar \var Z, we
score all networks over \var X, \var Y and \var Z.

\begin{figure}[!t]\centering
\includegraphics[width = \columnwidth]{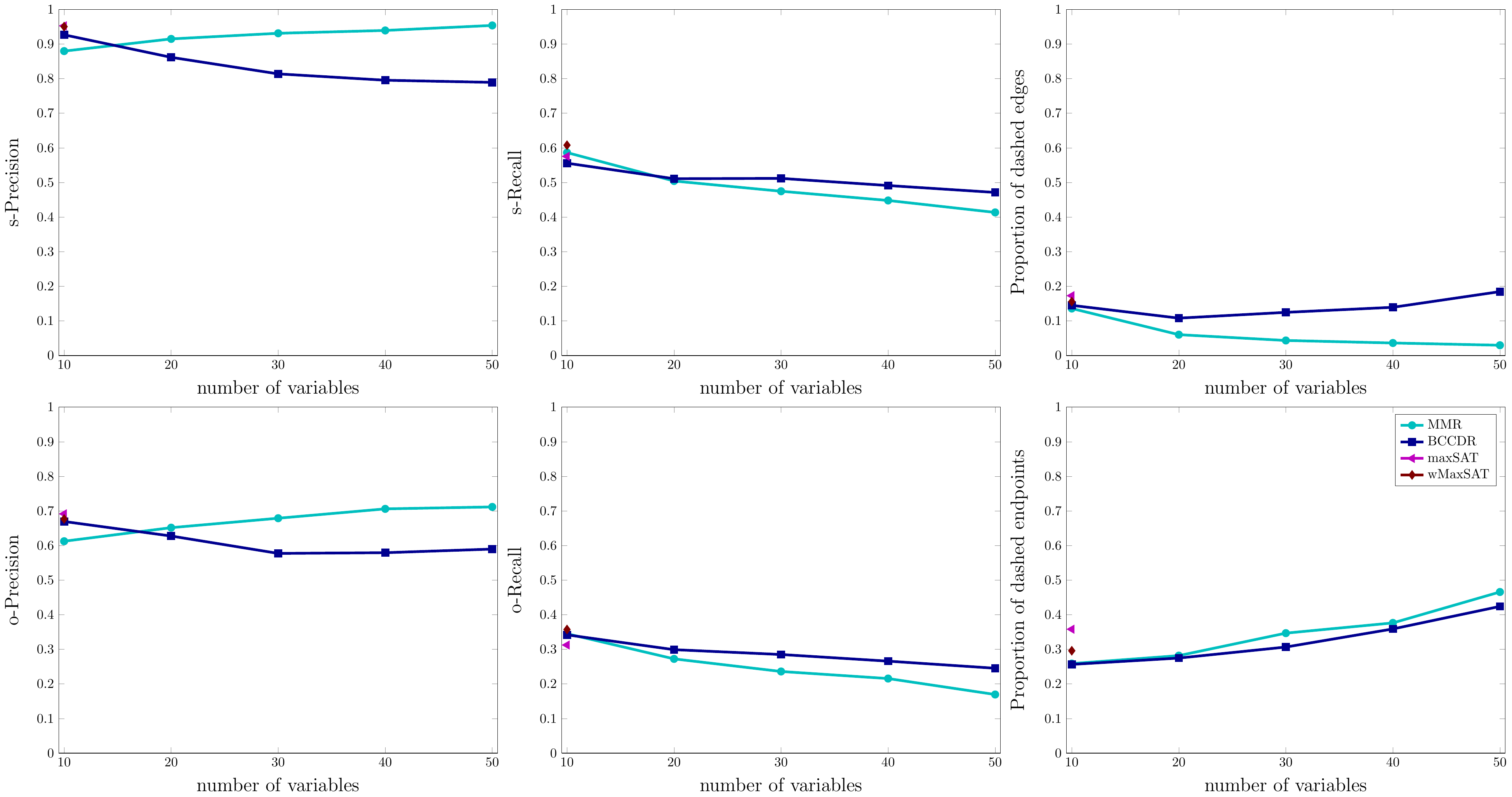}
\caption{\label{fig:cr_strategies} \textbf{Learning performance of \combine\; with various conflict resolution
strategies}. From left to right: Median  s-Precision, s-Recall, proportion of dashed edges (top) and
o-Precision, o-Recall and proportion of dashed endpoints (bottom) for networks of several sizes for various conflict
resolution strategies. Each data set consists of 100 samples. The numbers for wMaxSAT and maxSAT correspond to 22 and 23
cases, respectively, in which the algorithms managed to return a solution within 500 seconds.}
\end{figure}

We use the BGE metric for gaussian distributions \citep{Geiger1994} as implemented in the BDAGL package \cite{EatonBDAGL}
to calculate the likelihoods of the DAGs. This metric is score equivalent, so we pre-computed representatives of the
Markov equivalent networks of up to 5 nodes, and scored only one network per equivalence class to speed up the method.
Priors for the DAGs were also pre-computed to be consistent with respect to the maximum attempted number of nodes (i.e.
5) as suggested in \cite{Claassen2012}.

\textbf{MaxSAT}: This approach tries to satisfy as many literals in \graph F as possible. Recall that the SAT problem
consists of a set of hard-constraints (conditionals, no cycles, no tail-tail edges), which should always be satisfied
(hard constraints), and a set of literals \graph F. Maximum SAT solvers cannot be directly applied to the entire SAT
formula since they do not distinguish between hard and soft constraints. To maximize the number of literals satisfied,
while ensuring all hard-constraints are satisfied we resorted to the following technique: we use the akmaxsat
\citep{kuegel2010improved} {\em weighted} max SAT solver that tries to maximize the sum of the weights of the satisfied
clauses. Each literal is assigned a weight of 1, and each hard-constraint is assigned a weight equal to the sum of all
weights in \graph F plus 10000. The summary graph returned by Algorithm \ref{algo:COmbINE} is based on the backbone of
the subset of literals selected by akmaxsat.

\textbf{wMaxSAT}: Finally, we augmented the above technique with a different weighted strategy that considers the
importance of each literal. Specifically, each literal was weighted proportionally to the logarithm of the corresponding
MMR. Again, each hard-constraint was assigned a weight equal to the sum of all weights in \graph F plus 10000, to ensure
that the solver will always satisfy these statements. The summary graph returned by Algorithm \ref{algo:COmbINE} is based
on the backbone of the subset of literals selected by akmaxsat.

We ran all methods for networks of 10, 20, 30, 40 and 50 variables for data sets of 100 samples to test them on cases
where statistical errors are common. For each network size we performed 50 iterations. \textbf{MaxSAT} and
\textbf{wMaxSAT} often failed to complete in a timely fashion; to complete the experiments we aborted the solver after
500 seconds. We note that this amount of time corresponds to more than 10 times the maximum running time of the MMR
method (calculating MMRs and solving the SAT instance), and more than twice times the maximum running time of the
BCCDR-based method (for 50 variables). Cases where the solver did not complete were not included in the reported
statistics. Unfortunately, \emph{the methods using weighted max SAT solving failed to complete in most cases for 10
variables}, and all cases for more than 10 variables.

The results are shown in Figure \ref{fig:cr_strategies}, where we can see the median performance of both algorithms over
50 iterations. Overall, \textbf{MMR} exhibits better Precision and identifies more solid edges, while \textbf{BCCDR}
exhibits slightly better Recall. \textbf{BCCDR} is better for variable size equal to 10, which could be explained from
the fact that \textbf{MMR} is not provided with sufficient number of p-values to estimate $\hat{\pi}_0$ and $\hat{\xi}$.
In terms of computational complexity, for networks of 50 variables, estimating the \textbf{BCCDR} ratios takes about 150
seconds on average, while estimating the \textbf{MMR} ratios takes less than a second. The more sophisticated search
strategies \textbf{MaxSAT} and \textbf{wMaxSAT} do not seem to offer any significant quality benefits, at least for the
single variable size for which we could evaluate them. Based on these results, we believe that \textbf{MMR} is a
reasonable and relatively efficient conflict resolution strategy.

\subsection{COmbINE performance with increasing maximum path length}
\begin{figure}[!t]\centering
\includegraphics[width = \columnwidth]{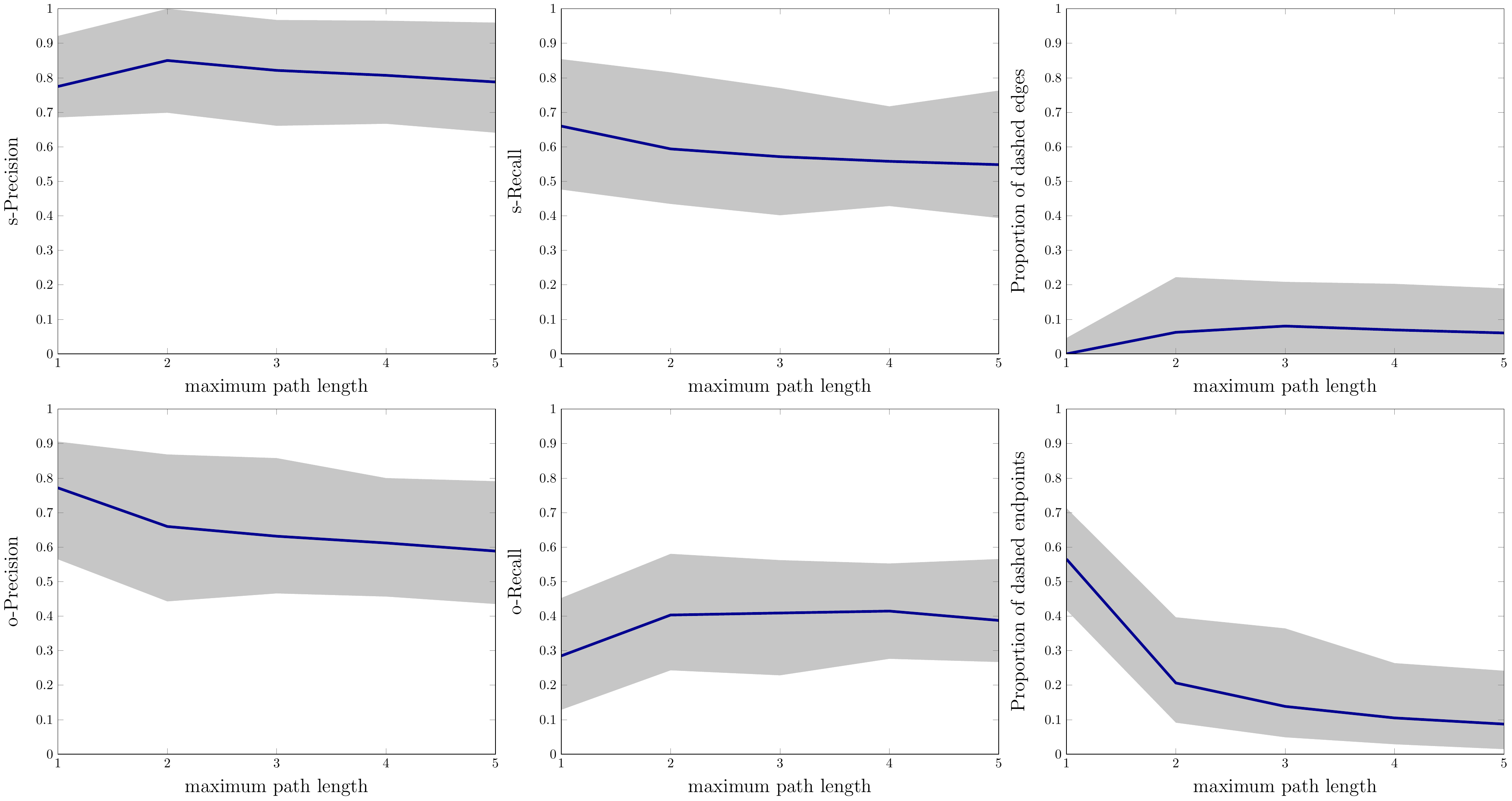}
\caption{\label{fig:err_vs_mpl}\textbf{Learning performance of \combine\; against maximum path length}. From left to
right: s-Precision, s-Recall, percentage dashed edges and o-Precision, o-Recall and percentage of dashed endpoints
(bottom) for varying maximum path length, averaged over all networks. Shaded area ranges from the 5 to the 95 percentile.
Maximum path length 3 seems to be a be a reasonable trade-off between performance, percentage of dashed features, and
efficiency.}
\end{figure}

In this section, we examine the behavior of the algorithm when the length of the paths considered is limited, in which
case the output is an approximation of the actual solution. The \combine\; pseudo-code in Algorithm \ref{algo:COmbINE}
accepts the maximum path length as a parameter.
\begin{figure}[!t]\centering
\includegraphics[width = \columnwidth]{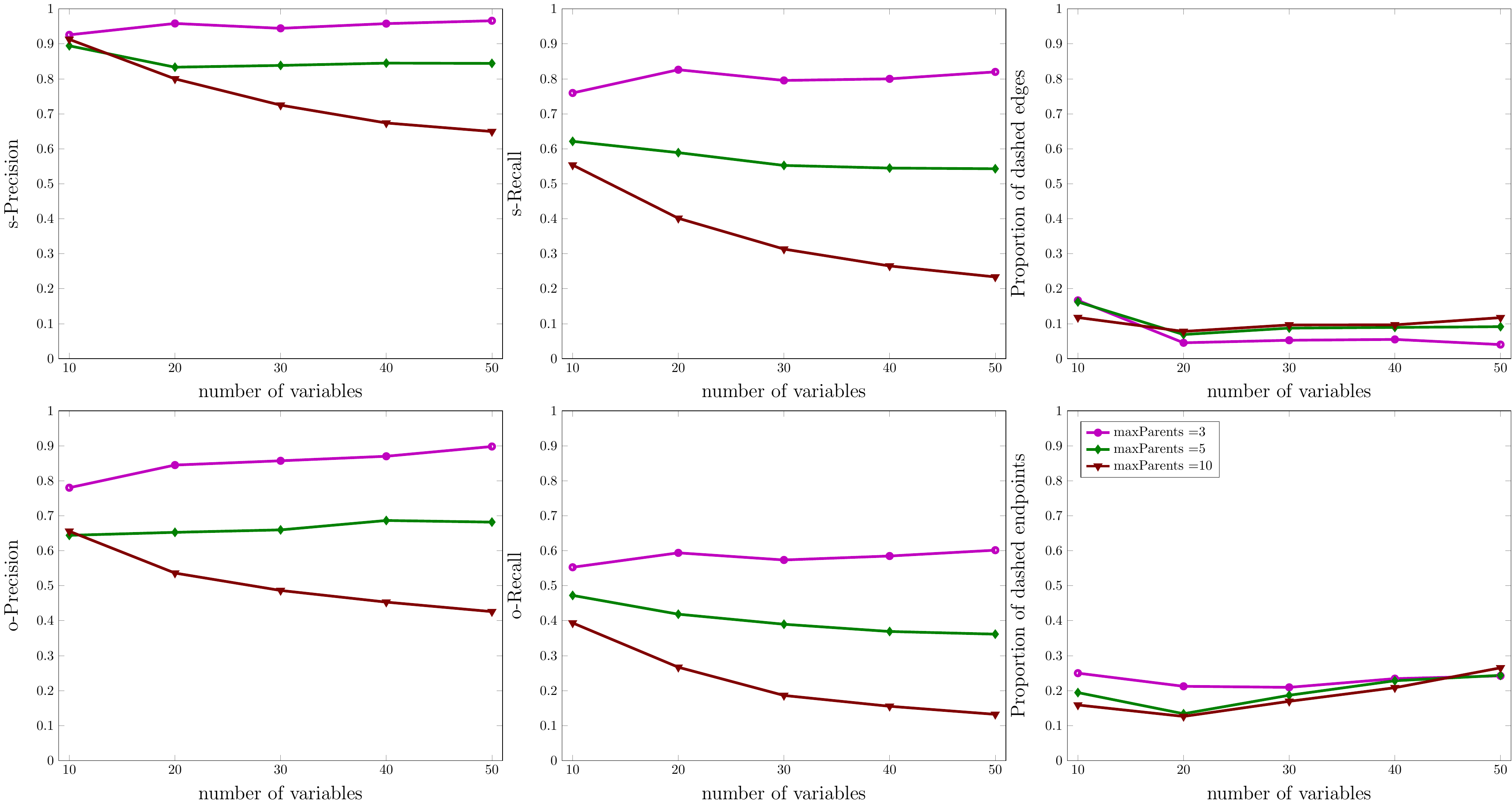}
\caption{\label{fig:err_vs_max_parents}\textbf{Learning performance of \combine\; for various network sizes and
densities}. From left to right: Median s-Precision, s-Recall, proportion of dashed edges (top) and o-Precision, o-Recall
and proportion of dashed endpoints (bottom) for varying network size and density. Density is controlled by limiting the
number of possible parents per variable.  As expected, the performance deteriorates as networks become denser.}
\end{figure}
Learning performance as a function of the maximum path length is shown in Figure \ref{fig:err_vs_mpl}. Notice that when
the path length is increased from 1 to 2 there is drop in the percentage of dashed endpoints, implying more orientations
are possible. For length equal to 1, only unshielded and discriminating colliders are identified, while for length larger
than 2 further orientations become possible thanks to reasoning with the inducing paths. When length is 1, notice that
there are almost no dashed edges (except for the edges added in line \ref{algoline:de} of Algorithm
\ref{algo:initializeSMCM}). When the maximum length increases, adjacencies in one data set, can be explained with longer
inducing paths in the underlying graph and more dashed edges appear. The learning performance of the algorithm is not
monotonic with the maximum length. Explaining an association (adjacency) through the presence of a long inducing path may
be necessary for asymptotic correctness. However, in the presence of statistical errors, allowing such long paths could
lead to complicated solutions or the propagation of errors.

Overall, it seems any increase of the maximum path length above 3 does not significantly affect performance. It seems
that a maximum path length of 3 is a reasonable trade-off among learning performance (precision and recall), percentage
of uncertainties, and computational efficiency. These experiments justify our choice of maximum length 3 as the default
parameter value of the algorithm.

\subsection{COmbINE performance as a function of network density and size\label{sec:vsMaxParents}}

In Figure \ref{fig:err_vs_max_parents} the learning performance of the algorithm is presented as a function of network
density and size. Density was controlled by the maximum parents allowed per variable, set by parameter maxParents during
the generation of the random networks. For all network sizes, learning performance monotonically decreases with increased
density, while the percentage of dashed features does not significantly vary. The size of the network has a smaller
impact on the performance, particularly for the sparser networks. For dense networks, performance is relatively poor and
becomes worse with larger sizes.

\subsection{COmbINE performance over sample size and number of input data sets}
\begin{figure}[!t]\centering
\includegraphics[width =\columnwidth]{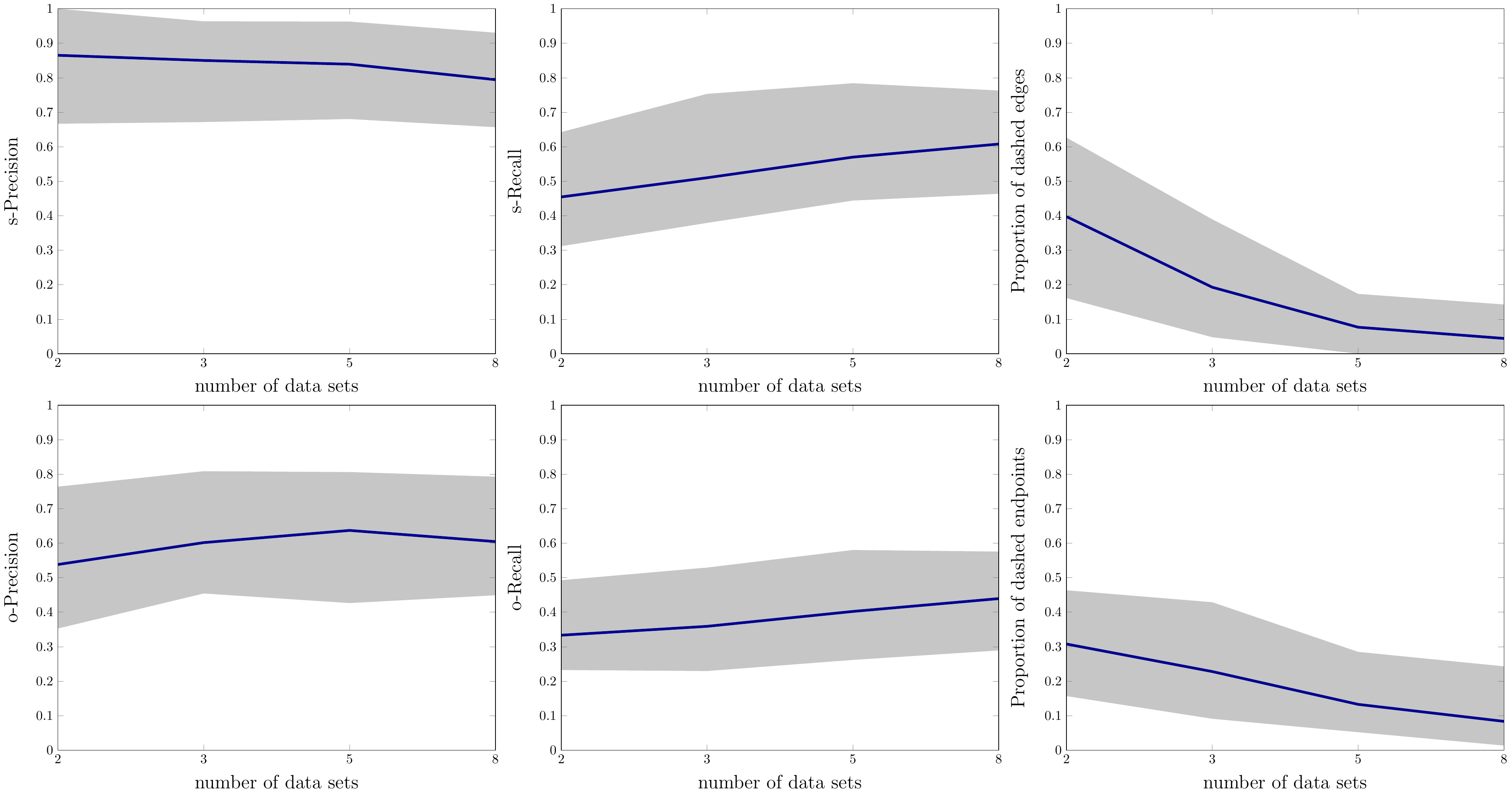}
\caption{\label{fig:err_vs_nexps_1000}\textbf{Learning performance of \combine\; for varying number of input data sets}.
From left to right: Median  s-Precision, s-Recall, Proportion of dashed edges (top) and o-Precision, o-Recall and
proportion of dashed endpoints of (bottom) for varying number of input data sets. Shaded area ranges from the 5 to the 95
percentile. Increasing the number of input data sets improves the performance of the algorithm.} \end{figure}

\begin{figure}[!h]\centering
\includegraphics[width =\columnwidth]{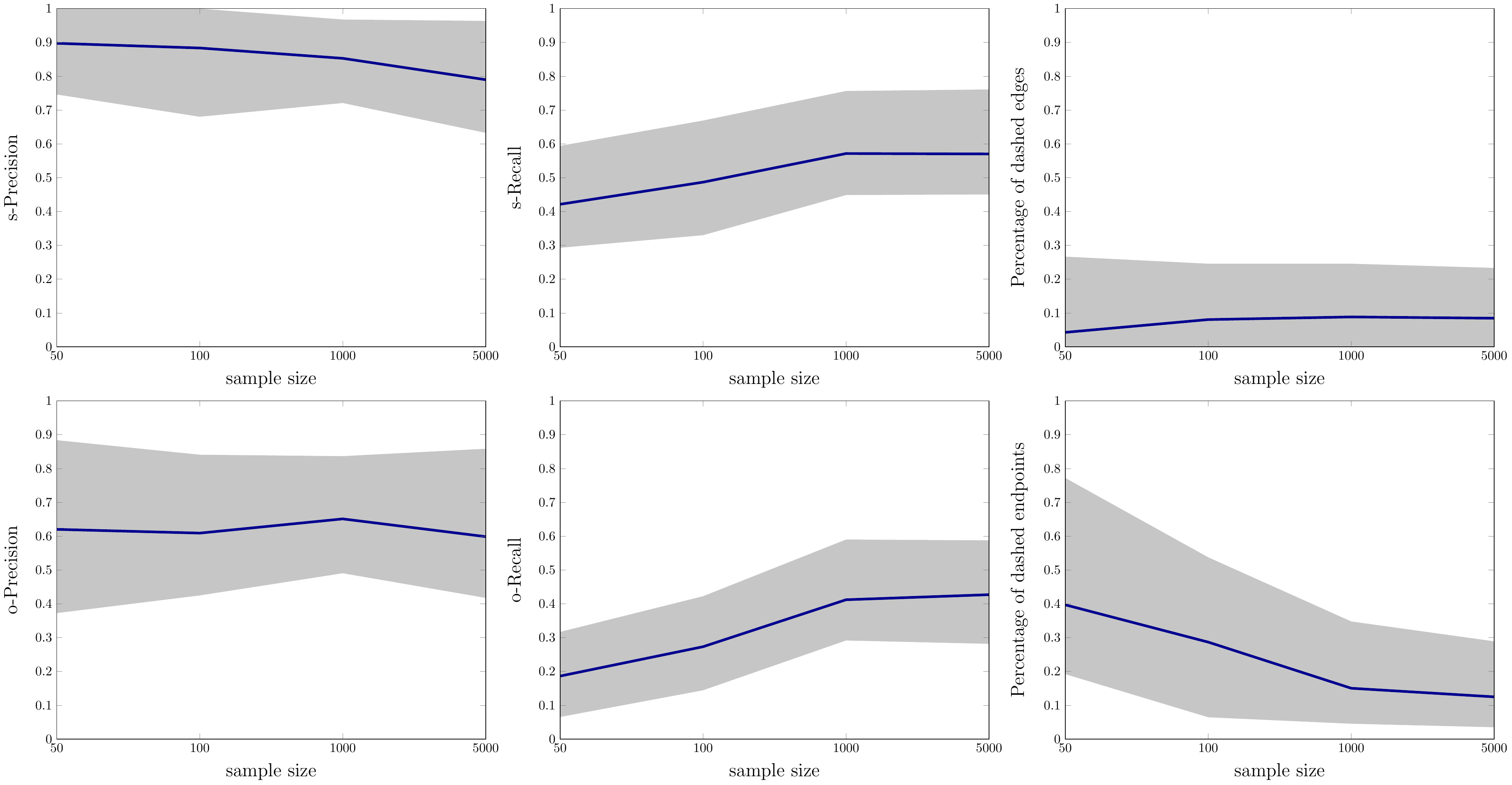}
\caption{\label{fig:err_vs_ss}\textbf{Learning performance of \combine\; for varying sample size per data set}. From left
to right: s-Precision, s-Recall, Proportion of dashed edges (top) and o-Precision, o-Recall and proportion of dashed
endpoints of (bottom) for varying sample size per data set. Shaded area ranges from the 5 to the 95 percentile.
Increasing the sample size improves the performance of the algorithm.}
\end{figure}

Figure \ref{fig:err_vs_nexps_1000} shows the performance of the algorithm with increasing the number of input data sets.
As expected, the percentage of uncertainties (dashed features) is steadily decreasing with increased number of input data
sets. Recall also steadily improves, while Precision is relatively unaffected. Figure \ref{fig:err_vs_ss} holds the
number of input data set constant to the default value 5, while increasing the sample size per data set. Recall in
particular improves with larger sample sizes, while the percentage of dashed endpoints drops.

\subsection{COmbINE performance for increasing number of latent variables}
\begin{figure}[!t]\centering
\includegraphics[width =\columnwidth]{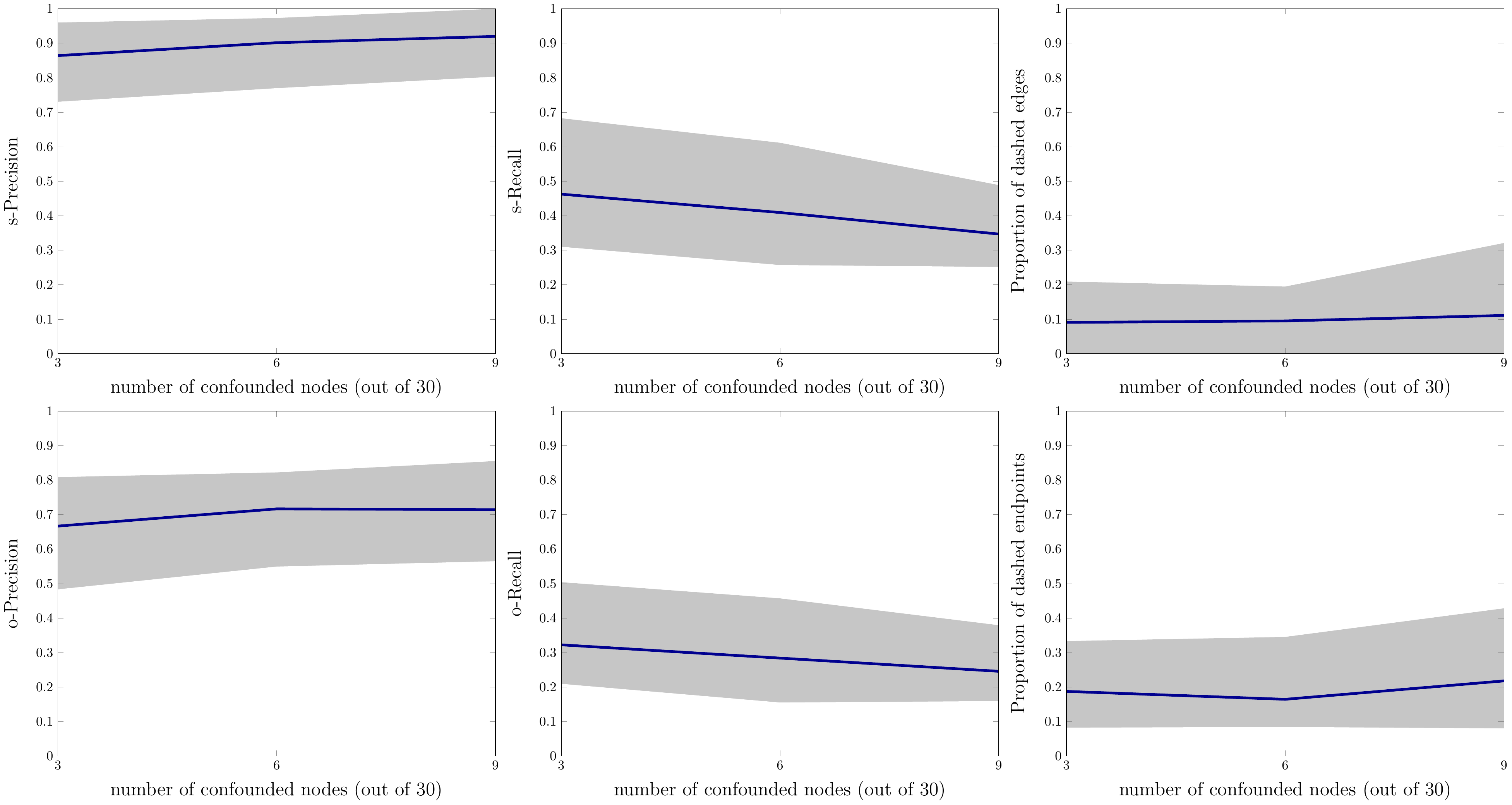}
\caption{\label{fig:err_vs_confounding}\textbf{Learning performance of \combine\; for varying percentage of confounded
variables}. From left to right: s-Precision, s-Recall, percentage of dashed edges (top) and o-Precision, o-Recall and
percentage of dashed endpoints (bottom) for varying  number of confounded nodes for networks of 30 variables. Shaded area
ranges from the 5 to the 95 percentile. Overall, the number of confounding variables does not seem to greatly affect the
algorithm' s performance.}
\end{figure}

We also examine the effect of confounding to the performance of \combine\;. To do so, we generated semi-Markov causal
models instead of DAGs in the generation of the experiments: We generated random DAG networks of 30 variables and then
marginalized out a percentage of the variables. Figure \ref{fig:err_vs_confounding} depicts \combine's performance
against 3, 6, and 9 of latent variables, corresponding to 10\%, 20\% and 30\% of the total number of variables in the
graph, respectively. Overall, confounding does not seem to greatly affect the performance of \combine. We must point out
however, that s-Recall is lower than the s-Recall with no confounded variables for the same network size (see Figure
\ref{fig:err_vs_max_parents}).

\subsection{Running Time for COmbINE}

\begin{figure}[!t]\centering
\includegraphics[width =0.7\columnwidth]{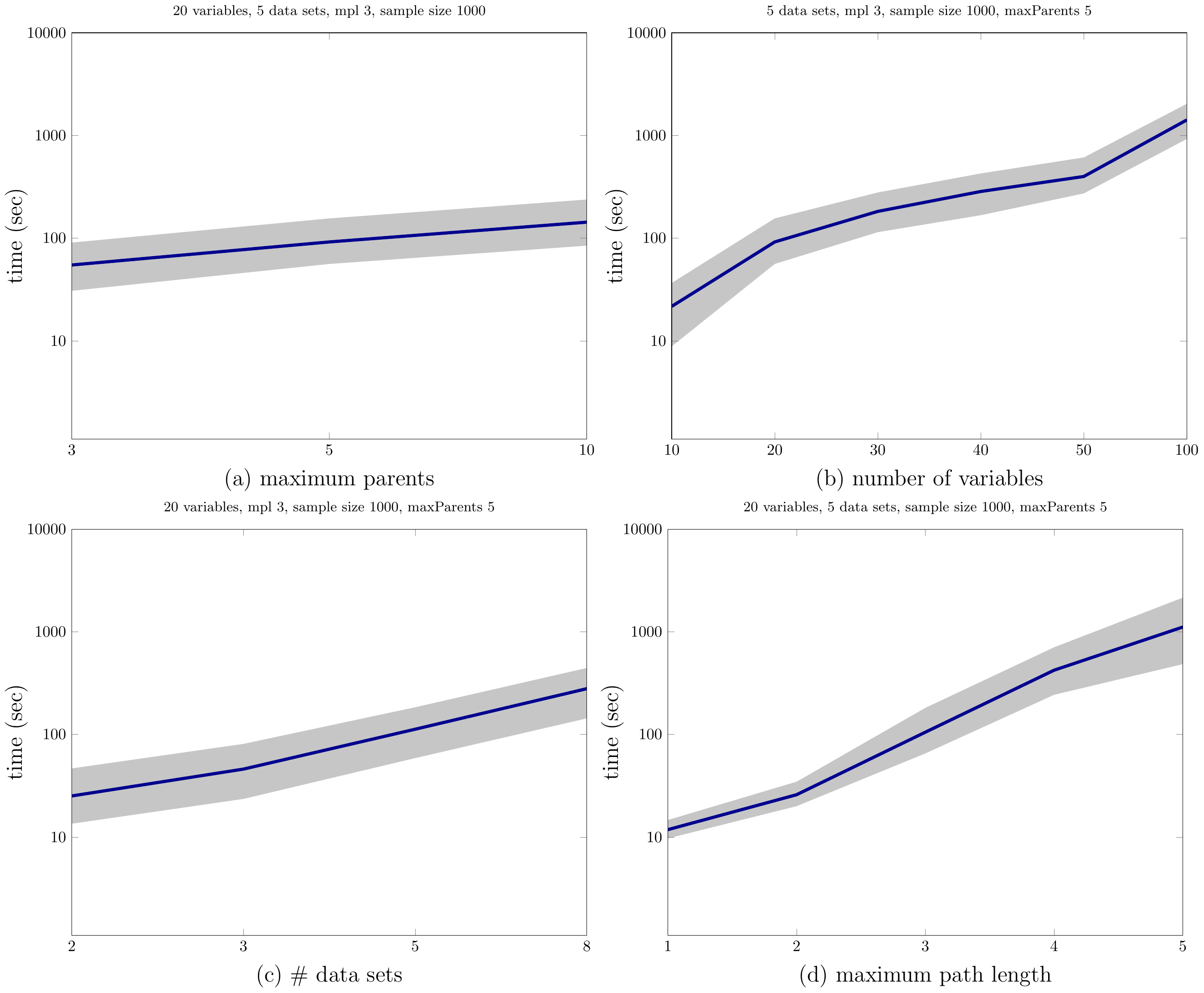}
\caption{\label{fig:times}\textbf{Running time  of \combine\;}. From left to right: Running time (in seconds) is plotted
in logarithmic scale against maximum parents per variable and number of variables (top row); number of data sets and
maximum path length (bottom row). Shaded area ranges from the 5 to the 95 percentile. The number of variables and the
maximum path length seem to be the most critical factors of computational performance. Notice that, \combine\; scales up
to problems with 100 total variables for limited path length and relatively sparse networks.}
\end{figure}

The running time of \combine\; depends on several factors, including the ones examined in the previous experiments:
Maximum path length, number of input data sets and sample size, and, naturally, the number of variables. Figure
\ref{fig:times} illustrates the running time of \combine\; against these factors. Figure \ref{fig:times} (b) presents the
running time of \combine\; against number of variables for networks of 5 maximum parents per variable. The experiments
regarding 10 to 50 variables  have also been presented in terms of learning performance in section
\ref{sec:vsMaxParents}. To further examine the scalability of the algorithm, {\em we also ran \combine\; in networks of
100 variables}, with 5 maximum parents per variable. The experiments were ran with the default parameter values. As we
can see in Figure \ref{fig:times}, the restriction on the maximum path length is the most critical factor for the
scalability of the algorithm.

\subsection{A case study: Mass Cytometry data}
Mass cytometry \citep{Bendall2011} is a recently introduced technique that enables measuring protein activity in cells,
and its main use is to classify hematopoietic cells and identify signaling profiles in the immune system. Therefore, the
proteins are usually measured in a sample of cells and then in a different sample of the same (type of) cells after they
have been stimulated with a compound that triggers some kind of signaling behavior. Identifying the  causal succession of
events during cell signaling is crucial to designing drugs that can trigger or suppress immune reaction. Therefore in
several studies both stimulated and un-stimulated cells are treated with several perturbing compounds to monitor the
potential effect on the signaling pathway.

Mass cytometry data seem to be an suitable test-bed for causal discovery methods: The proteins are measured in single
cells instead of representing tissue averages, the latter being known to be problematic for causal discovery
\citep{chu2003statistical}, and the samples range in thousands. However, the mass cytometer can measure only up to 34
variables, which may be too low a number to measure all the variables involved in a signaling pathway. Moreover, about
half of these variables are surface proteins that are necessary to distinguish (gate) the cells into sub-populations, but
are not functional proteins involved in the signaling pathway. It is therefore reasonable for scientists to perform
experiments measuring overlapping variable sets.

\cite{Bendall2011} and \cite{Bodenmiller2012} both use mass cytometry to measure protein abundance in cells of the immune
system. In both studies, the samples were treated with several different signaling stimuli. Some of the stimuli were
common in both studies. After stimulation with each activating compound, \cite{Bodenmiller2012} also test the cell's
response to 27 inhibitors. One of these inhibitors is also used in \cite{Bendall2011}. For this inhibitor,
\cite{Bendall2011} measured bone marrow cell samples of a single donor. In \cite{Bodenmiller2012}, measurements were
taken from Peripheral blood mononuclear cell samples of a (different) single donor. Despite differences in the
experimental setup, the signaling pathway of every stimulus and every sub-population of cells is considered universal
across (healthy) donors, so the data should reflect the same underlying causal structure.

We focused on two sup-populations of the cells, CD4+ and CD8+ T-cells, which are known to play a central role in immune
signaling. The data were manually gated by the researchers in the original studies. We also focused on one of the stimuli
present in both studies, PMA-Ionomycin, which is known to have prominent effects on T-cells. Proteins pBtk, pStat3,
pStat5, pNfkb, pS6, pp38, pErk, pZap70 and pSHP2 are measured in both data sets (initial p denotes that the concentration
of the  phosphorylated protein is measured). Four additional variables were included in the analysis, pAkt, pLat and
pStat1 measured only in \cite{Bodenmiller2012} and pMAPK measured only in \cite{Bendall2011}. To be able to detect
signaling behavior, we formed data sets that contain both stimulated and unstimulated samples. As mentioned above, the
cells were treated with several inhibitors. Some of these inhibitors target a specific protein, and some of them perturb
the system in a more general or unidentified way. We used three target specific compounds that can be modeled as hard
interventions (i.e. the compounds used to target these proteins are known to be specific and to have an effect in the
phosphorylation levels of the target). More information on the specific compounds can be found in the respective
publications.  We ended up with four data sets for each sub-population. Details can be found in Table
\ref{tab:case_study}.

\begin{table}
\centering
\begin{tabular}{|c|c|c|c|c|}\hline
Data set & Source & latent (\set{L_i}):  & manipulated(\set{I_i})& Donor \\ \hline
\set{D_1} & \cite{Bodenmiller2012} & pMAPK & pAkt & 1\\ \hline
\set{D_2} & \cite{Bodenmiller2012} & pMAPK & pBtk & 1\\ \hline
\set{D_3} & \cite{Bodenmiller2012} & pMAPK & pErk & 1\\ \hline
\set{D_4} & \cite{Bendall2011} & pAkt, pLat, pStat1 & pErk & 2\\ \hline
\end{tabular}
\caption{\label{tab:case_study}\textbf{Summary of the mass cytometry  data sets co-analyzed with \combine}. The procedure
was repeated for two sub-populations of cells, CD4+ cells and CD8+ cells.}
\end{table}

\begin{figure}[!t]\centering
\begin{tabular}{|l|r|}
\hline
\includegraphics[width = 0.4\columnwidth]{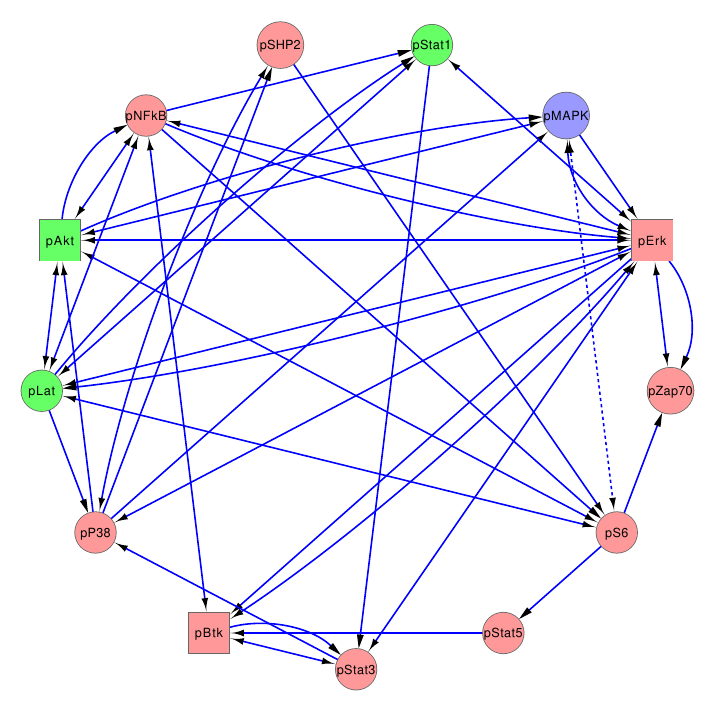}&
\includegraphics[width = 0.4\columnwidth]{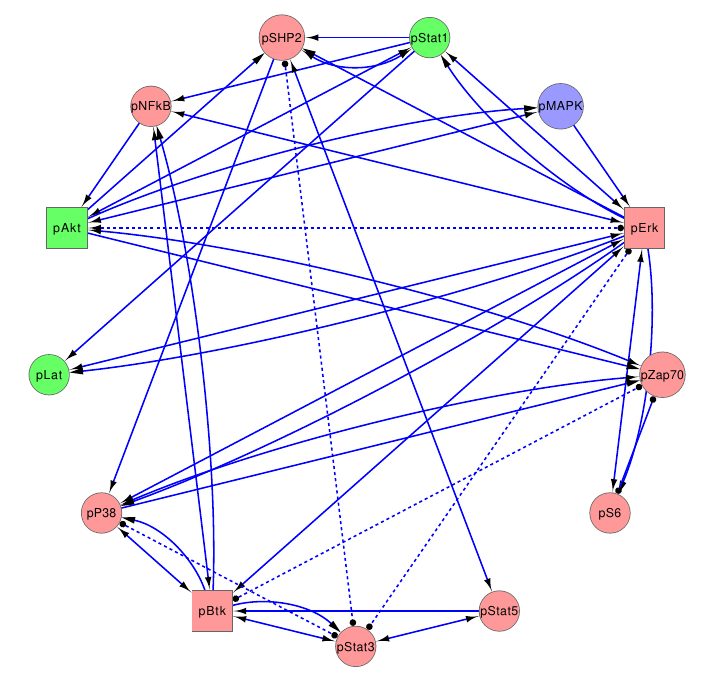}\\ \hline
\end{tabular}
\caption{\label{fig:case_study}\textbf{A case study for \combine: Mass cytometry data}. \combine\; was run on 4 different
mass cytometry data  for two different cell populations: CD4+ T-cells (left) and CD8+ T-cells (right). In each data set,
one variable was manipulated (pAkt, pBTk, pErk, pErk respectively). Variables pAkt, pLat and pStat1 are only measured in
data sets 1-3, while pMAPK is only measured in data set 4. Notice that pAkt is predicted to be a direct cause of pMAPK
in both CD4+ and CD8+ cells, even though the two variables have never been measured together.}
\end{figure}

Protein interactions are typically non-linear, so we discretized the data into 4 bins. We ran Algorithm
\ref{algo:COmbINE} with maximum path length 3. We used the $G^2$ test of independence for FCI with $\alpha =0.1$ and
maxK=5. We used Cytoscape \citep{smoot2011cytoscape} to visualize the summary graphs produced by COmbINE, illustrated in
Figure \ref{fig:case_study}.

Unfortunately, the ground truth for this problem is not known for a full quantitative evaluation of the results.
Nevertheless, this set of experiments demonstrates the availability of real and important data sets and problems that are
suited integrative causal analysis. Second, these experiments provide a proof-of-concept for the specific algorithm. One
type of interesting type of inference possible with \combine\; and similar algorithms is the prediction that pAkt is a
direct cause of pMAPK in both CD4+ and CD8+ cells, {\em even though the variables are not jointly measured in any of the
input data sets}. Evidence of a direct protein interaction between the two proteins does exists in the literature
\cite{rane2001p38}. Thus, methods for learning causal structure from multiple manipulations over overlapping variables
potentially constitute a powerful tool in the field of mass cytometry.

We do not make any claims for the validity of the output graphs and they are presented only as a proof-of-concept, as
there are several potential pitfalls. \combine\; assumes lack of feedback cycles, which is not guaranteed in this system
(we note however, that acyclic networks have been successfully used for reverse engineering protein pathways in the past
\citep{Sachs2005}). Causal discovery methods that allow cycles \cite{hyttinen2013discovering}  on the other hand rely on
the assumption of linearity, which is also known to be heavily violated in such networks. Thus, which set of assumptions
best approximates the specific system is unknown.

\section{Conclusions and Future Work\label{sec:conclusions}}
We have presented \combine, a sound and complete algorithm that performs causal discovery from multiple data sets that
measure overlapping variable sets under different interventions in acyclic domains. \combine\; works by converting the
constraints on inducing paths in the sought out semi Markov causal model (SMCMs) that stem from the discovered
(in)dependencies into a SAT instance. \combine\; outputs a summary of the structural characteristics of the underlying
SMCM, distinguishing between the characteristics that are identifiable from the data (e.g., causal relations that are
postulated as present), and the ones that are not (e.g., relations that could be present or not). In the empirical
evaluation the algorithm outperforms in efficiency a recently published similar one \citep{hyttinen2013discovering} that,
given an oracle of conditional independence, performs the same inferences by checking all \m-connections necessary for
completeness.

\combine\; is equipped with a conflict resolution technique that ranks dependencies and independencies discovered
according to confidence as a function of their p-values. This technique allows it to be applicable on real data that may
present conflicting constraints due to statistical errors. To the best of our knowledge, \combine\; is the only
implemented algorithm of its kind that can be applied on real data.

The algorithm is empirically evaluated in various scenarios, where it is shown to exhibit high precision and recall and
reasonable behavior against sample size and number of input data sets. It scales up to networks with up to 100 variables
for relatively sparse networks. Moreover, it is possible for the user to trade the number of inferences for improved
computational efficiency by limiting the maximum path length considered by the algorithm. As a proof-of-concept
application, we used \combine\; to analyze a real set of experimental mass-cytometry data sets measuring overlapping
variables under three different interventions.

\combine\; outputs a summary of the characteristics of the underlying SMCM that can be identified by computing the
backbone of the corresponding SAT instance. The conversion of a causal discovery problem to a SAT instance makes
\combine\; easily extendable to other inference tasks. One could instead produce all SAT solutions and obtain all the
SMCMs that are plausible (i.e., fit all data sets). In this case, \combine\; with input a single PAG would output all
SMCMs that are Markov Equivalent with the  PAG; there is no other known procedure for this task. Alternatively, one could
easily query whether there are solution models with certain structural characteristics of interest (e.g., a directed path
from $A$ to $B$); this is easily done by imposing additional SAT clauses expressing the presence of these features.
Incorporating certain types of prior knowledge such as causal precedence information can also be achieved by imposing
additional path constraints. Future work includes extending this work for admitting soft interventions and known
instrumental variables. The conflict resolution technique proposed could be employed to standard causal discovery
algorithms that learn from single data sets, in an effort to improve their learning quality.

\appendix
\section*{Appendix A.}
\label{app:proofs}

\textbf{Proof of Proposition \ref{prop:mcmindpaths}}\\
\textbf{Proposition}\\
Let \set O be a set of variables and \graph J the independence model over \set V. Let \graph S be a SMCM over variables
\set V that is faithful to \prob J and \graph M be the MAG over the same variables that is faithful to \prob J. Let
$\var X, \var Y\in \set O$. Then there is an inducing path between $\var X$ and $\var Y$ with respect to $\set L$, $\set
L\subseteq\set V$  in \graph S if and only if there is an inducing path between $\var X$ and $\var Y$ with respect to
\set L in \graph M.\\

\begin{proof}($\Rightarrow$) Assume there exists a path \var p in  \graph S that is inducing w.r.t. \set L. Then by
theorem \ref{the:indpathssmcm} there exists no $\set Z\subseteq\set V\setminus\set L\cup \{X, Y\}$ such that \var X and
\var Y are \m-separated given \set Z in \graph S, and since \graph S and \graph M entail the same \m-separations there
exists no $\set Z\subseteq\set V\setminus\set L\cup \{X, Y\}$ such that \var X and \var Y are \m-separated given \set Z
in \graph M. Thus, by Theorem \ref{the:indpaths} there exists an inducing path between \var X and \var Y with respect to
\set L in \graph M.\\
($\Leftarrow$) Similarly,  assume there exists a path \var p in  \graph M that is inducing w.r.t. \set L. Then by theorem
\ref{the:indpaths} there exists no $\set Z\subseteq\set V\setminus\set L\cup \{X, Y\}$ such that \var X and \var Y are
\m-separated given \set Z in \graph M, and since \graph S and \graph M entail the same \m-separations there  exists no
$\set Z\subseteq\set V\setminus\set L\cup \{X, Y\}$ such that \var X and \var Y are \m-separated given \set Z in \graph
S. Thus, by Theorem \ref{the:indpathssmcm} there exists an inducing path between \var X and \var Y with respect to \set L
in \graph S.
\end{proof}
\textbf{Proof of Lemma \ref{lemma:indpaths}}\\
\textbf{Lemma}\\
Let $\{\graph P_i\}_{i=1}^N$ be a set of PAGs and \graph S a SMCM such that \graph S is possibly underlying for $\{\graph
P_i\}_{i=1}^N$, and let \graph H be the initial search graph returned by Algorithm \ref{algo:initializeSMCM} for
$\{\graph P_i\}_{i=1}^N$. Then, if \var p is an ancestral path in \graph S, then \var p is a possibly ancestral path in
\graph H. Similarly, if \var p is a possibly inducing path with respect to \set L in \graph S, then \var p is a possibly
inducing path with respect to \set L in \graph H.
\begin{proof} We will first prove that any path in \graph S is a path also in \graph H, i.e. \graph H has a superset of
edges compared to \graph S. If \var X and \var Y are adjacent in \graph S, then one of the following holds:
\begin{enumerate}
\item $\exists i$ s.t.  $\var X, \var Y \in \set O_i\setminus\set I_i$. Then the edge is present in $\graph
    S\manip{\set I_i}$, and \var X and \var Y are adjacent in $\graph P_i$: the edge is added to \graph H in Lines
    \ref{algoline:addedges} of Algorithm \ref{algo:initializeSMCM}.
\item $\not \exists i$ s.t.  $\var X, \var Y \in \set O_i\setminus\set I_i$. Then the edge is added in \graph H in
    Line \ref{algoline:addedgesu} of Algorithm \ref{algo:initializeSMCM}.
\end{enumerate}
Therefore, every edge in \graph S is present also in \graph H. We must also prove that no orientation in \graph H is
oriented differently in \graph S: \graph H has only arrowhead orientations, so we must prove that, if  \var
X\starleftarrow \var Y in \graph H and \var X and \var Y are adjacent in both graphs, \var X\starleftarrow \var Y in
\graph S.

Arrows are added to \graph H  in Line \ref{algoline:addOrientations} or in Lines \ref{algoline:addedges1} of the
Algorithm.
Arrowheads added in Line \ref{algoline:addOrientations} occur in all $\graph P_i$. If X\starrightarrow \var Y in $\graph
P_i$, this means that \var Y is not an ancestor of \var X in $\graph S\manip{\set I_i}$. Assume that  \var X\pagleftarrow
\var Y in \graph S: If \var X in $\set I_i$, the edge would be absent in $\graph  S\manip{\set I_i}$ and $\graph P_i$. If
$\var X\not\in\set I_i$, \var X would be ancestor of \var Y in \graph S\manip{\set I_i}, which is a contradiction.
Therefore, if \var X and \var Y are adjacent in \graph S,  X\starrightarrow \var Y  in \graph S.

The latter type of arrows correspond to cases where an edge is not present in any $\graph P_i$, $\not \exists i$ s.t.
$\var X, \var Y\in\set O_i\setminus\set I_i$, but $\exists i$ s.t. $\var X, \var Y \in \set O_i$, $\var X\in \set I_i$
and $\var Y\not\in \set I_i$. Then an arrow is added towards \var X. Assume the opposite holds: \var X\pagrightarrow \var
Y in \graph S, then \var X\pagrightarrow \var Y in $\graph S\manip{\set I_i}$, and since both variables are observed in
\var i the edge would be present in $\graph P_i$, which is a contradiction. Thus, if the edge is present in \graph S, the
edge is oriented into \var X.

Thus, \graph H has a superset of edges of \graph S, and for any edge present in both graphs, the orientations are the
same. Thus, if Then, if \var p is an ancestral path in \graph S, then \var p is a possibly ancestral path in \graph H.
Similarly, if \var p is a possibly inducing path with respect to \set L in \graph S, then \var p is a possibly inducing
path with respect to \set L in \graph H.
\end{proof}
\textbf{Proof of Lemma \ref{lemma:sound}}\\
\textbf{Lemma}\\Let $\{\set D_i\}_{i=1}^N$ be a set of data sets over overlapping subsets of \set O, and $\{\set
I_i\}_{i=1}^N$ be a set of (possibly empty) intervention targets such that $\set I_i\subset \set O_i$ for each i. Let
$\graph P_i$  be output PAG of  FCI  for data set $\set D_i$, $\Phi \wedge\graph F'$ be the final formula of Algorithm
\ref{algo:COmbINE}, and \graph S be a possibly underlying SMCM for $\{\set P_i\}_{i=1}^N$. Then \graph S satisfies $\Phi
\wedge\graph F'$.\\
\begin{proof} Constraints in Lines \ref{algoline:ndc} and \ref{algoline:ntt} of Algorithm \ref{algo:addConstraints} are
satisfied since \graph S is a semi-Markov causal model.

Since $\graph M_i\in \graph P_i\forall i$, $\graph M_i$ and $\graph P_i$ share the same adjacencies and non-adjacencies.
If \var X and \var Y are  adjacent in $\graph P_i$,  \var X and \var Y are  adjacent in $\graph M_i$, and by Proposition
\ref{prop:mcmindpaths} there exists an inducing path with respect to $\set L_i$ in  $\graph S\manip{\set I_i}$, and by
Lemma \ref{lemma:indpaths} this path is a possibly inducing path in the initial search graph. If \var X and \var Y are
not adjacent in $\graph P_i$,  \var X and \var Y are not adjacent in $\graph M_i$, and by Proposition
\ref{prop:mcmindpaths} there exists no inducing path with respect to $\set L_i$ in  $\graph S\manip{\set I_i}$. Thus,
constraints added in Line \ref{algoline:bc1} of Algorithm \ref{algo:addConstraints} along with the corresponding literals
$(\neg) adjacent(X, Y, \graph P_i)$ are satisfied by \graph S.

If \var X\doublestar\var Y \doublestar \var Z is an unshielded triple  in $\graph P_i$, \var X\doublestar\var Y
\doublestar \var Z is an unshielded triple  in $\graph M_i$. If \var Y is a collider on the triple on $\graph P_i$ then
\var Y is a collider on the triple on $\graph M_i$ and by the semantics of edges in MAGs \var Y is not an ancestor of
\var X nor \var Z $\graph S\manip{\set I_i}$. Thus, constraints added to $\Phi$ in Line \ref{algoline:bc2b} along with
the corresponding literal $collider(X, Y, Z, \graph P_i)$ are satisfied by \graph S. Similarly, if \var Y is not a
collider on the triple, \var Y is  an ancestor of either \var X or \var Z in $\graph M_i$ and there exists a relative
ancestral path $p_{YX}$ or $p_{YZ}$ in $\graph S\manip{\set I_i}$. By Lemma \ref{lemma:indpaths}, this path is a possibly
ancestral path in the initial \graph H. Thus, \graph S satisfies the constraints added to $\Phi$ in in Line
\ref{algoline:bc2a} along with the corresponding literal $dnc(X, Y, Z, \graph P_i)$.

If $\langle \var W,\dots,\var X, \var Y, \var Z\rangle$ is a discriminating path for \var V in  $\graph P_i$ and $\graph
M_i$ and \var Y is a collider on the path in $\graph P_i$, then \var Y is a collider on the path in $\graph M_i$,
therefore \var Y is not an ancestor of either \var X or \var Z in $\graph S\manip{\set I_i}$, so \graph S satisfies the
constraints added to $\Phi$ in Line \ref{algoline:bc3b} of Algorithm \ref{algo:addConstraints} along with the
corresponding literal $collider(X, Y, Z, \graph P_i)$. Similarly, if \var Y is not a collider on the triple, \var Y is
an ancestor of either \var X or \var Z in $\graph M_i$ and there exists a relative ancestral path $p_{YX}$ or $p_{YZ}$ in
$\graph S\manip{\set I_i}$. By Lemma \ref{lemma:indpaths}, this path is a possibly ancestral path in the initial \graph
H. Thus, \graph S satisfies the constraints added to $\Phi$ in in Line \ref{algoline:bc3a} along with the corresponding
literal $dnc(X, Y, Z, \graph P_i)$.\end{proof}
\textbf{Proof of Lemma \ref{lemma:complete}}\\
\textbf{Lemma}\\Let $\{\set D_i\}_{i=1}^N$, $\{\set I_i\}_{i=1}^N$, $\{\set P_i\}_{i=1}^N$, $\Phi \wedge\graph F'$ be
defined as in Lemma \ref{lemma:sound}. If graph S satisfies $\Phi \wedge\graph F'$, then \graph S is a possibly
underlying SMCM for $\{\set P_i\}_{i=1}^N$.
\begin{proof}\graph S is a SMCM: \graph S is by construction a mixed graph, and it satisfies constraints in Lines
\ref{algoline:ndc} and \ref{algoline:ntt} of Algorithm \ref{algo:addConstraints}, so it has no directed cycles, and at
most one tail per edge.

$\graph M_i$ and $\graph P_i$ share the same edges: If \var X and \var Y are adjacent in $\graph P_i$, then by the
constraints in Line \ref{algoline:bc1} of Algorithm \ref{algo:addConstraints} there exists an inducing path with respect
to $\set L_i$ in $\graph S\manip{\set I_i}$, therefore \var X and \var Y are adjacent in $\graph M_i$. If \var X and \var
Y are not adjacent in $\graph P_i$ then by the same constraints there exists no inducing path with respect to $\set L_i$
in $\graph S\manip{\set I_i}$, therefore \var X and \var Y are not adjacent in $\graph M_i$.

$\graph M_i$ and $\graph P_i$ share the same unshielded colliders: Let \var X\doublestar \var Y\doublestar  \var Z be an
unshielded triple in  $\graph P_i$. Since $\graph P_i$ and $\graph M_i$ share the same edges, \var X\doublestar\var
Y\doublestar \var Z is an unshielded triple in $\graph M_i$. If the triple is an unshielded collider in $\graph P_i$ then
by the constraints in Line \ref{algoline:bc2b} of Algorithm \ref{algo:addConstraints} \var Y is not an ancestor of either
\var X or \var Z in $\graph S\manip{\set I_i}$, thus \var X\starrightarrow\var Y\starleftarrow \var Z in $\graph M_i$. If
on the other hand the triple is a definite non-collider in $\graph P_i$, then by the constraints in  Line
\ref{algoline:bc2a} of Algorithm \ref{algo:addConstraints} \var Y is an ancestor of either \var X or \var Z in $\graph
S\manip{\set I_i}$, therefore either \var Y\pagrightarrow \var X or \var Y\pagrightarrow \var Z in $\graph M_i$, thus,
the triple is an unshielded non-collider in $\graph M_i$.

If $\langle \var W,\dots,\var X, \var Y, \var Z\rangle$ is a discriminating path for \var V in both $\graph M_i$ and
$\graph P_i$, and \var Y is a collider on the path, then by the constraints in  Line \ref{algoline:bc3b} of Algorithm
\ref{algo:addConstraints}  \var Y is not an ancestor of  \var X or \var Z in  $\graph S\manip{\set I_i}$, therefore \var
Y is a collider on the same path in $\graph M_i$. If, conversely, \var Y is not a collider on the path, then by the
constraints in  Line \ref{algoline:bc3a} of Algorithm \ref{algo:addConstraints}, \var Y is an ancestor of either \var X
or \var Z, thus, \var X is not a collider on the same path in $\graph M_i$.
\end{proof}

\bibliography{INCA_BIB}
\end{document}